\documentclass[nohyperref]{article}

\usepackage[final]{neurips_2022}

\usepackage{microtype}
\usepackage{graphicx,color}
\usepackage{subfigure}
\usepackage{booktabs}
\usepackage{hyperref}
\usepackage{amsmath}
\usepackage{amssymb}
\usepackage{bm}
\usepackage{mathtools}
\usepackage{amsthm,threeparttable}
\usepackage[capitalize,noabbrev]{cleveref}
 \usepackage{multicol}
 \usepackage{multirow}
\usepackage{pifont} 
\usepackage{amsbsy}
\usepackage[normalem]{ulem}
\usepackage{algorithmic}
\usepackage[ruled]{algorithm2e}
\usepackage{enumitem}

\usepackage{selectp}

\theoremstyle{plain}
\newtheorem{theorem}{Theorem}
\newtheorem{proposition}{Proposition}
\newtheorem{lemma}{Lemma}
\newtheorem{corollary}[theorem]{Corollary}
\theoremstyle{definition}
\newtheorem{definition}{Definition}
\newtheorem{assumption}{Assumption}
\theoremstyle{remark}

\crefname{section}{sec.}{sec.}

\usepackage[textsize=tiny]{todonotes}

\title{Generalization Properties of NAS under Activation and Skip Connection Search}

\author{Zhenyu Zhu, \quad Fanghui Liu, \quad Grigorios G Chrysos, \quad Volkan Cevher\vspace{2mm} \\
{\hspace*{\fill}EPFL, Switzerland\hspace*{\fill}}\\
{\hspace*{\fill}\texttt{\{[first name].[surname]\}@epfl.ch}\hspace*{\fill}}
}

\begin{document}
\maketitle

\begin{abstract}
Neural Architecture Search (NAS) has fostered the automatic discovery of state-of-the-art neural architectures. Despite the progress achieved with NAS, so far there is little attention to theoretical guarantees on NAS. In this work, we study the generalization properties of NAS under a unifying framework enabling (deep) layer skip connection search and activation function search. To this end, we derive the lower (and upper) bounds of the minimum eigenvalue of the Neural Tangent Kernel (NTK) under the (in)finite-width regime using a certain search space including mixed activation functions, fully connected, and residual neural networks. We use the minimum eigenvalue to establish generalization error bounds of NAS in the stochastic gradient descent training. Importantly, we theoretically and experimentally show how the derived results can guide NAS to select the top-performing architectures, even in the case without training, leading to a train-free algorithm based on our theory. Accordingly, our numerical validation shed light on the design of computationally efficient methods for NAS. Our analysis is non-trivial due to the coupling of various architectures and activation functions under the unifying framework and has its own interest in providing the lower bound of the minimum eigenvalue of NTK in deep learning theory.
\end{abstract}

\section{Introduction}
\label{sec:introduction}

Neural Architecture Search (NAS)~\citep{45826} is a powerful technique that enables the automatic design of neural architectures. NAS defines a set of operations (referred to as the \emph{search space}), that include various activation functions and layer types, or potential connections among layers~\citep{JMLR:v20:18-598, Ren2020ACS}. Optimization over the search space returns the optimal architecture as a subset of the possible combinations of operations. NAS\footnote{\label{foot:generalization_nas_architecture} In the sequel, we interchangeably refer to NAS as the ``architecture obtained from NAS'' or the framework to design the neural architecture.} obtains state-of-the-art results in image recognition~\citep{liu2019auto, ding2020autospeech, zhang2019customizable, chen2019detnas} or can be used to further improve architectures defined by a human expert~\citep{tan2019efficientnet}. The spectacular results obtained by NAS have led to a significant interest in the community to further improve the NAS algorithms, the search space etc. However, to date little focus has been provided in the following question: \emph{Can NAS\textsuperscript{~\ref{foot:generalization_nas_architecture}} achieve generalization guarantees similar to a typical neural network?}

Neural tangent kernel (NTK)-based analysis \citep{jacot2018neural} is a powerful method for analyzing the optimization and the generalization of deep networks~\citep{pmlr-v97-allen-zhu19a, cao2019generalization, chen2020generalized, arora2019fine}. The minimum eigenvalue of NTK has been used in previous work to demonstrate the global convergence of gradient descent, such as two-layer networks~\citep{du2018gradient}, and deep networks with polynomially wide layers~\citep{pmlr-v97-allen-zhu19a}. Besides, the minimum eigenvalue of NTK is also used to prove generalization bounds~\citep{arora2019fine} and memorization~\citep{montanari2020interpolation}. However, previous work mainly focuses on a limited set of architectures, e.g., fully-connected (FC) neural networks ~\citep{allen2018learning, bartlett2017spectrally} or residual neural networks \citep{7780459,huang2020deep}, in which a single activation function is used throughout the network. These off-the-shelf theoretical results cannot be directly applied to analyze the rich search space (of NAS) that is covering various/mixed architectures and parameters. That makes the non-trivial analysis on NAS worth of study on its own right.

The recent work of \citet{pmlr-v139-oymak21a} is the first work to provide generalization guarantees on a related problem, i.e., activation functions search. The study provides generalization results on two-layer networks relying on the minimum eigenvalue with a strictly larger than zero assumption, i.e., $\lambda_{\min}(\bm K) > 0$ for the NTK matrix $\bm K$.

In this work, we introduce the first theoretical guarantees for multilayer NAS where the search space includes activation functions and skip connections. We study the upper/lower bound of the minimum eigenvalue of NTK (in the (in)finite regime) under mixed activation functions and architectures which evade the minimum eigenvalue assumption of \citet{pmlr-v139-oymak21a}. Then, we provide optimization and generalization guarantees of deep neural networks (DNNs) equipped with NAS. 
Our results indicate that the minimum eigenvalue estimation can act as a powerful metric for NAS. This method, called Eigen-NAS, is train-free, but still effective with experimental validation when compared to recent promising algorithms \citep{pmlr-v139-xu21m, chen2021neural, mellor2021neural}. Formally, our main contribution and findings are summarized below:

i) We build a general theoretical framework based on NTK for NAS with search on popular activation functions in each layer, fully-connected, and skip connections. We derive the NTK formula of these architectures in the (in)finite-width regime under the unifying framework.

ii) We derive the upper and lower bounds of the minimum eigenvalue of the NTK under the (in)finite-width regime for the considered architectures. We introduce a new technique to ensure the probability of concentration inequality remains positive. Our analysis highlights how the upper and lower bounds differs under activation function search and skip connection search and can guide NAS. 

iii) We establish a connection between the minimum eigenvalue and generalization of the searched DNN trained by stochastic gradient descent (SGD). Our theoretical results show that the generalization performance largely depends on the minimum eigenvalue of NTK for NAS, which provides theoretical guarantees for the searched architecture.

iv) Our theoretical results are supported by thorough experimental validations with the following findings: 1) our upper and lower bounds on the minimum eigenvalue largely depend on the activation function in the first layer rather than the activation functions in deeper layers. 2) The applied NAS algorithm always picks up ReLU (Rectified Linear Unit) and LeakyReLU in the optimal architecture, which coincides with our theory that predicts ReLU and LeakyReLU achieve the largest minimum eigenvalues. 3) The skip connections are required in each layer under our not very large DNNs. Furthermore, our experimental evidence on Eigen-NAS indicates that the minimum eigenvalue is a promising metric to guide NAS (without training) as suggested by our theory.

{\bf Technical challenges.} 
The technical challenges of this paper mainly focus on how to analyze activation functions with different properties and skip connections under a unifying framework. This work is non-trivial; previous works mainly focus on the ReLU activation function~\citep{pmlr-v139-nguyen21g, cao2019generalization, pmlr-v97-allen-zhu19a} in optimization and generalization of a single fully-connected neural network. Their proofs heavily depend on the properties of $\mathrm{ReLU}$, e.g., homogeneity and $\mathrm{ReLU}(x) = x\mathrm{ReLU}'(x)$ which are invalid when other commonly-used activation functions, e.g., Tanh, Sigmoid, and Swish, are used. 
This problem becomes harder when mixed activation functions and residual connections are considered.
To tackle these technical challenges, we develop the following techniques: a) to handle the non-homogeneous property of Tanh, Sigmoid, and Swish, we develop a new integral estimation approach for the minimal eigenvalue estimation. b) To establish the connection between the minimum eigenvalues of NTK and generalization errors, we use the Lipschitz continuity to avoid the special property of ReLU.
More importantly, we introduce a new way to use Gershgorin circle theorem for minimum eigenvalue estimation, which avoids concentration inequalities with negative probability in some certain cases \citep{pmlr-v139-nguyen21g}.\looseness-1
\section{Related work}
\label{sec:related_work}

{\bf Network architecture search (NAS):} The idea of NAS stems from \citet{45826}, while the idea of cell search, i.e., searching core building blocks and composing them together, emerged in \citet{zoph2018learning}. The earlier literature used discrete optimization techniques for obtaining the architecture. DARTS~\citep{liu2019darts} considers NAS as a continuous bi-level optimization task. Recent variants of DARTS~\citep{xu2019pc, wu2019fbnet} and several train-free methods~\citep{mellor2021neural, chen2021neural, pmlr-v139-xu21m} have demonstrated success in reducing the search time or improving the search algorithm. However, the aforementioned works have not provided generalization guarantees for the optimal architecture.

{\bf Optimization and generalization of DNNs via NTK:} In the NTK framework \citep{jacot2018neural,du2019gradient,chen2020much}, the training dynamics of (in)finite-width networks can be exactly characterized by kernel tools. Leveraging NTK facilitates studies on the global convergence of GD~\cite{pmlr-v97-allen-zhu19a, du2019gradient, nguyen2021proof} in DNNs via the minimum eigenvalue of NTK. In fact, it also controls the generalization performance of DNNs \citep{du2018gradient,cao2019generalization,allen2018learning}, which is further studied in \citet{bachspaper}.
\section{Problem Settings}
\label{sec:preliminaries}
In this section we introduce the problem setting of our NAS framework based on the search space and algorithm (search strategy) for our paper.

Let $X \subseteq \mathbb{R}^d$ be a compact metric space and $Y \subseteq \mathbb{R}$. We assume that the training set $\mathcal{D}_{tr} = \{  (\bm x_i, y_i) \}_{i=1}^N $ is drawn from a probability measure $\mathcal{D}$ on $X \times Y$, with its marginal data distribution denoted by $\mathcal{D}_X$.
The goal of a supervised learning task is to find a hypothesis (i.e., a neural network used in this work) $f: X \rightarrow Y$ such that $f(\bm x; \bm W)$ parameterized by $\bm W$
is a good approximation of the label $y \in Y$ corresponding to a new sample $\bm x \in X$.
In this paper, we consider the classification task, evaluated by minimizing the expected risk 

\begin{equation*}
\min_{\bm{W}}~\ell_{\mathcal{D} }(\bm{W}):= \mathbb{E}_{(\bm {x},y)\sim \mathcal{D}}~ \ell [y f(\bm {x};\bm W)]\,,
\end{equation*}

where $\ell [y f(\bm {x};\bm W)]$ is the classification loss $\ell(\cdot)$ as a surrogate of the expected 0-1 loss $\ell_{\mathcal{D} }^{0-1}(\bm W):=\mathbb{E}_{(\bm{x},y)\sim \mathcal{D} }[1\left \{ y f(\bm{x}; \bm W)<0 \right \} ]$. In this paper, we employ the cross-entropy loss, which is
defined as $\ell (z) = \log[1+\exp(-z)]$.

\emph{Notation:}  For an integer $L$, we use the shorthand $[L] = \left \{ 1,2,\dots ,L \right \}$. The multivariate standard Gaussian distribution is $\mathcal{N}(\bm 0, \mathbb{I}_d)$ with the zero-mean vector $\bm 0$ and the identity-variance matrix $\mathbb{I}_d$. We denote the direct sum by $\oplus$. We follow the standard Bachmann–Landau notation in complexity theory e.g., $\mathcal{O}$, $o$, $\Omega$, and $\Theta$ for order notation.

\subsection{Neural Networks and Search Space}
\label{ssec:search_space}

In this work, we consider a particular parametrization of $f$ as a deep neural network (DNN) with depth $L$ ($L \geq 3$)\footnote{Our results hold for the $L=2$ setting corresponding to one-hidden layer neural network with slight modifications on notation, so we focus on $L \geq 3$ for simplicity.} which includes the fully-connected (FC) neural networks setting and the residual neural networks setting, and various activation functions in each layer. This enables a quite general NAS setting.
Formally, we define a single-output DNN with the output $\bm{f}_l(\bm{x})$ in each layer
\begin{equation}
\begin{matrix}
\bm{f}_l(\bm{x}) \!=\! \left\{\begin{matrix}
\bm{x}  & l=0\,,\\
\sigma_1(\bm{W}_1 \bm{x})  & l=1\,,\\
\sigma_l(\langle \bm{W}_l, \bm{f}_{l-1}(\bm{x}) \rangle) \!+\! \alpha_{l-1}\bm{f}_{l-1}(\bm{x})  & 2\!\leq\! l \!\leq\! L\!-\!1, \\
\left \langle \bm{W}_L, \bm{f}_{L-1}(\bm{x}) \right \rangle & l=L\,,
\end{matrix}\right. \\
\end{matrix}
\label{eq:network}
\end{equation}
where the weights of the neural networks are $\bm{W}_1 \in \mathbb{R}^{m \times d} $, $\bm{W}_l \in \mathbb{R}^{m\times m} $, $l = 2,\dots ,L-1$ and $\bm{W}_L \in \mathbb{R}^{m}$. The binary parameter $\alpha_l$ is for layer search, and the activation function is $\sigma_{l}(\cdot)$. 
The neural network output is $f(\bm x; \bm W) = f_L(\bm x)$.

{\bf Architecture search:} A binary vector $\bm \alpha = [\alpha_1, \alpha_2, \cdots, \alpha_{L-2}]^{\!\top} $ represents the skip connections, where the $\alpha_l \in \{0,1 \}$ in~\cref{eq:network} indicates whether there is a skip connection in the $l$-th layer.
Notice that we unify FC and residual neural networks under the same framework.

{\bf Activation function search:} 
We select five representative activation functions defined by $\mathcal{F}_{\sigma} = \{ \mathrm{ReLU}, \mathrm{LeaklyReLU}, \mathrm{Sigmoid}, \mathrm{Tanh}, \mathrm{Swish} \}$ used in~\cref{eq:network},  that can be bounded, unbounded, smooth, non-smooth, monotonic, or non-monotonic, as reported in~\cref{tab:activation_functions}.
We define $\bm \sigma =[\sigma_1, \sigma_2, \cdots, \sigma_{L-1}]^{\!\top}$ with $\sigma_l \in \mathcal{F}_{\sigma}$ for any $l \in [L-1]$ as the indicator to show which activation function is selected in each layer.
Our NAS framework allows for a different activation function in each layer, which enlarges the search space.

In our setting, we conduct the architecture search and the skip connection search independently, and accordingly, our search space is defined as the direct sum of them:
\begin{equation}
\mathcal{W} : =  \mathbb{R}^{L-2} \oplus \mathcal{F}_{\sigma}^{L-1}  \oplus \{ \mathbb{R}^{m\times d} \times (\mathbb{R}^{m\times m})^{L-2}\times \mathbb{R}^{m} \}  \,,  
\label{eq:problem_setting_parameter_space}
\end{equation}
where $\bm{W} := ( \bm \alpha, \bm \sigma, \bm{W}_1, \dots , \bm{W}_L) \in \mathcal{W}$ represents the collection of weight matrices and indicator for skips and selected activation functions for all layers.

\begin{table*}[t]
\small
\centering
\begin{threeparttable}
\caption{Formula of different activation functions, definitions of relevant constants and some intermediate results.}
\label{tab:activation_functions}
\setlength{\tabcolsep}{2.5mm}{
\begin{tabular}{c|c|c|c|c|c}
    \toprule[1pt]
    $\sigma_{l}$ & ReLU & LeakyReLU & Sigmoid\tnote{[1]} & Tanh\tnote{[2]} & Swish\\
    \midrule
    Formula & $\max(0,x)$  & $\max(\eta x,x),~\eta \in (0,1)$ & $\frac{1}{1+e^{-x}}-\frac{1}{2}$ & $\frac{e^{x}-e^{-x}}{e^{x}+e^{-x}}$& $\frac{x}{1+e^{-x}}$  \\[5pt]
    \hline
    $\beta_1(\sigma_{l})$ & $1$  & $1+\eta^2$ & $1/8$ & $2$& $1$  \\
    
    $\beta_2(\sigma_{l})$ & $1$  & $1+\eta^2$ & $1/8$ & $2$& $1.22$  \\
    
    $\beta_3(\sigma_{l})$ & $1$  & $1+\eta^2$ & $f_{S}(t)$ & $f_{T}(t)$& $1/2$  \\

\bottomrule
\end{tabular}}
	\begin{tablenotes}
				\footnotesize
				\item[{[1]}] We consider the integral $f_{\mathrm{S}}(y) =\int_{-\infty}^{\infty}\frac{2}{\sqrt{2\pi y}}e^{-\frac{x^2}{2y}}{f'_{\mathrm{Sigmoid}}}(x)^2\mathrm{d}x.$ We add $-1/2$ in Sigmoid to ensure $f_{\mathrm{Sigmoid}}(0)=0$ facilitates our theoretical analysis. The parameter is $t := 3(1+\eta^2)(2+\eta^2)^{L-3}$.
			\end{tablenotes}
\end{threeparttable}
\end{table*}

\subsection{Algorithm (Search Strategy)}
\label{ssec:algorithm}

The search strategy is the core part in NAS to pick up the optimal architecture from the search space. 
Here we build a general Algorithm~\ref{alg:algorithm_DARTS} combining the search strategy for NAS (the first part) and the subsequent neural network training by SGD (the second part). 

We firstly utilize a typical NAS algorithm, e.g., random search WS~\citep{li2020random} or DARTS\footnote{This algorithm directly outputs the final optimal architecture and optimal parameters.}, to search skip connections and activation functions independently, which results in the optimal architecture $\{ (\sigma^*_i)_{i=1}^{L-1}, (\alpha^*_i)_{i=1}^{L-2} \}$ with the max probability, see~\cref{ssec:DARTS_experiment} for details.
In particular, Algorithm~\ref{alg:algorithm_DARTS} also allows for the guidance of NAS in a  train-free strategy via some specific metrics, e.g., the minimum eigenvalue of NTK (and its variant), see our Eigen-NAS method in~\cref{ssec:NAS_201_experiment}.

Then, we conduct neural network training on the selected architecture by SGD. For ease of theoretical analysis, we employ the constant step-size SGD with one epoch and randomly choose the weight parameters during all the iterations, which is commonly used in deep learning theory~\citep{cao2019generalization,Zou2019GradientDO}.

\begin{algorithm}[t]
\caption{SGD for training DNNs by NAS}
\label{alg:algorithm_DARTS}
\begin{algorithmic}
\STATE {\bfseries Input:} search space $\mathcal{S}$, data $\mathcal{D}_{tr} = \{ (\bm x_i, y_i)_{i=1}^N \}$, step size $\gamma$ and ${\tt Flag}_{\mbox{\tiny method}} \in \left \{\tt EigenNAS, DARTS, \cdots \right \} $.\\
// {\tt conduct NAS algorithms}\\
\IF{${\tt Flag}_{\mbox{\tiny GuideNAS}} = \tt EigenNAS$}
\STATE Guide NAS from $\mathcal{S}$ by our Eigen-NAS algorithm.
\ELSIF{${\tt Flag}_{\mbox{\tiny GuideNAS}} = \tt DARTS$} \STATE Search neural network architectures from $\mathcal{S}$ using the DARTS algorithm.
\ENDIF
\STATE Output the optimal architecture $\{ (\sigma^*_i)_{i=1}^{L-1}, (\alpha^*_i)_{i=1}^{L-2} \} \in \mathcal{S}$ with max probability.\\
// {\tt do neural network training via SGD}\\
\STATE   Gaussian initialization:
$\bm{W}_l^{(1)} \sim \mathcal{N}(0,1/m)$, $l \in [L]$\\
\STATE Construct the neural network $f(\bm x; \bm{W}_l^{(1)})$ based on $\{ (\sigma^*_i)_{i=1}^{L-1}, (\alpha^*_i)_{i=1}^{L-2} \}$ 
\FOR{$i=1$ {\bfseries to} $N$}
\STATE{
$\bm{W}^{(i+1)} = \bm{W}^{(i)}- \gamma\cdot \nabla_{\bm{W}} \ell \big( f(\bm x_i; \bm{W}^{(i)}) y_i \big)\,.$} 
\ENDFOR \\
\textbf{Output} 
Randomly choose $\hat{\bm{W}}$ uniformly from $\left \{ \bm{W}^{(1)}, \dots , \bm{W}^{(N)} \right \} $.
\end{algorithmic}
\end{algorithm}

\section{Main result}
\label{ssec:main_result}

In this section, we state the main theoretical results. We present the assumptions used in our proof in~\cref{ssec:assumptions}. Then in~\cref{ssec:Recursive_NTK} we provide the recursive form of NTK for DNNs defined by~\cref{eq:network} with mixed activation functions and skip connections.
The upper and lower bounds of the minimum eigenvalue of NTK in the infinite and finite-width setting is given in~\cref{ssec:Minimum_eigenvalue_NTK_infinite} and~\ref{ssec:Minimum_eigenvalue_NTK_finite}, respectively.
Finally, in~\cref{ssec:NTK_Connection_to_Generalization}, we connect the minimum eigenvalue of NTK and the generalization error bound of DNNs under these search schemes.
The proofs of our theoretical results presented in this section are deferred to Appendix~\ref{sec:infinitely_width},~\ref{sec:finitely_width}, and~\ref{sec:Relationship_NTK_Generalization}, respectively.

\subsection{Assumptions}
\label{ssec:assumptions}

We make the following assumptions on data and activation functions. Our assumptions are frequently employed in the literature as we highlight below. 

\begin{assumption}
\label{assumption:distribution_1}
The training data $\bm{x}_1,\cdots,\bm{x}_n$ are i.i.d. sampling from a distribution under the $\ell_2$ normalization $\| \bm{x}_i \|_{2} = 1$ for any $i \in [n]$.
Besides, we assume that with probability 1, for any $i\neq j$, $\bm{x}_i\nparallel \bm{x}_j$, i.e., $\max_{i\neq j} \left \langle \bm{x}_i, \bm{x}_j \right \rangle \leq C_{\max} < 1$.
\end{assumption}

\begin{assumption}
\label{assumption:activation_functions}
The activation function $\sigma: \mathbb{R} \rightarrow \mathbb{R}$ satisfies $\sigma \in L^2(\mathbb{R}, e^{-x^2/2}/\sqrt{2\pi})$, where $L^2(\mathbb{R}, g)$ denotes the square integrable function.
\end{assumption}

{\bf Remark:} The first assumption on normalized data is commonly used in practice and theory on over-parameterized neural networks~\citep{du2018gradient, du2019gradient, pmlr-v97-allen-zhu19a, oymak2020toward, pmlr-v119-malach20a} and no parallel data points is standard in statistics and machine learning~\citep{du2018gradient,du2019gradient}. The second assumption is general as the studied activation functions in Table~\ref{tab:activation_functions} satisfy it.

\subsection{Recursive NTK for DNNs defined by~\cref{eq:network}}
\label{ssec:Recursive_NTK}

Recall that NTK \citep{jacot2018neural} under the infinite-width setting ($m\to \infty$) is:
\begin{equation*}
K^{(L)}(\bm{x},\widetilde{\bm{x}}) := \mathbb{E}_{\bm{W}}\left \langle \frac{\partial f(\bm{x};\bm{W})}{\partial\bm{W}},\frac{\partial f(\widetilde{\bm{x}};\bm{W})}{\partial\bm{W}}  \right \rangle\,,
\end{equation*}
where the NTK matrix for residual networks is derived by the following regular chain rule.

\begin{lemma}
\label{lemma:NTK_matrix_recursive_form}
For any $l \in [3, L]$ and $s \in [2, L]$, denote
\begin{equation*}
\small
    \begin{split}
        & \bm{G}^{(1)}=\bm{XX}^\top\,,\quad\bm{A}^{(2)} = \bm{G}^{(2)}=2\mathbb{E}_{\bm{w} \sim \mathcal N(\bm 0,\mathbb{I}_{d})}[\sigma_1(\bm{Xw})\sigma_1(\bm{Xw})^\top]\,,\\
        & \bm{G}^{(l)}=2\mathbb{E}_{\bm{w} \sim \mathcal N(\bm 0,\mathbb{I}_{N})}[\sigma_{l-1}(\sqrt{\bm{A}^{(l-1)}} \bm{w})\sigma_{l-1}(\sqrt{\bm{A}^{(l-1)}} \bm{w})^\top]\,,\quad\bm{A}^{(l)}=\bm{G}^{(l)}+\alpha_{l-2}\bm{A}^{(l-1)}\,,\\
        & \dot{\bm{G}}^{(s)} = 2\mathbb{E}_{\bm{w} \sim \mathcal N(\bm{0},\mathbb{I}_{N})}[{\sigma}'_{s-1}(\sqrt{\bm{A}^{(s-1)}} \bm{w}){\sigma}'_{s-1}(\sqrt{\bm{A}^{(s-1)}} \bm{w})^\top]\,.
    \end{split}
\end{equation*}

Then the NTK for residual networks defined in~\cref{eq:network} can be written as
\begin{equation*}
    \bm{K}^{(L)}=\bm{G}^{(L)} + \sum_{l=1}^{L-1}\bm{G}^{(l)}\circ \dot{\bm{G}}^{(l+1)} \circ (\dot{\bm{G}}^{(l+2)}+  \alpha_{l}\bm{1}_{N \times N})\circ \cdots \circ (\dot{\bm{G}}^{(L)}+ \alpha_{L-2}\bm{1}_{N \times N})\,.
\end{equation*}

\end{lemma}

{\bf Remark:}
$(i)$ Our NTK formula of ResNet differs from the one of \citet{tirer2021kernelbased, huang2020deep, belfer2021spectral} in two critical ways:
1) each skip-layer in our model skips one fully-connected layer and one activation function, as opposed to the two-layer skip of previous works, 2) our formulation does not require every layer to have a parallel skip connection, which increases the flexibility of the network. Those differences also result in a different NTK matrix.\\
$(ii)$ Our NTK formulation covers different activation functions, and we adopt the same initialization (coefficient) on them to ensure fair/equal search in our NAS framework.

\cref{lemma:NTK_matrix_recursive_form} covers both FC and residual neural networks, which facilitates the analysis of the minimum eigenvalue of NTK under the unifying framework. If $\alpha_{l} = 0$ for $l \in [L-1]$, our NTK formulation for residual neural networks degenerates to that of a fully connected neural network, and $\bm{A}^{l}$ and $\bm{G}^{l}$ become equal.

\subsection{Minimum Eigenvalue of NTK for infinite-width}
\label{ssec:Minimum_eigenvalue_NTK_infinite}

We are now ready to state the main result of the infinite-width neural network. We provide the upper and lower bounds of the minimum eigenvalue of NTK for an infinite-width neural network mixed with five different activation functions. The main differences between different activation functions are illustrated in~\cref{tab:activation_functions}.

\begin{theorem}
\label{thm:lambda_min_inf_mixed}
For a DNN defined by~\cref{eq:network} and a not very large $L$, let $\bm{K}^{(L)}$ be the limiting NTK recursively defined in~\cref{lemma:NTK_matrix_recursive_form}. Then, 
under Assumptions~\ref{assumption:distribution_1}, choose $r \geq \frac{\log (2n)}{1-C_{\text{max}}}$, we have
\begin{equation*}
\lambda _{\min}(\bm{K}^{(L)}) \geq \mu_{r}(\sigma_1)^{2}\prod_{p=3}^{L}\Bigg(\beta_3(\sigma_{p-1})+\alpha_{p-2}\Bigg)\,,
\end{equation*}
\begin{equation*}
\lambda _{\min}(\bm{K}^{(L)}) \leq\sum_{l=1}^{L}\Bigg(\beta_1(\sigma_{l-1}) \prod_{p=2}^{l-1} \big(\beta_1(\sigma_{p-1})+\alpha_{p-2}) \big) \prod_{p=l+1}^{L}(\beta_2(\sigma_{p-1})+\alpha_{p-2}) \Bigg) \,, 
\end{equation*}
where $\mu_r(\sigma_1)$ is the $r$-st Hermite coefficient of the first layer activation function, and $\beta_1, \beta_2, \beta_3$ are three constants on various activation functions defined in~\cref{tab:activation_functions}.
\end{theorem}

{\bf Remark:} 
A not very large depth, e.g., $L \leq 10$, is often sufficient for the search phase in practical implementations~\citep{liu2018hierarchical, dong2021nats}. In addition, existing NAS algorithms such as DARTS tend to have architectures with wide and shallow cell structures as suggested by~\citet{Shu2020Understanding}. \cref{thm:lambda_min_inf_mixed} shows the upper and lower bounds of the minimum eigenvalue of NTK under the mix of activation functions and skip connections. The following conclusions can be drawn from our results:

\begin{enumerate}[leftmargin=3.3mm]
    \itemsep-0.2em 
    \item The bounds of the minimum eigenvalue depend significantly on the depth of the network $L$, the skip connections via $\alpha_p$, which makes the minimum eigenvalue increase fast as $L$ and the number of skip connections increase. Besides, the minimum eigenvalue is also affected by activation functions via $\beta_1, \beta_2, \beta_3$. Nevertheless, the lower bound is independent of $\beta_1$ and $\beta_2$. 
    \item Different activation functions lead to different tendencies (increase or decrease) on $\lambda_{\min} (\bm{K}^{(L)})$. As the depth increases, the lower bound $\lambda_{\min} (\bm{K}^{(L)})$ under ReLU remains unchanged, increases under LeakyReLU, and decreases when Sigmoid, Tanh or Swish applied, which brings in new findings when compared to the ReLU-network analysis of~\citet{pmlr-v139-nguyen21g}. For the upper bound for $\lambda_{\min} (\bm{K}^{(L)})$, we can see our results are positively correlated with the depth $L$. 
    \item One can see that $\mu_1(\sigma_1)$ is only related to the activation function of the first layer, which implies that the activation function in the first layer is very important as $\lambda_{\min} (\bm{K}^{(L)})$ largely depends on it.
\end{enumerate}

\subsection{Minimum Eigenvalue of NTK for finite-width}
\label{ssec:Minimum_eigenvalue_NTK_finite}
To study the finite-width, we firstly introduce the Jacobian of the network. Let $\bm{F}=[f(\bm{x}_1) ,\ldots,f(\bm{x}_N)]^T$. Then, the Jacobian $\bm{J}$ of $\bm{F}$ with respect to $\bm{W}$ is $\bm{J} =\left[\frac{\partial \bm{F}}{\partial\text{vec}(\bm{W}_1)},\ldots,\frac{\partial \bm{F}}{\partial\text{vec}(\bm{W}_L)}\right]$,
where $\bm{J}$ have dimension $\mathbb{R}^{(((L-2)\times m+d+1)\times m\times N}$.
The empirical Neural Tangent Kernel (NTK) matrix can be defined as $\bar{\bm{K}}^{(L)} =\bm{JJ}^{\top} =\sum_{l=1}^{L} \left[\frac{\partial \bm{F}}{\partial\text{vec}(\bm{W}_l)}\right] \left[\frac{\partial \bm{F}}{\partial\text{vec}(\bm{W}_l)}\right]^{\top}$.

Accordingly, we generalize~\cref{thm:lambda_min_inf_mixed} from the infinite-width to finite-width setting below.
\begin{theorem}
\label{thm:lambda_min_finite}

For an $L$-layer network defined by~\cref{eq:network}, let $\bm{K}^{(L)} = \bm{JJ}^{\top}$ be the NTK matrix, and the weights of the network be initialized as $[\bm{W}_l]_{i,j}\sim \mathcal{N} (0,1/m)$, for all $l \in [L]$. Under Assumptions~\ref{assumption:distribution_1}, with probability at least $1-\sum_{l=1}^{L-1}\exp(-\Omega (m))-\exp(-\Omega (1))$, $\lambda_{\min}(\bm{JJ}^{\top})$ can be bounded by:

\begin{equation*}
    \Theta \bigg(\prod_{i=2}^{L-1}(\beta_3(\sigma_i)+\alpha_{i-1} ) \bigg) \leq \lambda_{\min}(\bm{JJ}^{\top})\leq \sum_{k=0}^{L-1}\Theta \left(\prod_{i=k+2}^{L-1}(\beta_2(\sigma_i)+\alpha_{i-1} ) \right)\,,
\end{equation*}
where the definitions of $\beta_2$, and $\beta_3$ are the same as those in~\cref{thm:lambda_min_inf_mixed}.

\end{theorem}
{\bf Remark:} 
\cref{thm:lambda_min_finite} achieves a similar result as~\cref{thm:lambda_min_inf_mixed} if the width $m$ is large.

\subsection{Connection to Generalization Error Bound}
\label{ssec:NTK_Connection_to_Generalization}

Based on the aforementioned upper and lower bounds of the minimum eigenvalue of NTK under different settings, here we establish its relationship with the generalization error of DNNs. We provide a bound on the expected 0-1 error obtained by Algorithm~\ref{alg:algorithm_DARTS}.

\begin{theorem}
\label{thm:NTK_Generalization}

Given a DNN defined by~\cref{eq:network} with $\bm{y} = (y_1, \dots, y_N)^{\top}$ determined by Algorithm~\ref{alg:algorithm_DARTS} with the step size of SGD $\gamma = \kappa C_1 \cdot \sqrt{\bm{y}^{\top}({\bm{K}^{(L)}})^{-1}\bm{y}}/({m}\sqrt{N}) $ for some small enough absolute constant $\kappa$. Under Assumptions~\ref{assumption:distribution_1} and~\ref{assumption:activation_functions}, for any $\delta \in \left ( 0,e^{-1} \right ]$ and a not very large $L$, if the width $m \geq \hat{m}$, where $\hat{m}$ depends on $ \lambda_{\min}(\bm{K}^{(L)}), \delta, N$, and $L$, then with probability at least $1-\delta$ over the randomness of $\bm{W}^{(1)}$, we obtain the following high probability bound:
\begin{equation*}
\mathbb{E}[\ell_{\mathcal{D} }^{0-\!1}\!(\hat{\bm{W}}\!)] \leq \tilde{\mathcal{O} } \left(\! C_2\sqrt{\frac{\bm{y}^{\top}  ({\bm{K}^{(L)}})^{-1} \bm{y}}{N}} \!\right) + \mathcal{O}\!\left(  \sqrt{\frac{\log(1/\delta )}{N} }  \right)\,,
\end{equation*}
where $C_1 = \sqrt{L}/(3\mathrm{Lip}_{\max}+1)^{L-1}$ and $C_2 = \sqrt{L}(3\mathrm{Lip}_{\max}+1)^{L-1}$ are two constants depending only on $L$ and $\mathrm{Lip}_{\max}$ is the maximum value of the Lipschitz constants of the all activation functions.
\end{theorem}

{\bf Remark:} According to the courant minimax principle~\citep{10.5555/248979}: $\frac{1}{\lambda_{\min}(\bm{K}^{(L)})} =\lambda_{\max}(({\bm{K}^{(L)}})^{-1}) = \max \frac{\bm{y}^{\top}({\bm{K}^{(L)}})^{-1})\bm{y}}{\bm{y}^{\top}\bm{y}}$, that means $\bm{y}^{\top}(({\bm{K}^{(L)}})^{-1})\bm{y} \leq \frac{\bm{y}^{\top}\bm{y}}{\lambda_{\min}(\bm{K}^{(L)})}$, then the minimum eigenvalue plays a significant role in our analysis as well as our application on NAS.
The quantity $\bm{y}^{\top}(({\bm{K}^{(L)}})^{-1})\bm{y}$ can be independent of $N$ in some certain cases~\citep{arora2019fine}, leading to a classical $\mathcal{O}({N^{-1/2}})$ convergence rate for generalization.

\cref{thm:NTK_Generalization} gives an algorithm-dependent generalization error bound of DNNs defined by~\cref{eq:network} trained with SGD with different activation functions and skip connections.
If $m$ is large enough, the learning rate is infinitesimal, which means the generalization error bound mainly depends on the NTK matrix, similarly to \citet{cao2019generalization,du2019gradient}. Admittedly, our result is in an exponential increasing order of the depth. However, in practice, the depth $L$ during the search phrase is smaller than $20$, or even $10$~\citep{liu2018hierarchical, dong2021nats}. As we detail in~\cref{sec:discussion}, our results extend previously known results.

According to~\cref{thm:NTK_Generalization}, the generalization performance of DNNs is controlled by the minimum eigenvalue of the NTK matrix, which is in turn affected by different activation functions and skip connections, as discussed in~\cref{thm:lambda_min_inf_mixed}. Apart from the NTK matrix itself, the condition $m \geq \hat{m}$ is also affected by different activation functions, which implies that the required minimum width is different in these cases.

\subsection{Proof sketch}

Our work extends the proofs of \citet{NEURIPS2020_8abfe8ac,cao2019generalization} beyond ReLU, which is critical for enabling search across activations. 
The extension to other activation functions and skip connections is non-trivial due to non-linearity, inhomogeneity and nonmonotonicity.

To derive the upper and lower bounds on the minimum eigenvalue, we start from~\cref{lemma:NTK_matrix_recursive_form} on the NTK formula under the mixed activation functions and skip connections, and we transform the minimum eigenvalue estimation to the computation (estimation) of the bound $\bm G$, $\dot{\bm G}$ ($\lambda_{\min}(\bm G)$). The infinite-width and finite-width are included in~\cref{sec:infinitely_width} and~\ref{sec:finitely_width} respectively.
For the upper bound, we estimate the diagonal elements of $\bm G$ and use the property that the minimum eigenvalue is less than the mean of the diagonal elements of a matrix to prove.
For the lower bound, we use Hermite expansion. Combining these results concludes the proof.

To derive the generalization error bounds, we need a series of lemmas (see~\cref{sec:Relationship_NTK_Generalization}).  If the input weights are close, the output of each neuron with any activation function does not change too much (see~\cref{lemma:bound_for_perturbation}).
If the initilizations are close, the neural network output $f(\bm x; \bm W)$ is almost linear in $\bm W$ (see~\cref{lemma:lemma_4.1_in_GuQuanquan}), and the loss function $\ell[ y_i f(\bm{x}_i;\bm{W}) ]$ is almost a convex function of $\bm W$ for any $i \in [N]$ (see~\cref{lemma:lemma_4.2_in_GuQuanquan}).
Accordingly, the gradient and loss of the neural network can be upper bounded by Lemmas~\ref{lemma:lemma_B.3_in_GuQuanquan} and~\ref{lemma:lemma_4.3_in_GuQuanquan}, respectively, which concludes the proof when combined with some relevant results ~\citep{cao2019generalization,pmlr-v97-allen-zhu19a}.
Further discussion on the differences is deferred to \cref{sec:discussion}.
\section{Numerical Validation}
\label{sec:experiment}
To validate our theoretical results, we conduct a series of experiments on NAS. Firstly, we simulate the NTK matrices under different depths in~\cref{ssec:Simulation_of_NTK} to verify the relationship between the minimum eigenvalue of NTK and the network depth $L$ in~\cref{thm:lambda_min_inf_mixed}. In~\cref{ssec:DARTS_experiment} we use the DARTS algorithm~\citep{liu2019darts} to conduct experiments on activation function search and skip connection search under the search space of~\cref{eq:network}. Finally, we use the minimum eigenvalue of NTK to guide the training of NAS on the benchmark NAS-Bench-201~\citep{dong2020nasbench201}, with a comparison of recent NAS algorithms. Additional experiments on NAS-Bench-101~\citep{ying2019bench} and transfer learning are deferred to~\cref{ssec:bench101_experiment} and~\ref{ssec:transfer_learning_experiment}.

\subsection{DARTS experiment}
\label{ssec:DARTS_experiment}

In this section we employ a typical NAS algorithm, DARTS \citep{liu2019darts}, to assess our theoretical results on activation functions and skip connections. We select Fashion-MNIST~\citep{xiao2017fashion} as a standard benchmark. Details about Fashion-MNIST are shared in~\cref{ssec:dataset}.

{\bf Search space and search strategy:} Our search space is defined by~\cref{eq:problem_setting_parameter_space} on skip connections, activation functions, and weight parameters.
We follow the search strategy of \citet{liu2019darts} in a two-level scheme, one level is for weight parameter search $\bm W$ and the other level is for architecture search $\{\bm \alpha, \bm \sigma \} $, which results in the final optimal architecture $\{ \bm \alpha^*, \bm \sigma^*, \bm W^* \}$. 
Different from \citet{liu2019darts}, the activation function search and the skip connection search in our setting is independent.
To obtain $\bm \sigma^*$, we use the softmax function to normalize the weights and choose the specific activation function with the highest probability in each layer.
To obtain $\bm \alpha^*$, we initialize each entry $\alpha_l = 1/2$ ($l\in[L-2]$), constrain it to $[0,1]$ during training, and retain the skip connection when $\alpha^*_l > 1/2$. 

\begin{figure}[t]
\centering
    \subfigure[activation functions $\bm \sigma$]{\includegraphics[width=0.4\linewidth]{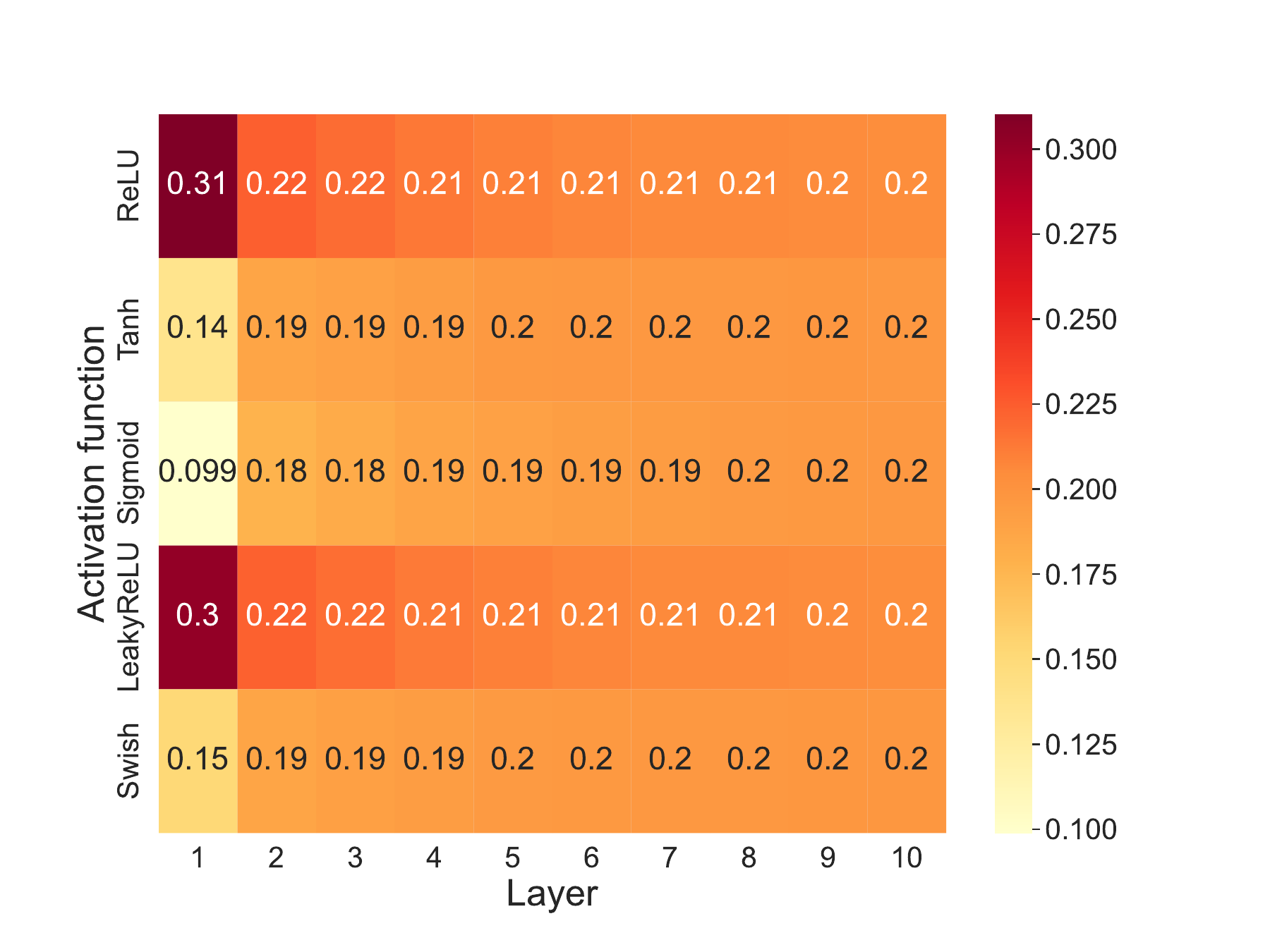}\label{fig:DARTS_heatmap1}\hspace{-1mm}}
    \subfigure[skip connections $\bm \alpha$]{\includegraphics[width=0.4\linewidth]{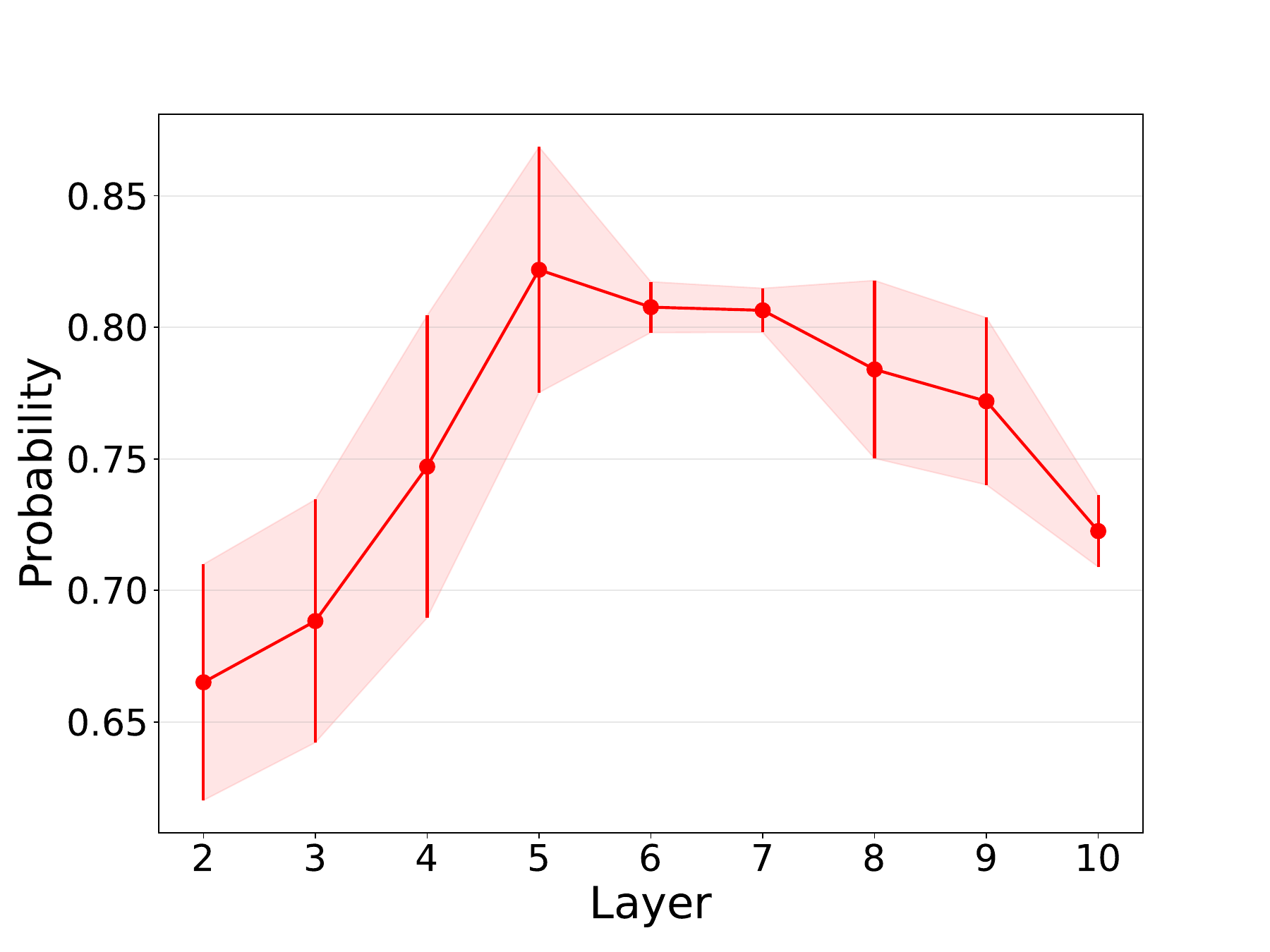}\vspace{-9mm}
    \label{fig:DARTS_heatmap2}}
\caption{Architecture search results on activation functions indicated by the probability of $\bm \sigma$ in (a) and skip connections indicated by $\bm \alpha$ in (b). We notice that for each layer, ReLU and LeakyReLU are selected a the higher probability.}
\label{fig:DARTS_heatmap}
\end{figure}

\begin{table*}[tb]
\centering
\caption{Results on CIFAR-10, CIFAR-100 and ImageNet-16 as part of NAS-Bench-201. The best performance is highlighted by {\bf bold}.
The results of NASWOT, TE-NAS and KNAS are reported from the corresponding papers. The results of ResNet, NAS-RL and DARTS are reported in ~\citep{pmlr-v139-xu21m}. The results illustrate that Eigen-NAS outperforms the prior art in CIFAR-100 and Imagenet-16. In particular, Eigen-NAS outperforms KNAS in all three cases when the same number of top-$k$ architectures are selected, i.e., $k=20$, and still achieves promising performance when smaller $k=5$ used, which we attribute to the more precise minimum eigenvalue estimation.} 
\begin{tabular}{l@{\hspace{0.25cm}} c@{\hspace{0.2cm}}c@{\hspace{0.2cm}}c@{\hspace{0.2cm}}c@{\hspace{0.2cm}} c} 
    \hline
    Type & Model/Algorithm & CIFAR-10 (\%) & CIFAR-100 (\%) & ImageNet-16 (\%)\\
    \hline
    w/o Search & ResNet ~\citep{7780459} & $\bm{93.97}$ & $70.86$ & $42.63$\\
    Search & NAS-RL ~\citep{45826}& $92.83$ & $70.71$ & $44.10$\\
    Gradient & DARTS ~\citep{liu2019darts} & $88.32$ & $67.34$ & $33.04$\\
    Train-free & NASWOT ~\citep{mellor2021neural} & $92.96$ & $70.03$ & $44.43$\\
    Train-free & TE-NAS ~\citep{chen2021neural} & $93.90$ & $71.24$ & $42.38$\\
    Train-free & KNAS ~\citep{pmlr-v139-xu21m} ($k=20$) & $93.38$ & $70.78$ & $44.63$\\
    Train-free & NASI (T)~\citep{shu2022nasi} & $93.08\pm0.24$ & $69.51\pm0.59$ & $40.87\pm0.85$\\
    Train-free & NASI (4T)~\citep{shu2022nasi} & $93.55\pm0.10$ & $71.20\pm0.14$ & $44.84\pm1.41$\\
    Train-free & \textbf{Eigen-NAS} ($k=20$) & $93.46\pm0.01$ & $\bm{71.42}\pm0.63$ & $\bm{45.54}\pm0.04$\\
    Train-free & \textbf{Eigen-NAS} ($k=5$) & $93.43\pm0.08$ & $69.92\pm1.82$ & $45.53\pm0.06$\\
    \hline
\end{tabular}
\label{tab:NAS_benchmark_experiment}
\end{table*}

{\bf NAS Results:} 
We conduct the experiment via DARTS on a feedforward neural network with $L=10$ and $m=1024$, with 5 runs.
After training, the probability of these activation functions and skip connections in each layer is reported in Figure~\ref{fig:DARTS_heatmap1} and~\ref{fig:DARTS_heatmap2}, respectively. We have the following findings:
Firstly, after the search process, LeakyReLU and ReLU are selected as the activations with the highest probability in each layer.
This coincides with our theoretical results in~\cref{thm:lambda_min_inf_mixed}. One minor difference is that the probability of LeakyReLU is slightly inferior to ReLU in practice.
The reason behind this could be the sparsity of ReLU ~\citep{dedios2020sparsity}.
Secondly, in the first layer, we observe the largest difference on the probability of various activation functions. As the network becomes deeper, the differences decrease with the last layers having no difference between different activation functions. This phenomenon matches our theory well. To be specific, in~\cref{thm:lambda_min_inf_mixed}, our result on the minimum eigenvalue largely depend on the first layer and its Hermite coefficient. Besides, this result also provides a justification on omitting the high-order terms while retaining the first layer activation terms. Thirdly, for the skip search result, we find that the skip connections are required in each layer when $L\leq 10$, as suggested by our theoretical results in~\cref{thm:lambda_min_inf_mixed}. It also verifies the results of~\citet{zhou2020theory}. We expect that the skip connections might not be required in each layer for deep neural networks, since their capacity can already be enough~\citep{7780459}; but we defer the related study to a future work.

Interestingly, the search strategy favors the activation functions and the skip connections with larger minimum eigenvalue of NTK, which enjoy better generalization performance.
This result also motivates us to study the following question: \emph{can the minimum eigenvalue of NTK guide the search process in NAS?}
We provide an affirmative answer in the next section with experimental validations.

\subsection{NAS-Bench-201 Experiment}
\label{ssec:NAS_201_experiment}

In this experiment, we use the minimum eigenvalue to guide NAS on NAS-Bench-201~\citep{dong2020nasbench201}. Each experiment is repeated 5 times, while it can run on a single GPU in a few hours. 

{\bf Benchmark and baselines:} NAS-Bench-201~\citep{dong2020nasbench201} is a commonly used benchmark for NAS algorithm evaluation, which includes three datasets: a) CIFAR-10~\citep{krizhevsky2014cifar}, b) CIFAR-100~\citep{krizhevsky2014cifar} and c) ImageNet-16 ~\citep{chrabaszcz2017downsampled} for image classification. Details on the datasets exist in~\cref{ssec:dataset}. Apart from that, we evaluate the proposed approach with some baselines including ResNet, DARTS, RL based algorithm and some train-free algorithms.

{\bf Algorithm procedure:} Our algorithm, called Eigen-NAS, also belongs in the train-free category. Eigen-NAS follows KNAS, which leverages the minimum eigenvalue of NTK to guide NAS. 
However, due to the $\mathcal{O}(N^3)$ time complexity of computing these eigenvalues, KNAS instead computes $\| \bm K \|_{\mathrm{F}}$. However, from the expression $\lambda_{\min} (\bm K) \leq \frac{1}{N}\sum_{i=1}^{N} K_{ii} \leq \| \bm K \|_{\mathrm{F}}$ we utilize the first inequality in Eigen-NAS to obtain a tighter (and more computationally efficient) bound to $\lambda_{\min}$. 
The computation cost of our method is $\mathcal{O}(N)$, which is less than computing the Frobenius norm ($\mathcal{O}(N^2)$). Sequentially, the top-$k$ best candidates architectures are chosen in KNAS and our Eigen-NAS, and then the best architecture is chosen by the validation error. Please refer to the results in~\cref{tab:NAS_benchmark_experiment}. 
Due to the page limit, the algorithm is located in~\cref{sec:additional_experiments}. \looseness-1

{\bf Results:} The experimental results in~\cref{tab:NAS_benchmark_experiment} verify that Eigen-NAS guided by the proposed metric above achieves the best performance on both the CIFAR-100 and ImageNet-16 datasets, and competitive performance on CIFAR-10, outperforming KNAS in all three cases when $k=20$ for both methods. Even when we consider a smaller $k=5$, Eigen-NAS can outperform KNAS, which we attribute to the more precise minimum eigenvalue estimation. 

\section{Conclusion}
\label{sec:conclusion}

In this work, we explore the relationship between the minimum eigenvalue of NTK and neural architecture search. We derive upper and lower bounds on the minimum eigenvalues of NTK for (in)finite residual networks under different mixtures of activation functions, and establish a connection between the minimum eigenvalues and the generalization properties of the special search space: activation function and skip connection search of NAS. Our theoretical results on various activation functions and mixed activation cases can also be a tool for deep learning theory researchers to prove generic results rather than studying a single architecture, e.g., ReLU networks. In addition, we use the minimum eigenvalue as a guide for the training of NAS in a train-free method, which greatly exceeds the efficiency of the classic NAS algorithm. When compared with existing train-free methods, our algorithm, called Eigen-NAS, achieves a higher accuracy. We posit that this will be useful for studying computationally efficient methods on NAS.

A core limitation is whether our proof framework can cover more general structures in NAS, such as the most commonly used convolutional neural networks (CNNs). Even though this seems possible, this is non-trivial due to the tensors that emerge. To be specific, it requires the element-recursive form of NTK matrices in~\citet{arora2019exact} to be transformed into a global-recursive form (similar to~\cref{lemma:NTK_matrix_recursive_form}), then analyze its minimum eigenvalue. Besides, the contraction operation of tensors, the locality and boundary effects of convolutional layer in CNNs make the analysis difficult. Therefore, we believe this is a topic on its own right.
Another limitation of our work is that it does not analyze the various algorithms proposed for searching through the search space. We believe that a deeper understanding of such algorithms, such as DARTS can provide further insights into how to design improved search spaces. 
In addition, the upper and lower bounds of the minimum eigenvalues of the NTK matrices for different activation functions given by ~\cref{thm:lambda_min_inf_mixed} have some overlaps, which means that our suggestions on activation functions selection based on these bound appear a bit vacuous in theory but still coincide with our experimental validations.
Maybe, a tighter bound without overlap for different activation functions is needed to address this theoretical issue.
 
\section*{Acknowledgements}
\label{sec:acks}

We are also thankful to the reviewers for providing constructive feedback. Research was sponsored by the Army Research Office and was accomplished under Grant Number W911NF-19-1-0404. This work was supported by Hasler Foundation Program: Hasler Responsible AI (project number 21043). This work was supported by SNF project – Deep Optimisation of the Swiss National Science Foundation (SNSF) under grant number 200021\_205011. This work was supported by Zeiss. This project has received funding from the European Research Council (ERC) under the European Union's Horizon 2020 research and innovation programme (grant agreement n° 725594 - time-data). Corresponding authors: Fanghui Liu and Zhenyu Zhu.

\newpage
\bibliography{literature}
\bibliographystyle{abbrvnat}

\newpage
\appendix
\onecolumn
\allowdisplaybreaks
\section*{Appendix introduction} 
\label{sec:appendix_intro}
The Appendix is organized as follows:
\begin{itemize}
    \item In ~\cref{sec:background}, we state the introductory notations and definitions.
    \item We prove~\cref{thm:lambda_min_inf_mixed} in~\cref{sec:infinitely_width}. We also provide the result when the residual network only has the same activation function.
    \item In~\cref{sec:finitely_width}, we extend the results of infinitely width to finite-width and provide the proof for them.
    \item In~\cref{sec:Relationship_NTK_Generalization}, we prove~\cref{thm:NTK_Generalization}.
    \item In~\cref{sec:discussion}, we discussion some key points of the proof and the motivation of the analysis.
    \item In \cref{sec:additional_experiments}, we detail our experimental settings, our Eigen-NAS algorithm as used in~\cref{ssec:NAS_201_experiment}. We conduct additional numerical validations.
    \item Finally, in~\cref{sec:nas_societal_impact}, we discuss the societal impact of this work.
\end{itemize}

\section{Background}
\label{sec:background}

\subsection{Symbols and Notation}
\label{sec:symbols_and_notations}
In the paper, vectors are indicated with bold small letters, matrices with bold capital letters. To facilitate the understanding of our work, we include the some core symbols and notation in \cref{table:symbols_and_notations}. 

\begin{table}[ht]

\caption{Core symbols and notations used in this project.}
\label{table:symbols_and_notations}
\small
\centering
\begin{tabular}{c | c | c}
\toprule
Symbol & Dimension(s) & Definition \\
\midrule
$\mathcal{N}(\mu,\sigma) $ & - & Gaussian distribution of mean $\mu$ and variance $\sigma$ \\
$1\left \{A\right \}$ & - & Indicator function for event $A$\\
$[L]$ & - & Shorthand of $\left \{ 1,2,\dots ,L \right \}$\\
$\oplus$ & - & Direct sum\\
$\mathcal{O}$, $o$, $\Omega$ and $\Theta$ & - & Standard Bachmann–Landau order notation\\
$\circ$ & - & Element-wise hadamard product\\
\midrule
$\left \| \bm{v} \right \|_2$ & - & Euclidean norms of vectors $\bm{v}$ \\
$\left \| \bm{M} \right \|_2$ & - & Spectral norms of matrices $\bm{M}$ \\
$\left \| \bm{M} \right \|_{\mathrm{F}}$ & - & Frobenius norms of matrices $\bm{M}$ \\
$\left \| \bm{M} \right \|_{\mathrm{\ast}}$ & - & Nuclear norms of matrices $\bm{M}$ \\
$\lambda(\bm{M})$ & - & Eigenvalues of matrices $\bm{M}$ \\
$\bm{M}^{[l]}$ & - & $l$-th row of matrices $\bm{M}$\\
$\bm{M}_{i,j}$ & - & $(i, j)$-th element of matrices $\bm{M}$\\
\midrule
$N$ & - & Size of the dataset \\
$d$ & - & Input size of the network \\
$L$ & - & Depth of the network \\
$m$ & - & Width of intermediate layer\\
$\alpha_l$ & $\mathbb{R}$ & A binary variable measures whether there is a skip connection in the $l$-th layer \\
$\sigma_l$ & - & The activation function of $l$-th layer \\
$\beta_1, \beta_2, \beta_3$ & $\mathbb{R}, \mathbb{R}, \mathbb{R}$ & Three constants defined in~\cref{tab:activation_functions} \\
$\mu_i(\sigma)$ & $\mathbb{R}$ & The $i$-th Hermite coefficient of the activation function $\sigma$\\
\midrule
$\bm{x}_i$ & $\mathbb{R}^{d}$ & The $i$-th data point \\
$y_i$ & $\mathbb{R}$ & The $i$-th target vector \\
$\bm{W}_1$ & $\mathbb{R}^{m \times d}$ & Weight matrix for the input layer \\
$\bm{W}_l$ & $\mathbb{R}^{m \times m}$ & Weight matrix for the $l$-th hidden layer \\
$\bm{W}_L$ & $\mathbb{R}^{1 \times m}$ & Weight matrix for
the output layer \\
\midrule
\end{tabular}
\end{table}

\subsubsection{Feature map}
\label{sssec:problem_setting_feature_map}

Here we define the core notation about feature maps that are required in the proof. Firstly, we define $\omega $-neighborhood to describe the difference between two matrices.

For any $\bm{W} \in \mathcal{W} $, we define its $\omega $-neighborhood as follows:

\begin{definition}[$\omega $-neighborhood]
\label{def:omega_neighborhood}
\begin{equation*}
\mathcal{B} (\bm{W},\omega ) :=\left \{ \bm{W}'\in \mathcal{W}: \left \| \bm{W}'_l - \bm{W}_l \right \|_{\mathrm{F}} \leq \omega ,\bm{\alpha}'=\bm{\alpha},\bm{\sigma}'=\bm{\sigma}, l \in [L]   \right \}\,.
\end{equation*}
\end{definition}

Then we define $(\bm{D}_{l})_{k,k} = {\sigma_l}'((\bm{W}_l \bm{f}_{l-1})_k)$ as the back-propagation matrix of the activation function. We use the notation $\widetilde{\bm{W}} \in \mathcal{B} (\bm{W},\omega )$ to describe the relationship of the two matrices have $\omega $-neighborhood relationship.

In addition, we define the feature map of network and its perturbing matrix as follows:
\begin{definition}
\label{def:problem_setting_feature_map_and_perturbation}
\begin{equation*}
\small
\begin{matrix}
\widetilde{\bm{g}}_{i,1} = \widetilde{\bm{W}}_1 \bm{x}_i \,, & \bm{g}_{i,1} = \bm{W}_1\bm{x}_i \,, & for\ i \in [N],\\
\widetilde{\bm{f}}_{i,1} = \sigma_1  (\widetilde{\bm{W}}_1\bm{x}_i) \,, & \bm{f}_{i,1} = \sigma_1 (\bm{W}_1\bm{x}_i) \,,& for\ i \in [N],\\
\widetilde{\bm{g}}_{i,l} = \widetilde{\bm{W}}_{l}\widetilde{\bm{f}}_{i,l-1} \,,   & \bm{g}_{i,l} = \bm{W}_{l}\bm{f}_{i,l-1} \,, & for\ i \in [N] \ \text{and}\  l \!=\! 2,\dots, L-1,\\
\widetilde{\bm{f}}_{i,l} = \sigma_l (\widetilde{\bm{W}}_{l}\widetilde{\bm{f}}_{i,l-1})+\alpha_{l-1}\widetilde{\bm{f}}_{i,l-1} \,,  & \bm{f}_{i,l} = \sigma_l (\bm{W}_{l}\bm{f}_{i,l-1})+\alpha_{l-1}\bm{f}_{i,l-1} \,, & for\ i \in [N] \ \text{and}\  l \!=\! 2,\dots, L-1\,.
\end{matrix}
\end{equation*}

Let us define diagonal matrices $\widetilde{\bm{D}}_{i,l} \in \mathbb{R}^{m\times m} $ and $\bm{D}_{i,l} \in \mathbb{R}^{m\times m} $ by letting $(\widetilde{\bm{D}}_{i,l})_{k,k} = {\sigma_l}'((\widetilde{\bm{g}}_{i,l})_k) $ and $(\bm{D}_{i,l})_{k,k} = {\sigma_l}'((\bm{g}_{i,l})_k) $, $\forall k \in [m]$. Accordingly, we let $\hat{\bm{g}}_{i,l} = \widetilde{\bm{g}}_{i,l} - \bm{g}_{i,l}$, $\hat{\bm{f}}_{i,l} = \widetilde{\bm{f}}_{i,l} - \bm{f}_{i,l}$ and diagonal matrix $\hat{\bm{D}}_{i,l} = \widetilde{\bm{D}}_{i,l} - \bm{D}_{i,l}$.

\end{definition}

\subsubsection{Other notations}
\label{sssec:other_notations}

For the Hadamard product of the matrices $\bm{X}_1, \bm{X}_2, \cdots, \bm{X}_r$ that share the same dimensions, we use the following abbreviation:
\begin{equation*}
\bigcirc_{i=1}^{r}(\bm{X}_i) = \bm{X}_1 \circ \bm{X}_2 \circ \cdots \circ \bm{X}_r\,.
\end{equation*}

\section{The bound of the minimum eigenvalues of NTK for infinite-width}
\label{sec:infinitely_width}

We present the details of our results on~\cref{ssec:Minimum_eigenvalue_NTK_infinite} in this section. Firstly, we provide the proof of~\cref{thm:lambda_min_inf_mixed} in~\cref{ssec:result_mixed_activation_function_inf}. Then in~\cref{ssec:result_other_activation_function_inf} we provide the result when several activation functions exist alone.

\subsection{Proof of Lemma 1}

Our proof mainly follows the results of~\citet{huang2020deep}, but due to the different network structures, the proof process and results are slightly different. Moreover, we provide a matrix version results, which~\citet{huang2020deep} does not contain. For self-completeness, we include the proof here.

\begin{proof}
By~\citet[Proposition 3]{huang2020deep}, written as matrix form,we have:
\begin{equation*}
\begin{split}
    &\bm{A}^{(1)}=\bm{XX}^\top\,,\\
    &\bm{A}^{(2)}=2\mathbb{E}_{\bm{w} \sim \mathcal N(\bm 0,\mathbb{I}_{d})}[\sigma_1(\bm{Xw})\sigma_1(\bm{Xw})^\top]\,,\\
    &\bm{A}^{(l)}=2\mathbb{E}_{\bm{w} \sim \mathcal N(\bm 0,\mathbb{I}_{N})}[\sigma_{l-1}(\sqrt{\bm{A}^{(l-1)}} \bm{w})\sigma_{l-1}(\sqrt{\bm{A}^{(l-1)}} \bm{w})^\top]+\alpha_{l-2}\bm{A}^{(l-1)}\,.
\end{split}
\end{equation*}

Note that, it is slightly different from the original result because the network structure is slightly different.

Let $\bm{G}^{(1)}=\bm{A}^{(1)}$, $\bm{G}^{(2)}=\bm{A}^{(2)}$ and $\bm{G}^{(l)}=2\mathbb{E}_{\bm{w} \sim \mathcal N(\bm 0,\mathbb{I}_{N})}[\sigma_{l-1}(\sqrt{\bm{A}^{(l-1)}} \bm{w})\sigma_{l-1}(\sqrt{\bm{A}^{(l-1)}} \bm{w})^\top]$, then we have:

\begin{equation*}
    \begin{split}
        & \bm{A}^{(1)}=\bm{G}^{(1)}=\bm{XX}^\top\,,\\
        & \bm{A}^{(2)} = \bm{G}^{(2)}=2\mathbb{E}_{\bm{w} \sim \mathcal N(\bm 0,\mathbb{I}_{d})}[\sigma_1(\bm{Xw})\sigma_1(\bm{Xw})^\top]\,,\\
        & \bm{G}^{(l)}=2\mathbb{E}_{\bm{w} \sim \mathcal N(\bm 0,\mathbb{I}_{N})}[\sigma_{l-1}(\sqrt{\bm{A}^{(l-1)}} \bm{w})\sigma_{l-1}(\sqrt{\bm{A}^{(l-1)}} \bm{w})^\top]\,,\\
        & \bm{A}^{(l)}=\bm{G}^{(l)}+\alpha_{l-2}\bm{A}^{(l-1)}\,.    
    \end{split}
\end{equation*}

According to~\citet[Proposition 4]{huang2020deep}, written as matrix form,we have:
\begin{equation*}
    \bm{K}^{(L)}=\sum_{l=1}^{L}\bm{G}^{(l)}\circ \dot{\bm{G}}^{(l+1)} \circ (\dot{\bm{G}}^{(l+2)}+  \alpha_{l}\bm{1}_{N \times N})\circ \cdots \circ (\dot{\bm{G}}^{(L)}+ \alpha_{L-2}\bm{1}_{N \times N})\,,
\end{equation*}

where the $\dot{\bm{G}}^{(s)}$ satisfy that $\dot{\bm{G}}^{(s)} = 2\mathbb{E}_{\bm{w} \sim \mathcal N(\bm{0},\mathbb{I}_{N})}[{\sigma}'_{s-1}(\sqrt{\bm{A}^{(s-1)}} \bm{w}){\sigma}'_{s-1}(\sqrt{\bm{A}^{(s-1)}} \bm{w})^\top]$. Combining the above results, we finish the proof.

\end{proof}

\subsection{Proof of~\cref{thm:lambda_min_inf_mixed}}
\label{ssec:result_mixed_activation_function_inf}

In this part, we present the proof of~\cref{thm:lambda_min_inf_mixed}. Differently from \citet{oymak2020toward}, our result allows for activation functions search in each layer.

Before we prove~\cref{thm:lambda_min_inf_mixed}, we provide some propositions that are helpful to our proof. To facilitate the writing of the proof, let $\alpha_0 := 0$.

\begin{proposition}
\label{prop:maxGii} When $\sigma_1$ is \text{Tanh}, the remaining layers are with \text{LeakyReLU} and for $l \in [L-2]$, $\alpha_{l} = 1$, the quantity $G_{ii}^{(l)}$ has the largest upper bound:
\begin{equation}
    G_{ii}^{(l)} \leq \begin{cases}
1  & \text{ if } l=1 \\
2(2+\eta^2)^{l-2}  & \text{ if } l \geq 2\,.
\end{cases}
\label{eq:Gii_upper_bound}
\end{equation}
We set $G_{\max} = 2(2+\eta^2)^{L-2}$ as the upper bound of $G_{ii}^{(L)}$.
\end{proposition}
\begin{proof}
To prove our result, we need bound $\bm{G}^{(l)}$ under different activation functions.
We summarize them as below.

When $\sigma_{l-1}$ is \text{ReLU}:
\begin{equation}
\begin{split}
    G_{ii}^{(l)}&=2\mathbb{E}_{w \sim \mathcal N(0,A_{ii}^{(l-1)})}[\sigma_{l-1}(w)^2] =\int_{-\infty}^{\infty}\frac{2}{\sqrt{2\pi A_{ii}^{(l-1)}}}e^{-\frac{x^2}{2A_{ii}^{(l-1)}}}\max(0,x)^2 \mathrm{d}x\\
    &=\int_{0}^{\infty}\frac{2}{\sqrt{2\pi A_{ii}^{(l-1)}}}e^{-\frac{x^2}{2A_{ii}^{(l-1)}}}x^2 \mathrm{d}x \\
    &=A_{ii}^{(l-1)}\,.
    \end{split}
\label{eq:Gii_upper_bound_RelU}
\end{equation}

When $\sigma_{l-1}$ is \text{LeakyReLU}:
\begin{equation}
\begin{split}
    G_{ii}^{(l)}&=2\mathbb{E}_{w \sim \mathcal N(0,A_{ii}^{(l-1)})}[\sigma_{l-1}(w)^2] =\int_{-\infty}^{\infty}\frac{2}{\sqrt{2\pi A_{ii}^{(l-1)}}}e^{-\frac{x^2}{2A_{ii}^{(l-1)}}}\max(\eta x,x)^2 \mathrm{d}x\\
    &=\int_{0}^{\infty}\frac{2}{\sqrt{2\pi A_{ii}^{(l-1)}}}e^{-\frac{x^2}{2A_{ii}^{(l-1)}}}x^2 \mathrm{d}x +\int_{-\infty}^{0}\frac{2}{\sqrt{2\pi A_{ii}^{(l-1)}}}e^{-\frac{x^2}{2A_{ii}^{(l-1)}}}\eta^2 x^2 \mathrm{d}x\\
    &=(1+\eta^2)A_{ii}^{(l-1)}\,.
\end{split}
\label{eq:Gii_upper_bound_LeakyReLU}
\end{equation}

When $\sigma_{l-1}$ is \text{Sigmoid}:
\begin{equation}
\begin{split}
    G_{ii}^{(l)}&=2\mathbb{E}_{w \sim \mathcal N(0,A_{ii}^{(l-1)})}[\sigma_{l-1}(w)^2]=\int_{-\infty}^{\infty}\frac{2}{\sqrt{2\pi A_{ii}^{(l-1)}}}e^{-\frac{x^2}{2A_{ii}^{(l-1)}}}f_{\mathrm{Sigmoid}}(x)^2 \mathrm{d}x\\
    & \leq \int_{-\infty}^{\infty}\frac{2}{\sqrt{2\pi A_{ii}^{(l-1)}}}e^{-\frac{x^2}{2A_{ii}^{(l-1)}}}(\frac{1}{2})^2 \mathrm{d}x\\
    & = \frac{1}{2}\,.
\end{split}
\label{eq:Gii_upper_bound_Sigmoid}
\end{equation}

When $\sigma_{l-1}$ is \text{Tanh}:
\begin{equation}
\begin{split}
    G_{ii}^{(l)}&=2\mathbb{E}_{w \sim \mathcal N(0,A_{ii}^{(l-1)})}[\sigma_{l-1}(w)^2]=\int_{-\infty}^{\infty}\frac{2}{\sqrt{2\pi A_{ii}^{(l-1)}}}e^{-\frac{x^2}{2A_{ii}^{(l-1)}}}f_{\mathrm{Tanh}}(x)^2 \mathrm{d}x\\
    & \leq \int_{-\infty}^{\infty}\frac{2}{\sqrt{2\pi A_{ii}^{(l-1)}}}e^{-\frac{x^2}{2A_{ii}^{(l-1)}}}\mathrm{d}x\\
    & = 2\,.
\end{split}
\label{eq:Gii_upper_bound_Tanh}
\end{equation}

When $\sigma_{l-1}$ is \text{Swish}:

\begin{equation}
\small
\begin{split}
G_{ii}^{(l)}&=2\mathbb{E}_{w \sim \mathcal N(0,A_{ii}^{(l-1)})}[\sigma_{l-1}(w)^2]=\int_{-\infty}^{\infty}\frac{2}{\sqrt{2\pi A_{ii}^{(l-1)}}}e^{-\frac{x^2}{2A_{ii}^{(l-1)}}}f_{\mathrm{Swish}}(x)^2 \mathrm{d}x\\
& = \int_{-\infty}^{\infty}\frac{2}{\sqrt{2\pi A_{ii}^{(l-1)}}}e^{-\frac{x^2}{2A_{ii}^{(l-1)}}}\frac{x^2}{(1+e^{-x})^2} \mathrm{d}x\\
& = \int_{0}^{\infty}\frac{2}{\sqrt{2\pi A_{ii}^{(l-1)}}}e^{-\frac{x^2}{2A_{ii}^{(l-1)}}}x^2\times \bigg(\frac{1}{(1+e^{-x})^2}+\frac{1}{(1+e^{x})^2}\bigg) \mathrm{d}x\\
& \leq \int_{0}^{\infty}\frac{2}{\sqrt{2\pi A_{ii}^{(l-1)}}}e^{-\frac{x^2}{2A_{ii}^{(l-1)}}}x^2\mathrm{d}x\\
& = A_{ii}^{(l-1)}\,.
\end{split}
\label{eq:Gii_upper_bound_Swish}
\end{equation}

Combining~\cref{eq:Gii_upper_bound_RelU,eq:Gii_upper_bound_LeakyReLU,eq:Gii_upper_bound_Sigmoid,eq:Gii_upper_bound_Tanh,eq:Gii_upper_bound_Swish} with~\cref{lemma:NTK_matrix_recursive_form} we draw the conclusion and finish the proof.
\end{proof}

\begin{proposition}
The relationship between $A^{(l)}_{ii}$ and $A^{(l-1)}_{ii}$ for different activation functions can be summarized as~\cref{tab:relationship_between_aii_for_different_activation_function} according to the difference of $\sigma_{l-1}$.

\begin{table*}
\centering
\begin{tabular}{l@{\hspace{0.25cm}} c@{\hspace{0.2cm}}c@{\hspace{0.2cm}}c@{\hspace{0.2cm}}c@{\hspace{0.2cm}} c} 
    \hline
    $\sigma_{l-1}$ & ReLU & LeakyReLU & Sigmoid & Tanh & Swish\\
    \hline
    Upper bound & $1$  & $1+\eta^2$ & $\frac{1}{8}$ & $2$& $1$  \\
    
    Lower bound & $1$ &$1+\eta^2$& $(\frac{1}{2}-\frac{1}{2\sqrt{1+\frac{G_{\max}}{4}}})\frac{1}{G_{\max}}$ & $(2-\frac{2}{\sqrt{1+G_{\max}}})\frac{1}{G_{\max}}$& $\frac{1}{2}$  \\
    \hline
\end{tabular}
\caption{Upper and lower bounds for $A_{ii}^{(l)}/A_{ii}^{(l-1)} - \alpha_{l-2}$ for different activation functions $\sigma_{l-1}$ and the binary variable $\alpha_{l-2} \in \{ 0 , 1\}$ indicates whether $(l-1)$-th layer has a skip connection or not.}
\label{tab:relationship_between_aii_for_different_activation_function}
\end{table*}

\end{proposition}
\begin{proof}
To prove our result, we need to bound the ratio $A_{ii}^{(l)}/A_{ii}^{(l-1)}$ for different activation functions. We illustrate how this is achieved in different cases below:

For $l \geq 2$:

When $\sigma_{l-1}$ is \text{ReLU} by~\cref{eq:Gii_upper_bound_RelU} we have:

\begin{equation}
A_{ii}^{(l)}=G_{ii}^{(l)}+ \alpha_{l-2}A_{ii}^{(l-1)} = (1+\alpha_{l-2})A_{ii}^{(l-1)}\,.
\label{eq:Aii_bound_ReLU}
\end{equation}

When $\sigma_{l-1}$ is \text{LeakyReLU} by~\cref{eq:Gii_upper_bound_LeakyReLU} we have:

\begin{equation}
A_{ii}^{(l)}=G_{ii}^{(l)}+\alpha_{l-2}A_{ii}^{(l-1)} = (1+\alpha_{l-2}+\eta^2)A_{ii}^{(l-1)}\,.
\label{eq:Aii_bound_LeakyReLU}
\end{equation}

When $\sigma_{l-1}$ is \text{Swish}, $G_{ii}^{(l)}$ can be upper by~\cref{eq:Gii_upper_bound_Swish} and lower bounded by:

\begin{equation}
\small
\begin{split}
G_{ii}^{(l)}&=2\mathbb{E}_{w \sim \mathcal N(0,A_{ii}^{(l-1)})}[\sigma_{l-1}(w)^2] =\int_{-\infty}^{\infty}\frac{2}{\sqrt{2\pi A_{ii}^{(l-1)}}}e^{-\frac{x^2}{2A_{ii}^{(l-1)}}}f_{\mathrm{Swish}}(x)^2 \mathrm{d}x\\
& = \int_{-\infty}^{\infty}\frac{2}{\sqrt{2\pi A_{ii}^{(l-1)}}}e^{-\frac{x^2}{2A_{ii}^{(l-1)}}}\frac{x^2}{(1+e^{-x})^2} \mathrm{d}x\\
& = \int_{0}^{\infty}\frac{2}{\sqrt{2\pi A_{ii}^{(l-1)}}}e^{-\frac{x^2}{2A_{ii}^{(l-1)}}}x^2\times \bigg(\frac{1}{(1+e^{-x})^2}+\frac{1}{(1+e^{x})^2}\bigg) \mathrm{d}x\\
& \geq \int_{0}^{\infty}\frac{2}{\sqrt{2\pi A_{ii}^{(l-1)}}}e^{-\frac{x^2}{2A_{ii}^{(l-1)}}}x^2\times \frac{1}{2}\mathrm{d}x\\
& = \frac{1}{2}A_{ii}^{(l-1)}\,,
\end{split}
\label{eq:Gii_lower_bound_Swish}
\end{equation}

which implies:
\begin{equation}
\left(\frac{1}{2}+\alpha_{l-2}\right) A_{ii}^{(l-1)} \leq A_{ii}^{(l)}\leq (1+\alpha_{l-2})A_{ii}^{(l-1)}\,.
\label{eq:Aii_bound_Swish}
\end{equation}

When $\sigma_{l-1}$ is \text{Sigmoid}, $G_{ii}^{(l)}$ can be upper by:
\begin{equation}
\begin{split}
    G_{ii}^{(l)}&=2\mathbb{E}_{w \sim \mathcal N(0,A_{ii}^{(l-1)})}[\sigma_{l-1}(w)^2] =\int_{-\infty}^{\infty}\frac{2}{\sqrt{2\pi A_{ii}^{(l-1)}}}e^{-\frac{x^2}{2A_{ii}^{(l-1)}}}f_{\mathrm{Sigmoid}}(x)^2 \mathrm{d}x\\
    & \leq \int_{-\infty}^{\infty}\frac{2}{\sqrt{2\pi A_{ii}^{(l-1)}}}e^{-\frac{x^2}{2A_{ii}^{(l-1)}}}\bigg(\frac{1}{4} - e^{-\frac{x^2}{4}}\bigg)\mathrm{d}x = \frac{1}{2}-\frac{1}{2\sqrt{1+\frac{A_{ii}^{(l-1)}}{2}}}\\
    & \leq \frac{A_{ii}^{(l-1)}}{8}\,,\quad\quad\quad\text{holds for $x \geq 0$}\,.
\end{split}
\label{eq:Gii_upper_bound_Sigmoid_2}
\end{equation}

Then $G_{ii}^{(l)}$ can be lower bounded by:
\begin{equation}
\begin{split}
    G_{ii}^{(l)}&=2\mathbb{E}_{w \sim \mathcal N(0,A_{ii}^{(l-1)})}[\sigma_{l-1}(w)^2] =\int_{-\infty}^{\infty}\frac{2}{\sqrt{2\pi A_{ii}^{(l-1)}}}e^{-\frac{x^2}{2A_{ii}^{(l-1)}}}f_{\mathrm{Sigmoid}}(x)^2 \mathrm{d}x\\
    & \geq \int_{-\infty}^{\infty}\frac{2}{\sqrt{2\pi A_{ii}^{(l-1)}}}e^{-\frac{x^2}{2A_{ii}^{(l-1)}}}\bigg(\frac{1}{4} - e^{-\frac{x^2}{8}}\bigg)\mathrm{d}x\\
    & =\frac{1}{2}-\frac{1}{2\sqrt{1+\frac{A_{ii}^{(l-1)}}{4}}} \\
    & \geq \left(\frac{1}{2}-\frac{1}{2\sqrt{1+\frac{G_{\max}}{4}}}\right)\frac{A_{ii}^{(l-1)}}{G_{\max}}\,,
\end{split}
\label{eq:Gii_lower_bound_Sigmoid}
\end{equation}
where we use the fact that the penultimate line is a concave function with respect to $A_{ii}^{(l-1)}$. When $A_{ii}^{(l-1)}=0$, the function value is 0. That means $G_{ii}^{(l)}/A_{ii}^{(l-1)}$ obtains the minimum value at $G_{ii}^{(l)} = G_{\max}$. Combined with~\cref{eq:Gii_upper_bound}, we get the last inequality.

Then, we have:

\begin{equation}
\left( \left[\frac{1}{2}-\frac{1}{2\sqrt{1+\frac{G_{\max}}{4}}}\right]\frac{1}{G_{\max}}+\alpha_{l-2}\right)A_{ii}^{(l-1)} \leq A_{ii}^{(l)}\leq \left(\frac{1}{8}+\alpha_{l-2}\right)A_{ii}^{(l-1)} \,.
\label{eq:Aii_bound_Sigmoid}
\end{equation}

When $\sigma_{l-1}$ is \text{Tanh}, $G_{ii}^{(l)}$ can be upper bounded by:

\begin{equation}
\begin{split}
    G_{ii}^{(l)}&=2\mathbb{E}_{w \sim \mathcal N(0,A_{ii}^{(l-1)})}[\sigma_{l-1}(w)^2] =\int_{-\infty}^{\infty}\frac{2}{\sqrt{2\pi A_{ii}^{(l-1)}}}e^{-\frac{x^2}{2A_{ii}^{(l-1)}}}f_{\mathrm{Tanh}}(x)^2 \mathrm{d}x\\
    & \leq \int_{-\infty}^{\infty}\frac{2}{\sqrt{2\pi A_{ii}^{(l-1)}}}e^{-\frac{x^2}{2A_{ii}^{(l-1)}}}(1 - e^{-x^2})\mathrm{d}x = 2-\frac{2}{\sqrt{1+2A_{ii}^{(l-1)}}}\\
    & \leq 2A_{ii}^{(l-1)}\,, \quad\quad\quad\text{holds for $x \geq 0$}\,.
\end{split}
\label{eq:Gii_upper_bound_Tanh_2}
\end{equation}
The lower bound is:
\begin{equation}
\begin{split}
    G_{ii}^{(l)}&=2\mathbb{E}_{w \sim \mathcal N(0,A_{ii}^{(l-1)})}[\sigma_{l-1}(w)^2] =\int_{-\infty}^{\infty}\frac{2}{\sqrt{2\pi A_{ii}^{(l-1)}}}e^{-\frac{x^2}{2A_{ii}^{(l-1)}}}f_{\mathrm{Tanh}}(x)^2 \mathrm{d}x\\
    & \geq \int_{-\infty}^{\infty}\frac{2}{\sqrt{2\pi A_{ii}^{(l-1)}}}e^{-\frac{x^2}{2A_{ii}^{(l-1)}}}(1 - e^{-\frac{x^2}{2}})\mathrm{d}x\\
    & = 2-\frac{2}{\sqrt{1+A_{ii}^{(l-1)}}}\\
    & \geq\bigg(2-\frac{2}{\sqrt{1+G_{\max}}}\bigg)\frac{A_{ii}^{(l-1)}}{G_{\max}}\,.
    \end{split}
\label{eq:Gii_lower_bound_Tanh}
\end{equation}
Similar to the Sigmoid, the penultimate line is an concave function with respect to $A_{ii}^{(l-1)}$. When $A_{ii}^{(l-1)}=0$, the function value is 0. That means $G_{ii}^{(l)}/A_{ii}^{(l-1)}$ obtains the minimum value at $G_{ii}^{(l)} = G_{\max}$. Combined with~\cref{eq:Gii_upper_bound}, we get the last inequality.

Then, we have:

\begin{equation}
\bigg(\bigg[2-\frac{2}{\sqrt{1+G_{\max}}}\bigg]\frac{1}{G_{\max}}+\alpha_{l-2}\bigg)A_{ii}^{(l-1)} \leq A_{ii}^{(l)}\leq (2+\alpha_{l-2})A_{ii}^{(l-1)}\,.
\label{eq:Aii_bound_Tanh}
\end{equation}

According to~\cref{eq:Aii_bound_ReLU,eq:Aii_bound_LeakyReLU,eq:Aii_bound_Swish,eq:Aii_bound_Sigmoid,eq:Aii_bound_Tanh}, we can summarized the results about bound of $A_{ii}^{(l)}/A_{ii}^{(l-1)}$ in Table~\ref{tab:relationship_between_aii_for_different_activation_function}.

\end{proof}

\begin{proposition}
The bound of $\dot{G}^{(l)}_{ii}$ with respect to different activation function $\sigma_{l-1}$ can be summarized in~\cref{tab:bound_of_dot_gii_for_different_activation_function}.

\begin{table*}
\centering
\begin{tabular}{l@{\hspace{0.25cm}} c@{\hspace{0.2cm}}c@{\hspace{0.2cm}}c@{\hspace{0.2cm}}c@{\hspace{0.2cm}} c} 
    \hline
    $\sigma_{l-1}$ & ReLU & LeakyReLU & Sigmoid & Tanh & Swish\\
    \hline
    Upper bound for $\dot{G}_{ii}^{(l)}$ & $1$  & $1+\eta^2$ & $1/8$ & $2$& $1.22$  \\
    
    Lower bound for $\dot{G}_{ii}^{(l)}$ & $1$ &$1+\eta^2$& $f_{\mathrm{S}}(G_{\max})$ & $f_{\mathrm{T}}(G_{\max})$& $1/2$  \\
    \hline
\end{tabular}
\caption{Upper and lower bounds for $\dot{G}_{ii}^{(l)}$ for different activation function $\sigma_{l-1}$.}
\label{tab:bound_of_dot_gii_for_different_activation_function}
\end{table*}

\end{proposition}

\begin{proof}

To prove our result, we need to bound $\dot{G}^{(l)}_{ii}$ with respect to different activation function $\sigma_{l-1}$ as follows.

When $\sigma_{l-1}$ is \text{ReLU}:
\begin{equation}
\begin{split}
    \dot{G}_{ii}^{(l)}&=2\mathbb{E}_{w \sim \mathcal N(0,A_{ii}^{(l-1)})}[{\sigma}'_{l-1}(w)^2] =\int_{0}^{\infty}\frac{2}{\sqrt{2\pi A_{ii}^{(l-1)}}}e^{-\frac{x^2}{2A_{ii}^{(l-1)}}} \mathrm{d}x\\
    &=1\,.
\end{split}
\label{eq:dotGii_RelU}
\end{equation}

When $\sigma_{l-1}$ is \text{LeakyReLU}:
\begin{equation}
\begin{split}
    \dot{G}_{ii}^{(l)}&=2\mathbb{E}_{w \sim \mathcal N(0,A_{ii}^{(l-1)})}[{\sigma}'_{l-1}(w)^2] \\
    &=\int_{0}^{\infty}\frac{2}{\sqrt{2\pi A_{ii}^{(l-1)}}}e^{-\frac{x^2}{2A_{ii}^{(l-1)}}} \mathrm{d}x +\int_{-\infty}^{0}\frac{2}{\sqrt{2\pi A_{ii}^{(l-1)}}}e^{-\frac{x^2}{2A_{ii}^{(l-1)}}} \eta^2 \mathrm{d}x\\
    &=1+\eta^2\,.
\end{split}
\label{eq:dotGii_LeakyReLU}
\end{equation}

When $\sigma_{l-1}$ is \text{Sigmoid}, according to the monotonicity of the $f_{\mathrm{S}}$, we have:
\begin{equation}
f_{\mathrm{S}}(G_{\max})\leq\dot{G}_{ii}^{(l)}=2\mathbb{E}_{w \sim \mathcal N(0,A_{ii}^{(l-1)})}[{\sigma}'_{l-1}(w)^2] =f_{\mathrm{S}}(A_{ii}^{(l-1)})\leq f_{\mathrm{S}}(0) \leq \frac{1}{8}\,.
\label{eq:dotGii_Sigmoid_1}
\end{equation}

When $\sigma_{l-1}$ is \text{Tanh}, according to the monotonicity of the $f_{\mathrm{T}}$, we have:

\begin{equation}
f_{\mathrm{T}}(G_{\max})\leq\dot{G}_{ii}^{(l)}=2\mathbb{E}_{w \sim \mathcal N(0,A_{ii}^{(l-1)})}[{\sigma}'_{l-1}(w)^2] =f_{\mathrm{T}}(A_{ii}^{(l-1)})\leq f_{\mathrm{T}}(0)\leq 2\,.
\label{eq:dotGii_Tanh_1}
\end{equation}

When $\sigma_{l-1}$ is \text{Swish}, The quantity $\dot{G}^{(l)}_{ii}$ can be upper bounded by:
\begin{equation}
\begin{split}
\dot{G}_{ii}^{(l)}&=2\mathbb{E}_{w \sim \mathcal N(0,A_{ii}^{(l-1)})}[{\sigma}'_{l-1}(w)^2]=\int_{-\infty}^{\infty}\frac{2}{\sqrt{2\pi A_{ii}^{(l-1)}}}e^{-\frac{x^2}{2A_{ii}^{(l-1)}}}f'_{\mathrm{Swish}}(x)^2 \mathrm{d}x\\
&\leq \int_{0}^{\infty}\frac{2}{\sqrt{2\pi A_{ii}^{(l-1)}}}e^{-\frac{x^2}{2A_{ii}^{(l-1)}}}\bigg(\sup_x f'_{\mathrm{Swish}}(x)\bigg)^2 \mathrm{d}x\\
&\quad + \int_{-\infty}^{0}\frac{2}{\sqrt{2\pi A_{ii}^{(l-1)}}}e^{-\frac{x^2}{2A_{ii}^{(l-1)}}}\bigg(\inf_x f'_{\mathrm{Swish}}(x)\bigg)^2 \mathrm{d}x\\
&=\bigg(\inf_x f_{\mathrm{Swish}}'(x)\bigg)^2+\bigg(\sup_x f'_{\mathrm{Swish}}(x)\bigg)^2\\
&\leq 1.22\,,
\end{split}
\label{eq:dotGii_Swish_1}
\end{equation}
where the last inequality holds by $1.099 < \sup_x f'_{\mathrm{Swish}}(x) < 1.1$ and $-0.1 < \inf_x f'_{\mathrm{Swish}}(x) < -0.099$.

Then the quantity $\dot{G}^{(l)}_{ii}$ can be lower bounded by:
\begin{equation}
\begin{split}
\dot{G}_{ii}^{(l)}&=2\mathbb{E}_{w \sim \mathcal N(0,A_{ii}^{(l-1)})}[{\sigma}'_{l-1}(w)^2]=\int_{-\infty}^{\infty}\frac{2}{\sqrt{2\pi A_{ii}^{(l-1)}}}e^{-\frac{x^2}{2A_{ii}^{(l-1)}}}f'_{\mathrm{Swish}}(x)^2 \mathrm{d}x\\
& = \int_{0}^{\infty}\frac{2}{\sqrt{2\pi A_{ii}^{(l-1)}}}e^{-\frac{x^2}{2A_{ii}^{(l-1)}}}\bigg(f'_{\mathrm{Swish}}(x)^2+f'_{\mathrm{Swish}}(-x)^2\bigg)\mathrm{d}x\\
&\geq \int_{0}^{\infty}\frac{2}{\sqrt{2\pi A_{ii}^{(l-1)}}}e^{-\frac{x^2}{2A_{ii}^{(l-1)}}}\frac{1}{2}\mathrm{d}x\\
&= \frac{1}{2}\,.
\end{split}
\label{eq:dotGii_Swish_2}
\end{equation}

Combining~\cref{eq:dotGii_RelU,eq:dotGii_LeakyReLU,eq:dotGii_Sigmoid_1,eq:dotGii_Tanh_1,eq:dotGii_Swish_1,eq:dotGii_Swish_2}, we can summarize the results about bound of $\dot{G}_{ii}^{(l)}$ in~\cref{tab:bound_of_dot_gii_for_different_activation_function}.

\end{proof}






Now we are ready to prove~\cref{thm:lambda_min_inf_mixed}.

\begin{proof}[Proof of~\cref{thm:lambda_min_inf_mixed}]

Now we are ready to present the estimation on $\lambda _{\min}(\bm{K}^{(L)})$ as below.

From~\cref{lemma:NTK_matrix_recursive_form}, we have the NTK formulation for residual neural networks:
\begin{equation*}
    \bm{K}^{(L)}=\bm{G}^{(L)} + \sum_{l=1}^{L-1}\bm{G}^{(l)}\circ \dot{\bm{G}}^{(l+1)} \circ (\dot{\bm{G}}^{(l+2)}+  \alpha_{l}\bm{1}_{N \times N})\circ \cdots \circ (\dot{\bm{G}}^{(L)}+ \alpha_{L-2}\bm{1}_{N \times N})\,.
\end{equation*}

It is clear that all the matrices $\bm{G}^{(l)}$, $\dot{\bm{G}}^{(l)}$ are positive semi-definite (PSD), then $(\dot{\bm{G}}^{(l+1)}+  \alpha_{l}\bm{1}_{N \times N})$ are also PSD.
For two PSD matrices $\bm{P}, \bm{Q} \in \mathbb{R}^{N \times N}$, it holds $\lambda _{\min}(\bm{P} \circ \bm{Q})\geq \lambda _{\min}(\bm{P})\min_{i\in [N]}Q_{ii}$~\citep{Schur+1911+1+28}\,.
Accordingly, we have:
\begin{equation*}
\lambda _{\min}(\bm{K}^{(L)})\geq \sum_{l=1}^{L}\lambda _{\min}(\bm{G}^{(l)})\min_{i\in[N]}\prod_{p=l+1}^{L}\bigg(\dot{G}^{(p)}_{ii}+\alpha_{p-2}\bigg)\,.
\end{equation*}

Then we bound $\lambda_{\min}(\bm{G}^{(2)})$:
\begin{equation*}
\begin{split}
\lambda_{\min}(\bm{G}^{(2)}) & = \lambda_{\min}\bigg(2\mathbb{E}_{\bm{w} \sim \mathcal N(0,\mathbb{I}_{d})}[\sigma_1(\bm{Xw})\sigma_1(\bm{Xw})^\top]\bigg)\\
& = 2\lambda_{\min}\bigg(\sum_{s=0}^{\infty}\mu_{s}(\sigma_1)^{2}\bigcirc_{i=1}^{s}(\bm{XX}^{\top})\bigg) \quad \text{\citep[Lemma D.3]{NEURIPS2020_8abfe8ac}}\\
& \geq 2\mu_{r}(\sigma_1)^{2}\lambda_{\min}(\bigcirc_{i=1}^{r}\bm{XX}^{\top}) \quad \bigg(\mbox{taking~}r \geq \frac{\log (2N)}{1-C_{\text{max}}}\bigg)\\
& \geq 2\mu_{r}(\sigma_1)^{2}\bigg(\min_{i\in [N]}\left \| \bm{x}_i \right \|_{2}^{2r}-(N-1)\max_{i\neq j}\left | \left \langle \bm{x}_i, \bm{x}_j \right \rangle \right |^r\bigg) \quad \text{(Gershgorin circle theorem)}\\
& \geq 2\mu_{r}(\sigma_1)^{2}\bigg(1-(N-1)C_{\text{max}}^r\bigg) \quad \left( \mbox{using~\cref{assumption:distribution_1}}  \right) \\
& \geq 2\mu_{r}(\sigma_1)^{2}\bigg(1-(N-1)\bigg(1-\frac{\log (2N)}{r}\bigg)^r\bigg) \\
& \geq 2\mu_{r}(\sigma_1)^{2}\bigg(1-(N-1)\exp(-\log(2N))\bigg) \\
& \geq \mu_{r}(\sigma_1)^{2}\,,
\end{split}
\end{equation*}
where $\big(1-\frac{\log (2N)}{r}\big)^r$ is an increasing function of $r$ when $r\geq 2\log (2N)$.
As a reminder, the symbol $\bigcirc$ denotes the Hadamard product, which is defined in~\cref{sssec:other_notations}.


That means:

\begin{equation}
\begin{split}
\lambda _{\min}(\bm{K}^{(L)}) & \geq \sum_{l=1}^{L}\lambda _{\min}(\bm{G}^{(l)})\min_{i\in[N]}\prod_{p=l+1}^{L}(\dot{G}^{(p)}_{ii}+\alpha_{p-2}) \geq \lambda _{\min}(\bm{G}^{(2)})\min_{i\in[N]}\prod_{p=3}^{L}(\dot{G}^{(p)}_{ii}+\alpha_{p-2})\\
& \geq  \mu_{r}(\sigma_1)^{2} \min_{i\in[N]}\prod_{p=3}^{L}(\dot{G}^{(p)}_{ii}+\alpha_{p-2}), \quad \bigg(r \geq \frac{\log (2n)}{1-C_{\text{max}}}\bigg) \,.
\end{split}
\label{eq:NTK_infinity_lower_bound}
\end{equation}

According to~\cref{tab:activation_functions} and~\cref{tab:bound_of_dot_gii_for_different_activation_function}, we have:

\begin{equation}
\prod_{p=l+1}^{L}(\dot{G}^{(p)}_{ii}+\alpha_{p-2}) \leq \prod_{p=l+1}^{L}\bigg(\beta_2(\sigma_{p-1})+\alpha_{p-2}\bigg)\,,
\label{eq:dot_gii_prod_bound_1}
\end{equation}

\begin{equation}
\prod_{p=3}^{L}(\dot{G}^{(p)}_{ii}+\alpha_{p-2}) \geq \prod_{p=3}^{L}\bigg(\beta_3(\sigma_{p-1})+\alpha_{p-2}\bigg)\,.
\label{eq:dot_gii_prod_bound_2}
\end{equation}

We know that the sum of eigenvalues of $\bm{K}^{(L)}$ is equal to the trace of $\bm{K}^{(L)}$. The upper bound of $\lambda _{\min}(\bm{K}^{(L)})$ is directly given by:

\begin{equation}
\lambda _{\min}(\bm{K}^{(L)}) \leq \frac{1}{N} \sum_{i=1}^{N} \sum_{l=1}^{L}(G^{(l)}_{ii})\prod_{p=l+1}^{L}(\dot{G}^{(p)}_{ii}+\alpha_{p-2})\,.
\label{eq:NTK_infinity_upper_bound}
\end{equation}

\paragraph{Final result - upper bound\\}

For $G^{(l)}_{ii}$ we have the following bound:

\begin{equation}
\small
G^{(l)}_{ii} = \beta_1(\sigma_{l-1}) A^{(l-1)}_{ii} \leq \beta_1(\sigma_{l-1}) \prod_{p=3}^{l-1} \bigg(\beta_1(\sigma_{p-1})+\alpha_{p-2}\bigg) A^{(2)}_{ii}\leq \beta_1(\sigma_{l-1}) \prod_{p=2}^{l-1} \bigg(\beta_1(\sigma_{p-1})+\alpha_{p-2}\bigg)\,.
\label{eq:gii_bound}
\end{equation}

By~\cref{eq:dot_gii_prod_bound_1,eq:NTK_infinity_upper_bound,eq:gii_bound}, we have:

\begin{equation}
\lambda _{\min}(\bm{K}^{(L)})\leq \sum_{l=1}^{L}\bigg(\beta_1(\sigma_{l-1}) \prod_{p=2}^{l-1} \bigg[\beta_1(\sigma_{p-1})+\alpha_{p-2}\bigg]\prod_{p=l+1}^{L}\bigg[\beta_2(\sigma_{p-1})+\alpha_{p-2}\bigg]\bigg)\,.
\label{eq:NTK_infinity_upper_bound_final}
\end{equation}

\paragraph{Final result - lower bound\\}

By~\cref{eq:NTK_infinity_lower_bound,eq:dot_gii_prod_bound_2}, we have:

\begin{equation*}
\lambda _{\min}(\bm{K}^{(L)}) \geq  \mu_{r}(\sigma_1)^{2}\prod_{p=3}^{L}\bigg(\beta_3(\sigma_{p-1})+\alpha_{p-2}\bigg), \quad \bigg(r \geq \frac{\log (2n)}{1-C_{\text{max}}}\bigg)\,.
\end{equation*}

\end{proof}

\subsection{Special cases}
\label{ssec:result_other_activation_function_inf}

To provide further insights into our proofs of mixed activation functions as provided in the previous sections, we now consider the special case of a single activation function in each layer.
\begin{corollary}
\label{thm:lambda_min_inf}
Under Assumption~\ref{assumption:distribution_1}, for a deep fully-connected ResNet with the same activation functions in every layer and for a not very large $L$, let $\bm{K}^{(L)}$ be the limiting NTK recursively defined in~\cref{lemma:NTK_matrix_recursive_form}. Then, we have:

For ReLU:
\begin{equation*}
\mu_{r}(\sigma_1)^{2}\prod_{p=3}^{L}(1+\alpha_{p-2}) \leq\lambda _{\min}(\bm{K}^{(L)})\leq \sum_{l=1}^{L}\bigg(\frac{\prod_{p=2}^{L} (1+\alpha_{p-2})}{1+\alpha_{l-2}}\bigg)\,.
\end{equation*}

For LeakyReLU:
\begin{equation*}
\mu_{r}(\sigma_1)^{2}\prod_{p=3}^{L}(1+\eta^2+\alpha_{p-2}) \leq \lambda _{\min}(\bm{K}^{(L)})\leq (1+\eta^2)\sum_{l=1}^{L} \bigg(\frac{\prod_{p=2}^{L} (1+\eta^2+\alpha_{p-2})}{1+\eta^2+\alpha_{l-2}}\bigg)\,.
\end{equation*}

For Sigmoid:
\begin{equation}
\mu_{r}(\sigma_1)^{2}\prod_{p=3}^{L}(f_{\mathrm{S}}(\frac{1}{2})+\alpha_{p-2}) \leq \lambda _{\min}(\bm{K}^{(L)})\leq\frac{1}{8} \sum_{l=1}^{L}\bigg(\frac{\prod_{p=2}^{L} (\frac{1}{8}+\alpha_{p-2})}{\frac{1}{8}+\alpha_{l-2}}\bigg)\,.
\label{eq:lambda_min_inf_Sigmoid}
\end{equation}

For Tanh:
\begin{equation}
\mu_{r}(\sigma_1)^{2}\prod_{p=3}^{L}(f_{\mathrm{T}}(2)+\alpha_{p-2}) \leq \lambda _{\min}(\bm{K}^{(L)})\leq 2\sum_{l=1}^{L}\bigg(\frac{\prod_{p=2}^{L}(2+\alpha_{p-2})}{2+\alpha_{l-2}}\bigg)\,.
\label{eq:lambda_min_inf_Tanh}
\end{equation}

For Swish:
\begin{equation*}
\mu_{r}(\sigma_1)^{2}\prod_{p=3}^{L}(\frac{1}{2}+\alpha_{p-2}) \leq \lambda _{\min}(\bm{K}^{(L)})\leq \sum_{l=1}^{L}\bigg(\prod_{p=2}^{l-1} (1+\alpha_{p-2})\prod_{p=l+1}^{L}(1.22+\alpha_{p-2})\bigg)\,.
\end{equation*}

The $\mu_r(\sigma_1)$ is $r$-st Hermite coefficient of the activation function. 

\end{corollary}

\begin{proof}

By,~\cref{tab:activation_functions,eq:NTK_infinity_lower_bound,eq:NTK_infinity_upper_bound_final}. we can have this result.

It should be noted that for Sigmoid network (all of activation functions are Sigmoid) and Tanh (all of activation functions are Tanh) network, the upper bound of $G_{\max}$ will change. By~\cref{eq:Gii_upper_bound_Sigmoid,eq:Gii_upper_bound_Tanh} we have for Sigmoid $G_{\max} = \frac{1}{2}$, For Tanh $G_{\max} = 2$. That means $f_{\mathrm{S}}(G_{\max})$ in the~\cref{thm:lambda_min_inf_mixed} is replaced by $f_{\mathrm{S}}(\frac{1}{2})$ in~\cref{eq:lambda_min_inf_Sigmoid} and $f_{\mathrm{T}}(G_{\max})$ in the~\cref{thm:lambda_min_inf_mixed} is replaced by $f_{\mathrm{T}}(2)$ in~\cref{eq:lambda_min_inf_Tanh}.

\end{proof}

\section{The bound of the minimum eigenvalues of NTK for finite-width}
\label{sec:finitely_width}

We present the details of our results on~\cref{ssec:Minimum_eigenvalue_NTK_finite} in this section. Firstly, we introduce the specific expression form for NTK of finite-width network in~\cref{sec:NTK_finite_width}. Then, we introduce some lemmas in~\cref{ssec:relevant_lemmas_finitely_width} to facilitate the proof of theorems, after that we provide the results of multiple activation functions are mixed in one network in~\cref{ssec:Results_mixed_finitely_width} directly, finally we discuss the results.

\subsection{Neural Tangent Kernel for finite-width}
\label{sec:NTK_finite_width}

\begin{equation*}
\bar{\bm{K}}^{(L)}=\bm{JJ}^{\top}=\sum_{l=1}^{L}\bigg[\frac{\partial \bm{F}}{\partial \mathrm{vec}(\bm{W}_l)}\bigg]\bigg[\frac{\partial \bm{F}}{\partial \mathrm{vec}(\bm{W}_l)}\bigg]^{\top}\,.
\end{equation*}

Let $\bm{F}_k=[\bm{f}_k(\bm{x}_1) ,\ldots,\bm{f}_k(\bm{x}_N)]^T$, by chain rule and some standard calculation, we have,

\begin{equation*}
\bm{JJ}^{\top}=\sum_{k=0}^{L-1}\bm{F}_k\bm{F}_k^{\top}\circ \bm{B}_{k+1}\bm{B}_{k+1}^{\top}\,,
\end{equation*}

where $\bm B_k \in \mathbb{R}^{N \times m}$ is a matrix of which the $i$-th row is given by

\begin{equation*}
(\bm{B}_k)_{i:}=\left\{\begin{matrix}
\bm{D}_{i,k}\prod _{l=k+1}^{L-1}(\bm{W}_l \bm{D}_{i,l}+\alpha_{l-1}\bm{I}_{m\times m})\bm{W}_L, k \in [L-2]\,,\\
\bm{D}_{i,L-1} \bm{W}_L,\qquad\qquad\qquad\qquad\qquad\qquad k = L-1\,,\\ 
1,\qquad\qquad\qquad\qquad\qquad\qquad\qquad\ \  k = L\,.
\end{matrix}\right.
\label{eq:NTK_finite_3}
\end{equation*}

\subsection{Relevant Lemmas}
\label{ssec:relevant_lemmas_finitely_width}

\begin{lemma}
\label{lemma:lemma_C.1_in_ICML}
Fix any $k \in [0,L-1] $ and $\bm{x}\sim P_X$, then for \text{ReLU}, \text{LeakyReLU}, \text{Sigmoid}, \text{Tanh} and \text{Swish} we have
\begin{equation*}
\left \| \bm{f}_k(\bm{x}) \right \|_2^2 = \Theta (1)\,,
\end{equation*}

with probability at least $1-\sum_{l=1}^{k}\exp(-\Omega (m))$ over $(\bm{W}_l)_{l=1}^k$ and $\bm{x}$. Moreover,

\begin{equation*}
\mathbb{E}_{\bm{x}}\left \| \bm{f}_k(\bm{x}) \right \|_2^2 = \Theta (1)\,,
\end{equation*}
with probability at least $1-\sum_{l=1}^{k-1}\exp(-\Omega (m))$ over $(\bm{W}_l)_{l=1}^k$.

\end{lemma}

\begin{proof}
We prove this by induction.

The result holds for $k = 0$ due to~\cref{assumption:distribution_1}.

Assume that the lemma holds for some $k-1$, i.e. 

\begin{equation*}
\left \| \bm{f}_{k-1}(\bm{x}) \right \|_2^2 = \Theta (1)\,,
\end{equation*}

with probability at least $1-\sum_{l=1}^{k-1}\exp(-\Omega (m))$ over $(\bm{W}_l)_{l=1}^k$ and $\bm{x}$. 

Let us condition on this event of $(\bm{W}_l)_{l=1}^{k-1}$ and study probability bounds over $\bm{W}_k$: Let $\bm{W}_k = \left [\bm{w}_1,\cdots ,\bm{w}_{m}   \right ] ^{\top}$ where $\bm{w}_j\sim \mathcal{N}(0,\mathbb{I}_{m}/m )$ and $f_{k}^{[j]}$ represents the $j$-th element of $\bm{f}_{k}$. Note that:

\begin{equation}
\left \| \bm{f}_k(\bm{x}) \right \|_2^2 = \sum_{j = 1}^{m} f_{k}^{[j]}(\bm{x})^2\,.
\label{eq:proof_lemma_C.1_in_ICML_1}
\end{equation}

Then we have:
\begin{equation}
\begin{split}
    \mathbb{E}_{\bm{W}_k} \left \| \bm{f}_k(\bm{x}) \right \|_2^2 & = \sum_{j = 1}^{m} \mathbb{E}_{\bm{w}_j} [f_{k}^{[j]}(\bm{x})^2]\\
    & = \sum_{j = 1}^{m}\mathbb{E}_{\bm{w}_j}\bigg(\bigg[\sigma_k\bigg(\left \langle \bm{w}_j, \bm{f}_{k-1}(\bm{x}) \right \rangle \bigg)+\alpha_{k-1}f_{k-1}^{[j]}(\bm{x})\bigg]^2\bigg) \quad\quad ~\cref{eq:network}\\
    & = \sum_{j = 1}^{m}\bigg(\mathbb{E}_{\bm{w}_j}\bigg[\bigg(\sigma_k(\left \langle \bm{w}_j, \bm{f}_{k-1}(\bm{x}) \right \rangle \bigg)^2\bigg]+\mathbb{E}_{\bm{w}_j}\bigg(\alpha_{k-1}^2 f_{k-1}^{[j]}(\bm{x})^2\bigg)\\
    & +\mathbb{E}_{\bm{w}_j}\bigg[2\sigma_k\bigg(\left \langle \bm{w}_j, \bm{f}_{k-1}(\bm{x}) \right \rangle\bigg) \alpha_{k-1} f_{k-1}^{[j]}(\bm{x})\bigg]\bigg)\\
    & =m\mathbb{E}_{w\sim \mathcal N(0,\left \| \bm{f}_{k-1}(\bm{x}) \right \|_2^2/m )}(\sigma_k(w)^2)+\sum_{j = 1}^{m}\alpha_{k-1}^2\mathbb{E}_{\bm{w}_j}\bigg( f_{k-1}^{[j]}(\bm{x})^2\bigg)\\
    & +2\sum_{j = 1}^{m}\alpha_{k-1}\mathbb{E}_{\bm{w}_j}\bigg(\sigma_k\bigg[\left \langle \bm{w}_j, \bm{f}_{k-1}(\bm{x}) \right \rangle\bigg]\bigg)\mathbb{E}_{\bm{w}_j}\bigg(f_{k-1}^{[j]}(\bm{x})\bigg)\\
    & = m\mathbb{E}_{w\sim \mathcal N(0,\left \| \bm{f}_{k-1}(\bm{x}) \right \|_2^2/m )}(\sigma_k(w)^2)+\alpha_{k-1}^2\left \| \bm{f}_{k-1}(\bm{x}) \right \|_2^2\\
    & +2\alpha_{k-1}\mathbb{E}_{w\sim \mathcal N(0,\left \| \bm{f}_{k-1}(\bm{x}) \right \|_2^2/m )}(\sigma_k(w))\sum_{j = 1}^{m}f_{k-1}^{[j]}(\bm{x})\,.
\end{split}
\label{eq:proof_lemma_C.1_in_ICML_2}
\end{equation}

According to~\cref{eq:Gii_upper_bound_RelU,eq:Gii_upper_bound_LeakyReLU,eq:Gii_upper_bound_Swish,eq:Gii_lower_bound_Swish,eq:Gii_upper_bound_Sigmoid_2,eq:Gii_lower_bound_Sigmoid,eq:Gii_upper_bound_Tanh_2,eq:Gii_lower_bound_Tanh}, we know that when $\sigma_{k-1}$ are in \text{ReLU}, \text{LeakyReLU}, \text{Sigmoid}, \text{Tanh} and \text{Swish} we have:

\begin{equation}
m\mathbb{E}_{w\sim \mathcal N(0,\left \| \bm{f}_{k-1}(\bm{x}) \right \|_2^2/m )}(\sigma_k(w)^2) = m\Theta \bigg(\frac{\left \| \bm{f}_{k-1}(\bm{x}) \right \|_2^2}{m}\bigg) = \Theta(\left \| \bm{f}_{k-1}(\bm{x}) \right \|_2^2)\,.
\label{eq:proof_lemma_C.1_in_ICML_3}
\end{equation}

When $\sigma_{k-1}$ is \text{ReLU}, \text{LeakyReLU} or \text{Swish},~\cref{eq:proof_lemma_C.1_in_ICML_3} can be written as:

\begin{equation*}
\frac{1}{2}\left \| \bm{f}_{k-1}(\bm{x}) \right \|_2^2 \leq m\mathbb{E}_{w\sim \mathcal N(0,\left \| \bm{f}_{k-1}(\bm{x}) \right \|_2^2/m )}(\sigma_k(w)^2) \leq (1+\eta^2)\left \| \bm{f}_{k-1}(\bm{x}) \right \|_2^2\,,
\end{equation*}

\begin{equation*}
0<\mathbb{E}_{w\sim \mathcal N(0,\left \| \bm{f}_{k-1}(\bm{x}) \right \|_2^2/m )}(\sigma_k(w))\leq\mathbb{E}_{w\sim \mathcal N(0,\left \| \bm{f}_{k-1}(\bm{x}) \right \|_2^2/m )}(f_{\mathrm{ReLU}}(w))=\frac{2\left \| \bm{f}_{k-1}(\bm{x}) \right \|_2}{5\sqrt{m}}\,.
\end{equation*}

According to the relationship between the vectors $1$-norm and $2$-norm, we have:
\begin{equation*}
-\sqrt{m}\left \| \bm{f}_{k-1}(\bm{x}) \right \|_2\leq\sum_{j = 1}^{m}f_{k-1}^{[j]}(\bm{x})\leq \sqrt{m}\left \| \bm{f}_{k-1}(\bm{x}) \right \|_2\,.
\end{equation*}

Then:
\begin{equation*}
-\frac{2\left \| \bm{f}_{k-1}(\bm{x}) \right \|_2^2}{5}\leq2\alpha_{k-1}\mathbb{E}_{w\sim \mathcal N(0,\left \| \bm{f}_{k-1}(\bm{x}) \right \|_2^2/m )}(\sigma_k(w))\sum_{j = 1}^{m}f_{k-1}^{[j]}(\bm{x})\leq \frac{2\left \| \bm{f}_{k-1}(\bm{x}) \right \|_2^2}{5}\,.
\end{equation*}

If we substitute into~\cref{eq:proof_lemma_C.1_in_ICML_2}, we have upper bound and lower bound for $\mathbb{E}_{\bm{W}_k} \left \| \bm{f}_k(\bm{x}) \right \|_2^2$:
\begin{equation*}
\begin{split}
    \mathbb{E}_{\bm{W}_k} \left \| \bm{f}_k(\bm{x}) \right \|_2^2 & =m\mathbb{E}_{w\sim \mathcal N(0,\left \| \bm{f}_{k-1}(\bm{x}) \right \|_2^2/m )}(\sigma_k(w)^2)+\alpha_{k-1}^2\left \| \bm{f}_{k-1}(\bm{x}) \right \|_2^2\\
    & +2\alpha_{k-1}\mathbb{E}_{w\sim \mathcal N(0,\left \| \bm{f}_{k-1}(\bm{x}) \right \|_2^2/m )}(\sigma_k(w))\sum_{j = 1}^{m}f_{k-1}^{[j]}(\bm{x})\\
    & \leq \bigg(1+\eta^2+\alpha_{k-1}+\frac{2}{5}\bigg)\left \| \bm{f}_{k-1}(\bm{x}) \right \|_2^2\\
    & \leq \bigg(\eta^2+\frac{12}{5}\bigg)\Theta(1)\,,
\end{split}
\end{equation*}

\begin{equation*}
\begin{split}
    \mathbb{E}_{\bm{W}_k} \left \| \bm{f}_k(\bm{x}) \right \|_2^2 & =m\mathbb{E}_{w\sim \mathcal N(0,\left \| \bm{f}_{k-1}(\bm{x}) \right \|_2^2/m )}(\sigma_k(w)^2)+\alpha_{k-1}^2\left \| \bm{f}_{k-1}(\bm{x}) \right \|_2^2\\
    & +2\alpha_{k-1}\mathbb{E}_{w\sim \mathcal N(0,\left \| \bm{f}_{k-1}(\bm{x}) \right \|_2^2/m)}(\sigma_k(w))\sum_{j = 1}^{m}f_{k-1}^{[j]}(\bm{x})\\
    & \geq \bigg(\frac{1}{2}+\alpha_{k-1}-\frac{2}{5}\bigg)\left \| \bm{f}_{k-1}(\bm{x}) \right \|_2^2\\
    & \geq \frac{1}{10}\Theta(1)\,.
\end{split}
\end{equation*}

That means, when $\sigma_{k-1}$ is \text{ReLU}, \text{LeakyReLU} or \text{Swish} we have:

\begin{equation}
    \mathbb{E}_{\bm{W}_k} \left \| \bm{f}_k(\bm{x}) \right \|_2^2 =\Theta(1)\,.
\label{eq:proof_lemma_C.1_in_ICML_10}
\end{equation}

When $\sigma_{k-1}$ is \text{Sigmoid} or \text{Tanh}, according to symmetry we have:
\begin{equation*}
\mathbb{E}_{w\sim \mathcal N(0,\left \| \bm{f}_{k-1}(\bm{x}) \right \|_2^2/m )}[\sigma_k(w)] = 0\,.
\end{equation*}

Then:
\begin{equation}
\begin{split}
    \mathbb{E}_{\bm{W}_k} \left \| f_k(\bm{x}) \right \|_2^2 & = \Theta(\left \| \bm{f}_{k-1}(\bm{x}) \right \|_2^2) + \alpha_{k-1}^2\left \| \bm{f}_{k-1}(\bm{x}) \right \|_2^2\\
    & = \Theta(\left \| \bm{f}_{k-1}(\bm{x}) \right \|_2^2)\,.
\end{split}
\label{eq:proof_lemma_C.1_in_ICML_12}
\end{equation}

By~\cref{eq:proof_lemma_C.1_in_ICML_10,eq:proof_lemma_C.1_in_ICML_12}, when $\sigma_{k-1}$ is \text{ReLU}, \text{LeakyReLU}, \text{Sigmoid}, \text{Tanh} or \text{Swish} we have:

\begin{equation*}
    \mathbb{E}_{\bm{W}_k} \left \| \bm{f}_k(\bm{x}) \right \|_2^2 =\Theta(1).
\end{equation*}

Thus, by applying Bernstein’s inequality to the sum of i.i.d. random variables in~\cref{eq:proof_lemma_C.1_in_ICML_1}, we have:

\begin{equation*}
\frac{1}{2} \mathbb{E}_{\bm{W}_k} \left \| \bm{f}_k(\bm{x}) \right \| _2^2 \leq \left \| \bm{f}_k(\bm{x}) \right \| _2^2 \leq \frac{3}{2} \mathbb{E}_{\bm{W}_k} \left \| \bm{f}_k(\bm{x}) \right \| _2^2\,,
\end{equation*}

with probability at least $1- \exp(-\Omega (m))$. i.e.:

\begin{equation*}
\left \| \bm{f}_{k}(\bm{x}) \right \|_2^2 = \Theta (1)\,,
\end{equation*}

with probability at least $1-\sum_{l=1}^{k}\exp(-\Omega (m))$.

The proof for $ \mathbb{E}_{x} \left \| \bm{f}_k(\bm{x}) \right \|_2^2$ can be done by following similar passages and using that $\left \| \mathbb{E}_{\bm{x}}[f_{k}^{[j]}(\bm{x})^2] \right \|_{\psi_1} \leq \mathbb{E}_{\bm{x}} \left \| f_{k}^{[j]}(\bm{x})^2 \right \|_{\psi_1}$.
\end{proof}

\begin{lemma}
\label{lemma:lemma_C.5_in_ICML}
Fix any layer $k \in [L-1]$, and $\bm{x}\sim P_X$. Then, we have that $\left \| \bm{D}_k \right \|_{\mathrm{F}}^2 = \Theta (m)$ with probability at least $1-\sum_{l=1}^{k}\exp(-\Omega (m))$ over $(\bm{W}_l)_{l=1}^k$ and $\bm{x}$.
\end{lemma}

\begin{proof}
By~\cref{lemma:lemma_C.1_in_ICML}, we have $f_{k-1}(\bm{x})\neq 0$ with probability at least $1-\sum_{l=1}^{k}\exp(-\Omega (m))$ over $(\bm{W}_l)_{l=1}^k$ and $\bm{x}$. Let us condition on this event and derive probability bounds over $\bm{W}_k$. Let $\bm{W}_k = [\bm{w}_1,\cdots ,\bm{w}_{m}]$. Then, $\left \| \bm{D}_k \right \|_{\mathrm{F}}^2 = \sum _{j=1}^{m}{\sigma'_k}^2(\left \langle \bm{f}_{k-1}(\bm{x}),\bm{w}_j \right \rangle)$. Thus:

\begin{equation*}
\begin{split}
    \mathbb{E}_{\bm{W}_k}\left \| \bm{D}_k \right \|_{\mathrm{F}}^2 = m\mathbb{E}_{\bm{w}_1}[{\sigma'_k}^2(\left \langle \bm{f}_{k-1}(\bm{x}),\bm{w}_1 \right \rangle)] = m\mathbb{E}_{w\sim \mathcal N(0,{\left \| \bm{f}_{k-1}(\bm{x}) \right \|_2^2}/{m} )}[{\sigma'_k}^2(w)]\,.
\end{split}
\end{equation*}

By~\cref{eq:dotGii_RelU,eq:dotGii_LeakyReLU,eq:dotGii_Sigmoid_1,eq:dotGii_Tanh_1,eq:dotGii_Swish_1,eq:dotGii_Swish_2}, we know that when $\sigma_{k}$ are in \text{ReLU}, \text{LeakyReLU}, \text{Sigmoid}, \text{Tanh} and \text{Swish} we have:

\begin{equation*}
\mathbb{E}_{\bm{W}_k}\left \| \bm{D}_k \right \|_{\mathrm{F}}^2 =m\Theta (1) = \Theta (m)\,.
\end{equation*}

By Hoeffding’s inequality on bounded random variables, we have:
\begin{equation*}
    \mathbb{P} \left(\left | \left \| \bm{D}_k \right \|_{\mathrm{F}}^2 - \mathbb{E}_{\bm{W}_k}\left \| \bm{D}_k \right \|_{\mathrm{F}}^2 \right |>t \right)\leq 2\exp \left(-\frac{2t^2}{m} \right)\,.
\end{equation*}

Picking $t := 0.01m$ concludes the proof.

\end{proof}

\begin{lemma}
\label{lemma:lemma_C.6_in_ICML}
For any $k \in [L-1]$, $k \leq p \leq L-1$ and $\bm{x}\sim P_X$, we have that:

\begin{equation*}
\small
\Theta \bigg(m\prod_{i=k+1}^{p}(\beta_3(\sigma_i)+\alpha_{i-1} ) \bigg) \leq \left \| \bm{D}_{k}\prod_{l=k+1}^{p}\bigg(\bm{W}_l \bm{D}_{l}+\alpha_{l-1}\bm{I}_{m\times m}\bigg) \right \|_F^2 \leq \Theta \bigg(m\prod_{i=k+1}^{p}(\beta_2(\sigma_i)+\alpha_{i-1} ) \bigg)\,,
\end{equation*}

with probability at least $1-\sum_{l=k+1}^{p}\exp(-\Omega (m))$ over $(\bm{W}_l)_{l=k+1}^p$ and $\bm{x}$.

\end{lemma}

\begin{proof}
We prove this by induction on $p$.

~\cref{lemma:lemma_C.5_in_ICML} implies that the statement holds for $p=k$. 

Suppose it holds for some $p-1$. Let $\bm{S}_p = \bm{D}_k\prod _{l=k+1}^{p} (\bm{W}_l \bm{D}_l + \alpha_{l-1}\bm{I}_{m\times m})$. Then,$\bm{S}_p = \bm{S}_{p-1}(\bm{W}_p \bm{D}_p+\alpha_{p-1}\bm{I}_{m\times m}) =  \bm{S}_{p-1}\bm{W}_p \bm{D}_p + \alpha_{p-1}\bm{S}_{p-1} $. Let $\bm{W}_p = [\bm{w}_1,\dots, \bm{w}_{m}]$. Then:

\begin{equation*}
    \left \| \bm{S}_p \right \| _{\mathrm{F}}^2 = \sum_{j=1}^{m} \left \| \bm{S}_{p-1} \bm{w}_j \right \| _2^2{\sigma'_p}(\left \langle \bm{f}_{p-1}(\bm{x}),\bm{w}_j \right \rangle )^2 + \alpha_{p-1}\left \| \bm{S}_{p-1} \right \| _{\mathrm{F}}^2 + 2 \alpha_{p-1}\left \langle \bm{S}_{p-1}\bm{W}_p \bm{D}_p, \bm{S}_{p-1} \right \rangle \,.
\end{equation*}

Then we have:

\begin{equation*}
\begin{split}
    \mathbb{E}_{\bm{W}_p} \left \| \bm{S}_p \right \| _{\mathrm{F}}^2 & = m \mathbb{E}_{\bm{w} \sim \mathcal N(0,\mathbb{I}_{m}/m )}\left \| \bm{S}_{p-1}\bm{w} \right \| _2^2{\sigma'_p}(\left \langle \bm{f}_{p-1}(\bm{x}),\bm{w} \right \rangle )^2 + \alpha_{p-1}\left \| \bm{S}_{p-1} \right \| _{\mathrm{F}}^2 + \mathbb{E}_{\bm{W}_p} 2 \alpha_{p-1}\left \langle \bm{S}_{p-1}\bm{W}_p \bm{D}_p, \bm{S}_{p-1} \right \rangle\\
    & = m \mathbb{E}_{\bm{w} \sim \mathcal N(0, \mathbb{I}_{m}/m )}\left \| \bm{S}_{p-1}\bm{w} \right \| _2^2\mathbb{E}_{\bm{w} \sim \mathcal N(0, \mathbb{I}_{m}/m )}{\sigma'_p}(\left \langle \bm{f}_{p-1}(\bm{x}),\bm{w} \right \rangle )^2+ \alpha_{p-1}\left \| \bm{S}_{p-1} \right \| _{\mathrm{F}}^2+0\\
    & = \left \| \bm{S}_{p-1} \right \| _{\mathrm{F}}^2\mathbb{E}_{\bm{w} \sim \mathcal N(0, \left \| \bm{f}_{p-1}(\bm{x}) \right \|_2^2/m)}{\sigma'_p}(w)^2+\alpha_{p-1}\left \| \bm{S}_{p-1} \right \| _{\mathrm{F}}^2\,.
\end{split}
\end{equation*}

From the previous result~\cref{eq:dotGii_RelU,eq:dotGii_LeakyReLU,eq:dotGii_Sigmoid_1,eq:dotGii_Tanh_1,eq:dotGii_Swish_1,eq:dotGii_Swish_2}, we have:

\begin{equation*}
    \beta_3(\sigma_p) \leq \mathbb{E}_{w \sim \mathcal N(0, \left \| \bm{f}_{p-1}(\bm{x}) \right \|_2^2/m)}{\sigma'_p}(w)^2 \leq \beta_2(\sigma_p)\,.
\end{equation*}

That is:
\begin{equation*}
    (\beta_3(\sigma_p)+\alpha_{p-1})\left \| \bm{S}_{p-1} \right \| _{\mathrm{F}}^2 \leq\mathbb{E}_{\bm{W}_p} \left \| \bm{S}_p \right \| _{\mathrm{F}}^2 \leq (\beta_2(\sigma_p)+\alpha_{p-1})\left \| \bm{S}_{p-1} \right \| _{\mathrm{F}}^2\,.
\end{equation*}

Moreover:

\begin{equation*}
\left \| \left \| \bm{S}_{p-1}\bm{w}_j \right \| _2^2{\sigma'_p}(\left \langle \bm{f}_{p-1}(\bm{x}),\bm{w}_j \right \rangle )^2 \right \| _{\psi _1} \leq \left \| \left \| \bm{S}_{p-1}\bm{w}_j \right \| _2 \right \| _{\psi _2}^2 \leq \frac{c}{m}\left \| \bm{S}_{p-1} \right \| _{\mathrm{F}}^2\,.
\end{equation*}

By Bernstein’s inequality ~\citep{vershynin12}, we have:

\begin{equation*}
\frac{1}{2} \mathbb{E}_{\bm{W}_p} \left \| \bm{S}_p \right \| _{\mathrm{F}}^2 \leq \left \| \bm{S}_p \right \| _{\mathrm{F}}^2 \leq \frac{3}{2} \mathbb{E}_{\bm{W}_p} \left \| \bm{S}_p \right \| _{\mathrm{F}}^2\,,
\end{equation*}

with probability at least $1-\exp(-\Omega (m))$. 
Finally, taking the intersection of all the events finishes the proof.

\end{proof}

\begin{lemma}
\label{lemma:lemma_C.6_in_ICML_}
For any layer $k \in [L-2]$ and $\bm{x}\sim P_X$, we have:

\begin{equation*}
\small
\Theta \bigg(\prod_{i=k+1}^{L-1}(\beta_3(\sigma_i)+\alpha_{i-1} ) \bigg) \leq \left \| \bm{D}_{k}\prod_{l=k+1}^{L-1}\bigg(\bm{W}_l \bm{D}_{l}+\alpha_{l-1}\bm{I}_{m\times m}\bigg) \bm{W}_L \right \|_2^2 \leq \Theta \bigg(\prod_{i=k+1}^{L-1}(\beta_2(\sigma_i)+\alpha_{i-1} )\bigg) \,,
\end{equation*}

with probability at least $1-\sum_{l=k+1}^{L-1}\exp(-\Omega (m))-\exp(-\Omega (1))$.
\end{lemma}

\begin{proof}

Let $\bm{B} = \bm{D}_{k}\prod_{l=k+1}^{L-1}(\bm{W}_l \bm{D}_{l}+\alpha_{l-1}\bm{I}_{m\times m})$.

By~\cref{lemma:lemma_C.6_in_ICML}, we have:

\begin{equation}
\Theta \bigg(m\prod_{i=k+1}^{L-1}(\beta_3(\sigma_i)+\alpha_{i-1} ) \bigg) \leq \left \| \bm{B} \right \|_{\mathrm{F}}^2 \leq \Theta \bigg(m\prod_{i=k+1}^{L-1}(\beta_2(\sigma_i)+\alpha_{i-1} )\bigg) \,,
\label{eq:proof_lemma_C.6_in_ICML_1_}
\end{equation}

with probability at least $1-\sum_{l=k+1}^{L-1}\exp(-\Omega (m))$.

Then, by Hanson-Wright inequality ~\citep{vershynin12}, we have:

\begin{equation}
\frac{1}{2m}\left \| \bm{B} \right \| _{\mathrm{F}}^2=\frac{1}{2}\mathbb{E}_{\bm{W}_L} \left \| \bm{B}\bm{W}_L \right \| _2^2 \leq \left \| \bm{B}\bm{W}_L \right \| _2^2 \leq \frac{3}{2}\mathbb{E}_{\bm{W}_L} \left \| \bm{B}\bm{W}_L \right \| _2^2 = \frac{3}{2m}\left \| \bm{B} \right \| _{\mathrm{F}}^2\,,
\label{eq:proof_lemma_C.6_in_ICML_2_}
\end{equation}

with probability at least $1-\exp(-\Omega (\left \| \bm{B} \right \|_{\mathrm{F}}^2/\left \| \bm{B} \right \|_2^2)) \geq 1-\exp(-\Omega (1))$ over $\bm{W}_L$.

According to~\cref{eq:proof_lemma_C.6_in_ICML_1_,eq:proof_lemma_C.6_in_ICML_2_}, we can get the result.

\end{proof}

\subsection{Results for mixed activation functions under the finite-width setting (Proof of~\cref{thm:lambda_min_finite})}
\label{ssec:Results_mixed_finitely_width}

\begin{proof}
We firstly present the lower bound of the minimal eigenvalue of $\bm{JJ}^{\top}$ and then derive its upper bound.
\
\paragraph{Lower bound\\}
For PSD matrices $\bm{P}, \bm{Q} \in \mathbb{R}^{N \times N}$, it holds $\lambda_{\min}(\bm{P}\circ \bm{Q})\geq \lambda_{\min}(\bm{P}) \min_{i \in [N]}Q_{ii}$. Then, by~\cref{thm:lambda_min_inf_mixed} and Theorem 5.1 of ~\citet{pmlr-v139-nguyen21g}:

\begin{equation*}
\begin{split}
\lambda_{\min}(\bm{JJ}^{\top})&\geq \sum_{k=0}^{L-1}\lambda_{\min}(\bm{F}_k \bm{F}_k^{\top})\min_{i\in[N]}\left \| (\bm{B}_{k+1})_{i:} \right \|_2^2\\
&\geq \lambda_{\min}(\bm{F}_0 \bm{F}_0^{\top})\min_{i\in[N]}\left \| (\bm{B}_{1})_{i:} \right \|_2^2\\
&\geq \Theta \bigg(\prod_{i=2}^{L-1}(\beta_3(\sigma_i)+\alpha_{i-1} ) \bigg)\,,
\end{split}
\end{equation*}

with probability at least $1-\sum_{l=1}^{L-1}\exp(-\Omega (m))-\exp(-\Omega (1))$. where the last inequality hold by~\cref{lemma:lemma_C.6_in_ICML_} and~\cref{thm:lambda_min_inf_mixed}.

\paragraph{Upper bound\\}

For \text{ReLU}, \text{LeakyReLU} and \text{Swish} we have:

\begin{equation*}
\begin{split}
    \lambda_{min}(\bm{JJ}^{\top})\leq \sum_{i=0}^{N}(\bm{JJ}^{\top})_{ii}/N & = \frac{1}{N}\sum_{i=0}^{N}\sum_{k=0}^{L-1}\left \| (\bm{F}_k)_{1:} \right \|_2^2 \left \| (\bm{B}_{k+1})_{1:} \right \|_2^2\\
    & = \frac{1}{N}\sum_{i=0}^{N}\sum_{k=0}^{L-1}\left \| \bm{f}_k(\bm{x_1}) \right \|_2^2 \left \| (\bm{B}_{k+1})_{1:} \right \|_2^2\,.
    \end{split}
\end{equation*}

By~\cref{lemma:lemma_C.1_in_ICML} and~\cref{lemma:lemma_C.6_in_ICML_} we have:

\begin{equation*}
\lambda_{min}(\bm{JJ}^{\top}) \leq \sum_{k=0}^{L-1}\Theta \bigg(\prod_{i=k+2}^{L-1}(\beta_2(\sigma_i)+\alpha_{i-1} ) \bigg)\,,
\end{equation*}

with probability at least $1-\sum_{l=1}^{L-1}\exp(-\Omega (m))-\exp(-\Omega (1))$.

\end{proof}

\section{Generalization error via the minimum eigenvalue of NTK}
\label{sec:Relationship_NTK_Generalization}

In this section, firstly, we provide some useful lemmas in~\cref{ssec:relevant_lemmas_generalization},then present the proof of~\cref{thm:NTK_Generalization} in~\cref{ssec:Proof_of_NTK_Generalization}.

\subsection{Relevant Lemmas}
\label{ssec:relevant_lemmas_generalization}

\begin{lemma}
\label{lemma:bound_for_Gaussian_matrix_fabinus_norm}
(~\citet[Theorem 4.4.5]{vershynin12}) Let $\bm{A}$ be an $N \times n$ matrix whose entries are independent standard normal random variables. Then for every $t \geq 0$, with probability at least $1-2\exp(-t^2/2)$, one has:
\begin{equation*}
s(\bm{A})_{\max} \leq \sqrt{N} + \sqrt{n} +t\,.
\end{equation*}
\end{lemma}

We need the following lemma to show that the output of each neuron with any activation function does not change too much if the input weights are close.
\begin{lemma}
\label{lemma:bound_for_perturbation}
Let $\bm{W} \in \mathbb{R} ^{m \times m}$ be the random Gaussian matrix with ${W}_{i,j} \sim \mathcal{N} (0, 1/m)$, $\mathrm{Lip}_{\max}$ be the maximum value of the Lipschitz constants of the all activation functions, with $\omega = \mathcal{O}((3\mathrm{Lip}_{\max}+1)^{-(L-1)})$, assuming $\widetilde{\bm{W}} \in \mathcal{B} (\bm{W},\omega )$, for any $l \in[L]$, it holds that $\left \| \hat{\bm{f}}_{i,l} \right \| _2 = \mathcal{O}(1)$ with probability at least $1-2l\exp(-m/2)-l\exp(-\Omega (m))$.
\end{lemma}
\begin{proof}
\label{proof:bound_for_perturbation_ResNet}
We provide the estimation on $ \hat{\bm{f}}_{i,1}$ and $ \hat{\bm{f}}_{i,l}$ ($l=2,3,\cdots,L$) in~\cref{def:problem_setting_feature_map_and_perturbation}, respectively. 
Firstly, $ \hat{\bm{f}}_{i,1}$ admits:
\begin{equation*}
\begin{split}
\left \| \hat{\bm{f}}_{i,1} \right \| _2 & = \left \| \widetilde{\bm{f}}_{i,1} - \bm{f}_{i,1} \right \| _2 = \left \| \sigma_1(\widetilde{\bm{W}}_1\bm{x}_i) - \sigma_1(\bm{W}_1\bm{x}_i) \right \| _2\\
& \leq \mathrm{Lip}_{\sigma_{1}}\left \| \widetilde{\bm{W}}_1 - \bm{W}_1 \right \| _2 \left \|\bm{x}_i \right \| _2 \leq \omega \mathrm{Lip}_{\sigma_{1}}\\ 
& = \mathcal{O}(1)\,.
\end{split}
\end{equation*}

For $ \hat{\bm{f}}_{i,l}$ with $l=2,3,\dots,L$, we have:
\begin{equation}
\small
\begin{split}
\left \| \hat{\bm{f}}_{i,l} \right \| _2 & = \left \| \widetilde{\bm{f}}_{i,l} - \bm{f}_{i,l} \right \| _2\\
& = \left \| \sigma_{l}(\widetilde{\bm{W}}_l\widetilde{\bm{f}}_{i,l-1}) + \alpha_{l-1}\widetilde{\bm{f}}_{i,l-1} - \sigma_{l}(\bm{W}_l\bm{f}_{i,l-1})-\alpha_{l-1}\bm{f}_{i,l-1} \right \| _2\\
& \leq \left \| \sigma_{l}(\widetilde{\bm{W}}_l\widetilde{\bm{f}}_{i,l-1}) - \sigma_{l}(\bm{W}_l\bm{f}_{i,l-1}) \right \| _2 + \alpha_{l-1}\left \| \hat{\bm{f}}_{i,l-1} \right \| _2\\ 
& \leq \mathrm{Lip}_{\sigma_{l}}\left \| \widetilde{\bm{W}}_l\widetilde{\bm{f}}_{i,l-1} - \bm{W}_l\bm{f}_{i,l-1} \right \| _2 + \left \| \hat{\bm{f}}_{i,l-1} \right \| _2\quad\quad\text{[Lipschitz continuity of $\sigma_{l}$]}\\
& = \mathrm{Lip}_{\sigma_{l}}\left \| \bm{W}_l(\widetilde{\bm{f}}_{i,l-1} - \bm{f}_{i,l-1})+(\widetilde{\bm{W}}_l-\bm{W}_l)\widetilde{\bm{f}}_{i,l-1} \right \| _2 + \left \| \hat{\bm{f}}_{i,l-1} \right \| _2\\
& \leq \mathrm{Lip}_{\sigma_{l}}\left\{\left \| \bm{W}_l(\widetilde{\bm{f}}_{i,l-1} - \bm{f}_{i,l-1})\right \| _2+\left \| (\widetilde{\bm{W}}_l-\bm{W}_l)\widetilde{\bm{f}}_{i,l-1} \right \| _2\right\} + \left \| \hat{\bm{f}}_{i,l-1} \right \| _2\\ 
& \leq \mathrm{Lip}_{\sigma_{l}}\left\{\left \| \bm{W}_l\right \| _2 \left \|\widetilde{\bm{f}}_{i,l-1} - \bm{f}_{i,l-1}\right \|_2+\left \| \widetilde{\bm{W}}_l-\bm{W}_l\right \| _2 \left \|\widetilde{\bm{f}}_{i,l-1} \right \| _2\right\} + \left \| \hat{\bm{f}}_{i,l-1} \right \| _2\\
& \leq (\mathrm{Lip}_{\sigma_{l}}\left \| \bm{W}_l\right \| _2+1) \left \|\hat{\bm{f}}_{i,l-1}\right \|_2+\mathrm{Lip}_{\sigma_{l}}\omega\bigg  (\left \|\widetilde{\bm{f}}_{i,l-1} - \bm{f}_{i,l-1}\right \| _2 + \left \| \bm{f}_{i,l-1}\right \| _2\bigg)\\
& = \left\{\mathrm{Lip}_{\sigma_{l}}(\left \| \bm{W}_l\right \| _2+\omega)+1\right\} \left \|\hat{\bm{f}}_{i,l-1}\right \|_2+\mathrm{Lip}_{\sigma_{l}}\omega \left \| \bm{f}_{i,l-1}\right \| _2\,.
\end{split}
\label{eq:bound_for_perturbation_2}
\end{equation}

By~\cref{lemma:bound_for_Gaussian_matrix_fabinus_norm}, choosing $t = \sqrt{m}$, with probability at least $1-2\exp(-m/2)$, we have:

\begin{equation*}
\left \| \bm{W}_l\right \| _2 = s(\bm{W}_l)_{\max} \leq \frac{\sqrt{m} + \sqrt{m} +\sqrt{m}}{\sqrt{m}} = 3\,.
\end{equation*}

Then, $\|\hat{\bm{f}}_{i,l}\|_2$ in~\cref{eq:bound_for_perturbation_2} can be further upper bounded with probability at least $1-2l\exp(-m/2)-l\exp(-\Omega (m))$:
\begin{equation*}
\begin{split}
\left \| \hat{\bm{f}}_{i,l} \right \| _2 & \leq \bigg((3 + \omega) \mathrm{Lip}_{\max}+1\bigg) \left \| \hat{\bm{f}}_{i,l-1}\right \|_2+ \mathrm{Lip}_{\max} \omega \left \| \bm{f}_{i,l-1}\right \| _2\\
& \leq \bigg([(3 + \omega) \mathrm{Lip}_{\max}+1]^{l-1}-1\bigg)\bigg(\mathrm{Lip}_{\sigma_{1}}\omega + \frac{\mathrm{Lip}_{\max}\omega \left \| \bm{f}_{i,l-1}\right \| _2}{(3+\omega)\mathrm{Lip}_{\max}}\bigg)+\mathrm{Lip}_{\sigma_1}\omega \\
& \leq (3\mathrm{Lip}_{\max}+1)^{L-1}\Theta (1)\omega+\mathrm{Lip}_{\sigma_1}\omega\\
& = \mathcal{O}(1)\Theta (1)+ \mathcal{O}(1)\\
& = \mathcal{O}(1)\,,
\end{split}
\end{equation*}
where the second inequality holds by the recursion which conclude the proof.

\end{proof}

We also need the following lemma, demonstrating that the neural network function is almost linear in terms of its weights if the initializations are close to each other.
\begin{lemma}
\label{lemma:lemma_4.1_in_GuQuanquan}
Let $\bm{W}, \bm{W}' \in \mathcal{B} (\bm{W}^{(0)},\omega )$  with $\omega = \mathcal{O}((3\mathrm{Lip}_{\max}+1)^{-(L-1)})$, for any $i \in [N]$,
with probability at least $1-2(L-1)\exp(-m/2)-L\exp(-\Omega (m))-2/m$, we have:
\begin{equation*}
\left | f(\bm {x}_i;\bm{W}') - f(\bm {x}_i;\bm{W}) - \left \langle \nabla  f(\bm {x}_i;\bm{W}), \bm{W}' -\bm{W}  \right \rangle \right | = \mathcal{O}(1)\,.
\end{equation*}
\end{lemma}

\begin{proof}
\label{proof:lemma_4.1_in_GuQuanquan}
We have the following expression:
\begin{equation}
\small
\begin{split}
    & \left | f(\bm {x}_i;\bm{W}') - f(\bm {x}_i;\bm{W}) - \left \langle \nabla  f(\bm {x}_i;\bm{W}), \bm{W}' -\bm{W}  \right \rangle \right | \\
    = & \left |\sum_{l=1}^{L-1}\bm{W}_{L}\prod_{r=l+1}^{L-1}(\bm{D}_{i,r}\bm{W}_r+\alpha_{r-1}\bm{I}_{m\times m})\bm{D}_{i,l}(\bm{W}'_l-\bm{W}_l)\bm{f}_{i,l-1}+\bm{W}'_L(\bm{f}'_{i,L-1}-\bm{f}_{i,L-1})  \right |\\
    \leq & \sum_{l=1}^{L-1}\left |\bm{W}_{L}\prod_{r=l+1}^{L-1}(\bm{D}_{i,r}\bm{W}_r +\alpha_{r-1}\bm{I}_{m\times m})\bm{D}_{i,l}(\bm{W}'_l-\bm{W}_l)\bm{f}_{i,l-1}\right |+\left | \bm{W}'_L(\bm{f}'_{i,L-1}-\bm{f}_{i,L-1})   \right |\\
    \leq & \sum_{l=1}^{L-1}\left \| \bm{W}_{L} \right \|_2\left \|  \prod_{r=l+1}^{L-1}(\bm{D}_{i,r}\bm{W}_r +\alpha_{r-1}\bm{I}_{m\times m})\bm{D}_{i,l}(\bm{W}'_l-\bm{W}_l)\bm{f}_{i,l-1} \right \|_2 \\
    & +\left \| \bm{W}'_L\right \|_2\left \| \bm{f}'_{i,L-1}-\bm{f}_{i,L-1} \right \|_2 \\
    \leq & \sum_{l=1}^{L-1}\left \| \bm{W}_{L} \right \|_2  \prod_{r=l+1}^{L-1}(\left \|  \bm{D}_{i,r}\right \|_2\left \|\bm{W}_r\right \|_2 +\alpha_{r-1})\left \|  \bm{D}_{i,l}\right \|_2\left \| \bm{W}'_l-\bm{W}_l \right \|_2 \left \|\bm{f}_{i,l-1} \right \|_2 +\left \| \bm{W}'_L\right \|_2\left \| \bm{f}'_{i,L-1}-\bm{f}_{i,L-1} \right \|_2\,.
\end{split}
\label{eq:proof_lemma_4.1_in_GuQuanquan_1}
\end{equation}

Here we require the derivative of the activation function ${\sigma}'$ is bound, i.e., $\left \|\bm{D}\right \|_2 \leq \mathrm{Lip}_{\max}$.
The considered activation functions in this paper satisfy this condition.

By~\cref{lemma:bound_for_Gaussian_matrix_fabinus_norm},~\cref{lemma:bound_for_perturbation} and~\cref{lemma:lemma_C.1_in_ICML} with probability at least $1-2(L-1)\exp(-m/2)-L\exp(-\Omega (m))$, we have $\left \| \bm{f}'_{i,L-1}-\bm{f}_{i,L-1} \right \|_2 \leq \mathcal{O}(1)$, $\left \|\bm{f}_{i,l-1} \right \|_2 = \Theta(1)$ and $\left \|\bm{W}_r\right \|_2 \leq 3 \quad \forall r \in [L-1]$.

Moreover, $m\left \| \bm{W}_{L} \right \|_2^2$ is a random Variables obey chi-square distribution with $m$ degrees of freedom. That means $\mathbb{E} (m\left \| \bm{W}_{L} \right \|_2^2)=m$ and $\mathbb{V}(m\left \| \bm{W}_{L} \right \|_2^2)=2m$. By Chebyshev's Inequality we have $P(|m\left \| \bm{W}_{L} \right \|_2^2-m|\geq m)\leq 2m/m^{2}$. i.e.:
\begin{equation*}
\left \| \bm{W}_{L} \right \|_2\leq \sqrt{2}\,,
\end{equation*}

with probability at least $1-2/m$.
 
Accordingly,~\cref{eq:proof_lemma_4.1_in_GuQuanquan_1} can be further upper bounded by:
\begin{equation*}
\begin{split}
    & \left | f(\bm {x}_i;\bm{W}') - f(\bm {x}_i;\bm{W}) - \left \langle \nabla  f(\bm {x}_i;\bm{W}), \bm{W}' -\bm{W}  \right \rangle \right | \\
    & \leq \sum_{l=1}^{L-1}(3\mathrm{Lip}_{\max} + 1)^{L-l-1}\omega\sqrt{2}\mathrm{Lip}_{\max}\Theta(1)  +(\sqrt{2}+\omega)\mathcal{O}(1)\\
    &  = \frac{(3\mathrm{Lip}_{\max}+1)^{L-1}-1}{3\mathrm{Lip}_{\max}}\omega\sqrt{2}\mathrm{Lip}_{\max}\Theta(1)  +(\sqrt{2}+\omega)\mathcal{O}(1)\\
    &  = \mathcal{O}(1)\,.
\end{split}
\end{equation*}
\end{proof}

We define $L_i(\bm{W})  = \ell[ y_i f(\bm{x}_i;\bm{W}) ]$, then the following lemma shows that, $L_i(\bm{W})$ is almost a convex function of $\bm{W}$ for any $i \in [N]$ 
if the initilizations are close.
\begin{lemma}
\label{lemma:lemma_4.2_in_GuQuanquan}
Let $\bm{W}, \bm{W}' \in \mathcal{B} (\bm{W}^{(0)},\omega )$  with $\omega = \mathcal{O}((3\mathrm{Lip}_{\max}+1)^{-(L-1)})$ 
, for any $i\in [N]$, it holds that:
\begin{equation*}
L_i(\bm{W}') \geq L_i(\bm{W}) + \left \langle \nabla_{\bm{W}} L_i(\bm{W}), \bm{W}'-\bm{W} \right \rangle -\mathcal{O}(1)\,,
\end{equation*}
with probability at least $1-2(L-1)\exp(-m/2)-L\exp(-\Omega (m))-2/m$.
\end{lemma}

\begin{proof}
\label{proof:lemma_4.2_in_GuQuanquan}

By the convexity of $\ell(z)$, we have:
\begin{equation*}
\small
    L_i({\bm{W}}') - L_i(\bm{W}) =  \ell[y_i  f(\bm{x}_i;{\bm{W}}') ] - \ell[ y_i f(\bm{x}_i;\bm{W}) ] \geq  \ell'[ y_i f(\bm{x}_i;\bm{W})] \cdot y_i\cdot [ f(\bm{x}_i;{\bm{W}}') - f(\bm{x}_i;\bm{W}) ]\,.
\end{equation*}
Using the chain rule leads to:
\begin{equation*}
\sum_{l=1}^L \left \langle  \nabla_{\bm{W}_l} L_i(\bm{W}), \bm{W}_l' - \bm{W}_l \right \rangle  = \ell'[ y_i f(\bm{x}_i;\bm{W})] \cdot y_i\cdot \left \langle  \nabla f(\bm{x}_i;\bm{W}) , {\bm{W}}' - \bm{W} \right \rangle\,.
\end{equation*}
Combining the above two equations, by triangle inequality, we have:
\begin{equation*}
\small
\begin{split}
        \ell'[ y_i f(\bm{x}_i;\bm{W})]\cdot y_i\cdot[ f(\bm{x}_i;{\bm{W}}') - f(\bm{x}_i;\bm{W}) ] & \geq \ell'[ y_i f(\bm{x}_i;\bm{W})] \cdot y_i\cdot  \left \langle  \nabla f(\bm{x}_i;\bm{W}) , {\bm{W}}' - \bm{W}  \right \rangle   - \varepsilon \\
    & = \textstyle \sum_{l=1}^L  \left \langle  \nabla_{\bm{W}_l} L_i(\bm{W}), \bm{W}_l' - \bm{W}_l  \right \rangle  - \varepsilon\,,
\end{split}
\end{equation*}
where $\varepsilon := | \ell'[ y_i f(\bm{x}_i;\bm{W})] \cdot y_i\cdot [ f(\bm{x}_i;{\bm{W}}') - f(\bm{x}_i;\bm{W})  - \left \langle \nabla f(\bm{x}_i;\bm{W}) , {\bm{W}}' - \bm{W} \right \rangle ] |$.  Then by upper-bounding $\varepsilon$ with~\cref{lemma:lemma_4.1_in_GuQuanquan} and the fact that $| \ell'[ y_i f(\bm{x}_i;\bm{W})] \cdot y_i | \leq 1$, we have:
\begin{equation*}
\begin{split}
    L_i({\bm{W}}') - L_i(\bm{W}) & \geq \sum_{l=1}^L \left \langle \nabla_{\bm{W}_l} L_i(\bm{W}), \bm{W}_l' - \bm{W}_l \right \rangle - \varepsilon\\
     & = \sum_{l=1}^L \left \langle \nabla_{\bm{W}_l} L_i(\bm{W}), \bm{W}_l' - \bm{W}_l \right \rangle - \mathcal{O}(1)\,.
\end{split}
\end{equation*}
\end{proof}

We need the following lemma to show that, the gradient of the neural network function can be upper bounded under near initialization.
\begin{lemma}
\label{lemma:lemma_B.3_in_GuQuanquan}
Let $\bm{W} \in \mathcal{B} (\bm{W}^{(0)},\omega )$   with $\omega = \mathcal{O}((3\mathrm{Lip}_{\max}+1)^{-(L-1)})$, for any $i \in [N]$, with probability at least $1-2(L-l)\exp(-m/2)-l\exp(-\Omega (m))-2/m$, it holds that:
\begin{equation*}
\left \| \nabla_{\bm{W}_l} f(\bm {x}_i;\bm{W}) \right \| _2, \left \| \nabla_{\bm{W}_l} L_i(\bm{W}) \right \| _2 \leq  \Theta(3\mathrm{Lip}_{\max}+1)^{L-l}\,.
\end{equation*}

\end{lemma}

\begin{proof}
\label{proof:lemma_B.3_in_GuQuanquan}

According to the triangle inequality and definition of operator norm, we have:

\begin{equation*}
\begin{split}
    \left \| \nabla_{\bm{W}_l} f(\bm {x} _i;\bm W  ) \right \| _2 & = \left \| \bm{f}_{i,l-1}\bm{W}_L\prod_{r=l+1}^{L-1}(\bm{D}_{i,r}\bm{W}_r+\alpha_{r-1}\bm{I}_{m\times m})\bm{D}_{i,l} \right \| _2\\
    & \leq \left \| \bm{f}_{i,l-1}\right \| _2 \left \|\bm{W}_L\prod_{r=l+1}^{L-1}(\bm{D}_{i,r}\bm{W}_r+\alpha_{r-1}\bm{I}_{m\times m})\bm{D}_{i,l} \right \| _2\\
    & \leq \left \| \bm{f}_{i,l-1}\right \| _2 \left \| \bm{W}_{L} \right \|_2  \prod_{r=l+1}^{L-1}(\left \|  \bm{D}_{i,r}\right \|_2\left \|\bm{W}_r\right \|_2 +\alpha_{r-1})\left \|  \bm{D}_{i,l}\right \|_2\,.
\end{split}
\end{equation*}

By~\cref{lemma:lemma_C.1_in_ICML} and~\cref{lemma:bound_for_Gaussian_matrix_fabinus_norm}, with probability at least $1-2(L-l-1)\exp(-m/2)-l\exp(-\Omega (m))-2/m$ we have $\left \| \bm{f}_{i,l-1}\right \| _2 = \Theta(1)$, $\left \| \bm{W}_i^{(0)}\right \| _2 \leq 3$ for $i = l+1,\cdots,L-1$, $\left \| \bm{W}_L^{(0)}\right \| _2 \leq \sqrt{2}$ and $\left \|\bm{D}\right \|_2 \leq \mathrm{Lip}_{\max}$ due to $\sigma'$ is bounded, then we have:

\begin{equation*}
    \left \| \nabla_{\bm{W}_l} f(\bm {x} _i;\bm W ) \right \| _2 \leq \Theta(1)(3\mathrm{Lip}_{\max}+1)^{L-l-1}\sqrt{2}\mathrm{Lip}_{\max} = \Theta(3\mathrm{Lip}_{\max}+1)^{L-l-1}\,,
\end{equation*}
which implies:
\begin{equation*}
\small
\left \| \nabla_{\bm{W}_l} L_i(\bm{W}) \right \| _2 \leq \left | {\ell}'[y_i\cdot f(\bm {x}_i;\bm{W})]\cdot y_i \right | \cdot \left \| \nabla_{\bm{W}_l} f(\bm {x}_i;\bm{W}) \right \| _2 \leq \left \| \nabla_{\bm{W}_l} f(\bm {x} _i;\bm W) \right \| _2 \leq \Theta(3\mathrm{Lip}_{\max}+1)^{L-l-1}\,,
\end{equation*}
where we use the fact that $| \ell'[ y_i f(\bm{x}_i;\bm{W})] \cdot y_i | \leq 1$.
\end{proof}

We need the following lemma to show that, the cumulative loss can be upper bounded under small changes on the parameters (i.e., weights).  
\begin{lemma}
\label{lemma:lemma_4.3_in_GuQuanquan}
    For any $\epsilon, \delta, R > 0$, there exists:
    \begin{equation*}
    m^{\star}= \frac{(3\mathrm{Lip}_{\max}+1)^{4L-4}L^2R^4}{4\varepsilon^2}\,,
    \end{equation*}
    such that if $m \geq m^* (\epsilon, \delta, R,L)$, then with probability at least $1 - \delta$ over the randomness of $\bm{W}^{(1)}$, for any $\bm{W}^* \in \mathcal{B} (\bm{W}^{(1)}, R m^{-1/2})$, Algorithm~\ref{alg:algorithm_DARTS} with $\gamma = \varepsilon/[m(3\mathrm{Lip}_{\max}+1)^{2L-2}]$, $N = LR^2(3\mathrm{Lip}_{\max}+1)^{2L-2}/(2\varepsilon^2)$, the cumulative loss can be upper bounded by:
    \begin{equation*}
        \sum_{i=1}^N L_i(\bm{W}^{(i)}) \leq \sum_{i=1}^N L_i(\bm{W}^{*}) + 3N\epsilon\,.
    \end{equation*}

\end{lemma}

{\bf Remark:} Discussion on the required width $m$ refer to \cref{sec:discussion}.
\begin{proof}
Set $\omega = 1/(3\mathrm{Lip}_{\max}+1)^{L-1}$ such that the conditions on $\omega$ given in Lemmas~\ref{lemma:bound_for_perturbation},~\ref{lemma:lemma_4.1_in_GuQuanquan},~\ref{lemma:lemma_4.2_in_GuQuanquan} and~\ref{lemma:lemma_B.3_in_GuQuanquan} hold. It is easy to see that as long as $m \geq R^2 (3\mathrm{Lip}_{\max}+1)^{2L-2}$, we have $\bm{W}^* \in \mathcal{B} (\bm{W}^{(1)}, \omega)$. 
We now show that under our parameter choice, 
$\bm{W}^{(1)},\ldots,\bm{W}^{(N)} $ are inside $ \mathcal{B} (\bm{W}^{(1)}, \omega)$ as well. 

This result follows by simple induction. Clearly we have $\bm{W}^{(1)} \in \mathcal{B} (\bm{W}^{(1)}, \omega)$. Suppose that $ \bm{W}^{(1)},\ldots, \bm{W}^{(i)} \in \mathcal{B} (\bm{W}^{(1)}, \omega) $. Then by~\cref{lemma:lemma_B.3_in_GuQuanquan}, for $l\in [L]$, we have $\|\nabla_{\bm{W}_l} L_i(\bm{W}^{(i)})\|_2 \leq \Theta(3\mathrm{Lip}_{\max}+1)^{L-l-1}$.

Therefore:
\begin{align*}
    \big\| \bm{W}_l^{(i+1)} - \bm{W}_l^{(1)} \big\|_2 \leq \sum_{j  = 1}^i\big\| \bm{W}_l^{(j+1)} - \bm{W}_l^{(j)} \big\|_2 \leq \Theta((3\mathrm{Lip}_{\max}+1)^{L-l-1}\gamma N)\,.
\end{align*}
Plugging in our parameter choice 
$\gamma = \varepsilon/[m(3\mathrm{Lip}_{\max}+1)^{2L-2}]$, $N = LR^2(3\mathrm{Lip}_{\max}+1)^{2L-2}/(2\varepsilon^2)$
for some small enough absolute constant $\nu$ provides:
\begin{align*}
\big\| \bm{W}_l^{(i+1)} - \bm{W}_l^{(1)} \big\|_{\mathrm{F}} \leq\Theta\bigg(\sqrt{m}(3\mathrm{Lip}_{\max}+1)^{L-l-1}\frac{LR^2}{2m\varepsilon}\bigg) \leq \omega\,,
\end{align*}
where the last inequality holds as long as $m \geq (3\mathrm{Lip}_{\max}+1)^{4L-4}L^2R^4/(4\varepsilon^2)$. Therefore by induction we see that $\bm{W}^{(1)},\ldots,\bm{W}^{(N)} \in \mathcal{B} (\bm{W}^{(1)}, \omega)$. As a result, the conditions of Lemmas.~\ref{lemma:bound_for_perturbation},~\ref{lemma:lemma_4.1_in_GuQuanquan},~\ref{lemma:lemma_4.2_in_GuQuanquan} and~\ref{lemma:lemma_B.3_in_GuQuanquan} are satisfied for $\bm{W}^*$ and $\bm{W}^{(1)},\ldots, \bm{W}^{(N)}$.

In the following, we utilize the results of Lemmas~\ref{lemma:bound_for_perturbation},~\ref{lemma:lemma_4.1_in_GuQuanquan},~\ref{lemma:lemma_4.2_in_GuQuanquan} and~\ref{lemma:lemma_B.3_in_GuQuanquan} to prove the bound of cumulative loss. First of all, by~\cref{lemma:lemma_4.2_in_GuQuanquan}, we have:
\begin{align*}
    L_i(\bm{W}^{(i)}) - L_i(\bm{W}^{*}) &\leq  \left \langle  \nabla_{\bm{W}} L_i(\bm{W}^{(i)}), \bm{W}^{(i)} - \bm{W}^*  \right \rangle  + \epsilon \\
    &= \sum_{l=1}^L \frac{\left \langle  \bm{W}_l^{(i)} - \bm{W}_l^{(i+1)}, \bm{W}_l^{(i)} - \bm{W}_l^*  \right \rangle  }{ \gamma } + \epsilon\,.
\end{align*}
Note that for the matrix inner product we have the equality $2\left \langle \bm{A},\bm{B} \right \rangle = \left \| \bm{A} \right \| _{\mathrm{F}}^2 + \left \| \bm{B} \right \| _{\mathrm{F}}^2 - \left \| \bm{A-B} \right \| _{\mathrm{F}}^2$. Applying this equality to the right hand side above provides:
\begin{align*}
    L_i(\bm{W}^{(i)}) - L_i(\bm{W}^{*}) \leq \sum_{l=1}^L \frac{ \| \bm{W}_l^{(i)} - \bm{W}_l^{(i+1)} \|_{\mathrm{F}}^2 + \| \bm{W}_l^{(i)} - \bm{W}_l^* \|_{\mathrm{F}}^2 - \| \bm{W}_l^{(i+1)} - \bm{W}_l^* \|_{\mathrm{F}}^2 } {2\gamma} + \epsilon\,.
\end{align*}
By~\cref{lemma:lemma_B.3_in_GuQuanquan}, for $l\in [L]$ we have $\| \bm{W}_l^{(i)} - \bm{W}_l^{(i+1)} \|_{\mathrm{F}} \leq \gamma\sqrt{m}\|\nabla_{\bm{W}_l} L_i(\bm{W}^{(i)})\|_2 \leq \Theta (\gamma\sqrt{m}(3\mathrm{Lip}_{\max}+1)^{L-l-1})$. 

Therefore:
\begin{align*}
\small
    L_i(\bm{W}^{(i)}) - L_i(\bm{W}^{*}) \leq \sum_{l=1}^L \frac{ \| \bm{W}_l^{(i)} - \bm{W}_l^* \|_{\mathrm{F}}^2 - \| \bm{W}_l^{(i+1)} - \bm{W}_l^* \|_{\mathrm{F}}^2 } {2\gamma} + \Theta ((3\mathrm{Lip}_{\max}+1)^{2L-2}\gamma m) + \epsilon\,.
\end{align*}
Telescoping over $i = 1,\ldots, N$, we obtain:
\begin{align*}
    \frac{1}{N}\sum_{i=1}^N L_i(\bm{W}^{(i)}) &\leq \frac{1}{N}\sum_{i=1}^N L_i(\bm{W}^{*}) + \sum_{l=1}^L \frac{ \| \bm{W}_l^{(1)} - \bm{W}_l^* \|_{\mathrm{F}}^2 } {2N\gamma} + \Theta ((3\mathrm{Lip}_{\max}+1)^{2L-2}\gamma m) + \epsilon\\
    &\leq \frac{1}{N}\sum_{i=1}^N L_i(\bm{W}^{*}) + \frac{ L R^{2} } {2\gamma mN} + \Theta ((3\mathrm{Lip}_{\max}+1)^{2L-2}\gamma m) + \epsilon\,,
\end{align*}
where in the first inequality we simply remove the term $-\| \bm{W}_l^{(N+1)} - \bm{W}_l^* \|_{\mathrm{F}}^2/(2\gamma) $ to obtain an upper bound, the second inequality follows by the assumption that $\bm{W}^*\in \mathcal{B} (\bm{W}^{(1)},Rm^{-1/2})$. 
Plugging in the parameter choice 
$\gamma = \varepsilon/[m(3\mathrm{Lip}_{\max}+1)^{2L-2}]$, $N = LR^2(3\mathrm{Lip}_{\max}+1)^{2L-2}/(2\varepsilon^2)$, then:
\begin{align*}
    \frac{1}{N}\sum_{i=1}^N L_i(\bm{W}^{(i)}) &\leq \frac{1}{N}\sum_{i=1}^N L_i(\bm{W}^{*}) + 3\epsilon\,,
\end{align*}
which finishes the proof.
\end{proof}

\subsection{Proof of~\cref{thm:NTK_Generalization}}
\label{ssec:Proof_of_NTK_Generalization}

\begin{proof}

By Lemmas~\ref{lemma:lemma_4.1_in_GuQuanquan},~\cref{lemma:lemma_4.3_in_GuQuanquan} and Theorem 3.3, Lemma 4.4, Corollary 3.10 in ~\citep{cao2019generalization}, let $C_1(L) = \sqrt{L}/(3\mathrm{Lip}_{\max}+1)^{L-1}$ and $C_2(L) = \sqrt{L}(3\mathrm{Lip}_{\max}+1)^{L-1}$, bring in our $\gamma$ and $N$ with a not very large $L$, we have:

\begin{equation*}
\mathbb{E}[\ell_{\mathcal{D} }^{0-\!1}\!(\hat{\bm{W}}\!)] \!\leq\! \tilde{\mathcal{O} }\! \left(\! C_2\sqrt{\frac{\bm{y}^{\top}  ({\bm{K}^{(L)}})^{-1} \bm{y}}{N}} \!\right) + \mathcal{O}\!\left( \! \sqrt{\frac{\log(1/\delta )}{N} } \! \right)\!.
\end{equation*}

\end{proof}

\section{Discussion on the key points and the motivation of the NTK analysis}
\label{sec:discussion}

In this section, we discuss the motivation and few key points in the proof of this paper and we also explain how the proof differs from previous results.

\textbf{The motivation for studying the minimum eigenvalue of NTK}:

To make this clearer, we provide an illustrative example on the significance of the minimum eigenvalue. Let us consider the square loss $\Phi(\theta) = \frac{1}{2}\sum_{i=1}^{n} \left \| f(x_i) - y_i \right \|^2$. A simple calculation shows that $\Phi(\theta) \leq \frac{[\nabla \Phi(\theta)]^2}{2\lambda_{\min}(K)}$. Thus if the minimum eigenvalue of NTK is strictly greater than 0, then minimizing the gradient on the LHS will drive the loss to zero. The larger the minimum eigenvalue, the smaller the loss.

Therefore, in this work, we are using the minimum eigenvalue to derive the generalization bound of NAS.

\textbf{Key points in the proof}:
\begin{itemize}
    \item \emph{Minimum eigenvalue}: Our proof framework is motivated by \citet{pmlr-v139-nguyen21g} on minimal eigenvalue of NTK of ReLU neural networks. However, our proofs differ from them in two aspects. Firstly, as we discussed in~\cref{sec:introduction}, extension to mixed activation functions is non-trivial due to the special properties of $\mathrm{ReLU}$. More importantly, we remark that the lower bound of the minimal eigenvalue of NTK in \citep[Theorem 3.2]{pmlr-v139-nguyen21g} holds with probability at least $1-N e^{-\Omega(d)} - N^2 e^{-\Omega\left(dN^{-\frac{2}{r-0.5}}\right)}$, where $r \geq 2$ is some constant.
    It can be found that, this concentration probability decreases as the number of training data increases. Thus, it could be negative for a large $N$.
    This is due to the use of Gershgorin circle theorem leading to a loose probability estimation. Instead, in this paper, we do not use this theorem, and we develop a tighter estimation based on~\citet{10.1214/ECP.v19-3807} under the assumption of isotropic data distribution. Accordingly, we achieve the reasonable $1- e^{-d}$ probability, \emph{c.f.} \cref{thm:lambda_min_inf_mixed}.
    \item \emph{Generalization}: Our proof framework is based on ~\citet{cao2019generalization} for generalization guarantees of deep ReLU neural networks requiring $m = \Omega(L^{56})$. Their results cannot be directly extended to other activation functions as the nice homogeneity and the derivative property of ReLU are used in their proof.
    To make our result feasible to various activation functions, we employ Lipschitz continuous properties of all activation functions, and achieve the generalization guarantees with $m = \Omega(4^{4L})$, \emph{c.f.} \cref{thm:NTK_Generalization} and~\cref{lemma:lemma_4.3_in_GuQuanquan}.
    Admittedly, our result is in an exponential increasing order of the depth. However, in practice, the depth of neural networks in NAS is usually smaller than $20$, or even $10$~\citep{liu2018hierarchical, dong2021nats}, which leads to $4^{4L} \ll L^{56}$ in this case when compared to their result.
    This result makes our theory reasonable and fair for NAS.
\end{itemize}

\section{Auxiliary numerical validations}
\label{sec:additional_experiments}

\subsection{Dataset details and algorithm}
\label{ssec:dataset}
We describe here the datasets that we have used for the numerical validation of our theory. Those are the following five datasets:

\begin{enumerate}
    \item \emph{Fashion-MNIST}~\citep{xiao2017fashion} includes grayscale images of clothing. The training set consists of $60,000$ examples and the test set of $10,000$ examples. The resolution of each image is $28\times 28$, with each image belonging to one of the $10$ classes. 

    \item \emph{MNIST}~\citep{726791} includes handwritten digits images. MNIST has a training set of $60,000$ examples and a test set of $10,000$ examples. The resolution of each image is $28\times 28$.

    \item \emph{CIFAR-10 and CIFAR-100~\citep{krizhevsky2014cifar}} depicts images of natural scenes. CIFAR-100 has a training set of $50,000$ examples and a test set of $10,000$ examples. The resolution of each RGB image is $32\times 32$.
    
    \item \emph{ImageNet-16}~\citep{chrabaszcz2017downsampled} is the down-sampled version of ImageNet~\citep{deng2009imagenet} with image size $16\times16$ on $120$ classes. 
\end{enumerate}

Our Eigen-NAS algorithm used in~\cref{ssec:NAS_201_experiment} is summarized as below.

\begin{algorithm}[H]
\caption{Eigen-NAS Algorithm}
\label{alg:algorithm_NAS}
\begin{algorithmic}
\REQUIRE Search space $\mathcal{S}$, training data $\mathcal{D}_{tr} = \{ (\bm x_i, y_i)_{i=1}^N \}$, validation data $\mathcal{D}_{val} = \{ (\bm x_j, y_j)_{j=1}^{N_v} \}$.\\
Initialize max\_iteration $ = M$\\
Initialize candidate set $\mathcal{C} = []$\\
\FOR{search\_iteration in $1, 2, \dots, $max\_iteration}
\STATE{Randomly sample architecture $s$ from search space $\mathcal{S}$.\\Compute $Eigen := \text{minimum eigenvalue of NTK}$.\\$\mathcal{C}.\text{append}(s, Eigen)$\\update $\mathcal{C}$ to kept top-K best architectures}
\ENDFOR \\
$s^{\star} = best_{s}(\mathcal{C}, \mathcal{D}_{tr}, \mathcal{D}_{val})$ \# Choose the best architecture
based on validation error after training 20 epochs.\\
\textbf{Output} $s^{\star}$
\end{algorithmic}
\end{algorithm}

\subsection{Compared algorithms}

We provide a thorough comparison with the following baselines: 

\begin{enumerate}[leftmargin=3.3mm]
    \itemsep-0.2em 
    \item \textit{Classical network:} ResNet ~\citep{7780459}, which is the default baseline used widely in image-related tasks. 
    \item \textit{Reinforcement learning based algorithm:} NAS-RL~\citep{45826} with the validation accuracy as a reward, which is an classical and representative NAS Algorithm.
    \item \textit{Differentiable algorithm:} DARTS~\citep{liu2019darts}\footnote{We directly use the results from ~\citet{pmlr-v139-xu21m}.}, which is the earliest and basic gradient-based NAS algorithm.
    \item \textit{Train-free algorithms using metrics to guide NAS:} A new type of NAS algorithm, they use some special metrics to pick models directly from candidate models. Common Train-free algorithms are: NASWOT~\citep{mellor2021neural} using the output of ReLU; TE-NAS~\citep{chen2021neural} leveraging the spectrum of NTK and linear partition of the input space; KNAS~\citep{pmlr-v139-xu21m} employing the Frobenius norm of NTK. Our Eigen-NAS algorithm also belongs to this type.
\end{enumerate}

\subsection{Training/test accuracy of DNNs by NAS}

Here we evaluate the classification results with 5 runs of the obtained architecture by DARTS under varying widths $m\in \{64, 128, 256, 512, 1024\}$ and depths $L \in \{5, 10\}$ on Fashion-MNIST.
\cref{fig:DARTS_experiment} shows that nearly $90\%$ accuracy is achieved on the test set under different depth and width settings.
The result is competitive on FC/residual networks within $10$ layers and without training tricks, e.g., data augmentation, batch norm and drop out. 
We find that when compared to the depth, the network width also contributes on test accuracy. As suggested by~\cref{eq:problem_setting_parameter_space}, the amount of parameters in the neural network is approximately proportional to the depth, but squared to the width. 

\begin{figure}[t]
\centering
    \includegraphics[width=0.9\linewidth]{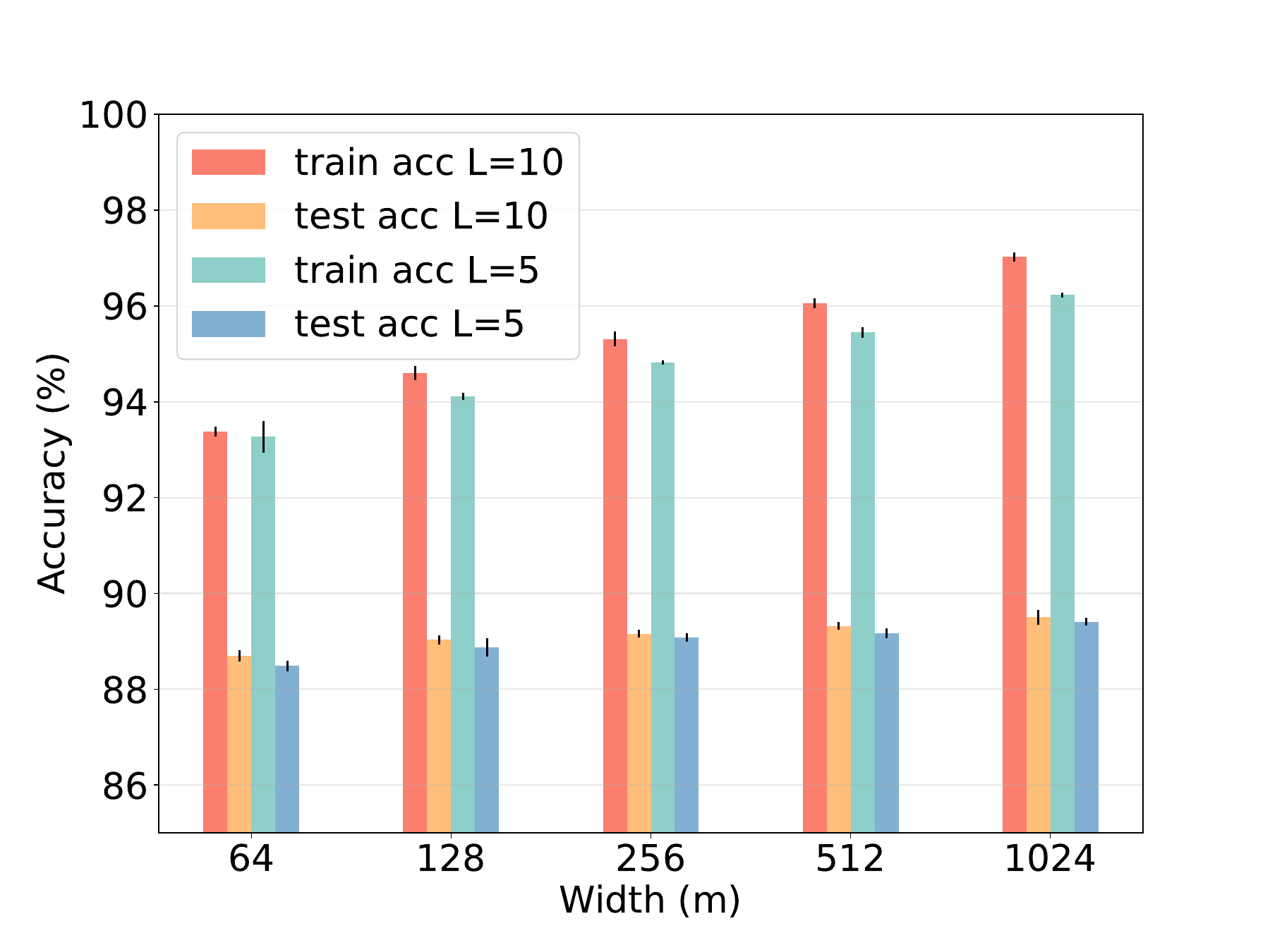}\vspace{-5mm}
\caption{The accuracy of neural networks by NAS under different widths and depths.}
\label{fig:DARTS_experiment}
\end{figure}

\subsection{Simulation of minimum eigenvalues of NTK}
\label{ssec:Simulation_of_NTK}

We calculate the minimum eigenvalue of each NTK matrix under different architectures with activation functions, skip connections and depths, according to~\cref{lemma:NTK_matrix_recursive_form}. 
We consider four special cases on skip connections: a) no skip connections with $\bm \alpha = \bm 0$, i.e., fully connected neural network in~\cref{fig:NTK_simulationa}; b) skip connections between all consecutive layers with $\bm \alpha = \bm 1$,  in~\cref{fig:NTK_simulationb}; c) the alpha of the first half (of the network) is $1$ and the alpha of the second half is $0$ shown in~\cref{fig:NTK_simulationc}; d) The alpha of the first half is $0$ and the alpha of the second half is $1$. The results are shown in~\cref{fig:NTK_simulationd}.

\cref{fig:NTK_simulationa} indicates that as the network depth increases, the minimum eigenvalue of NTK will become larger when LeakyReLU, ReLU, Swish and Tanh employed, but Sigmoid leads to a decreasing minimum eigenvalue, which is consistent with the upper bound shown in~\cref{thm:lambda_min_inf_mixed}. 
The LeakyReLU, ReLU and Swish generate the fastest increasing rate of depth, while Tanh and Sigmoid are slow, which coincides with the derived lower bound in~\cref{thm:lambda_min_inf_mixed} and previous work \cite{bachspaper}. \cref{fig:NTK_simulationb} shows that, under the skip connection, the tendency of the minimum eigenvalue of NTK is similar to that of FC neural networks when various activation functions are employed. However, the specific values and the growth rate are significantly larger than FC neural networks.
This result is consistent with the conclusion we state in~\cref{thm:lambda_min_inf_mixed} about skip layers leading to the increase of minimum eigenvalue of NTK with respect to the depth. Moreover,~\cref{fig:NTK_simulation_2} show similar growth speed.

Then, we plot the comparison figure of NTK under above two settings and two settings in main paper for the same activation function in~\cref{fig:DARTS_experiment_}. In addition to reconfirming the order between different activation functions, we can also see that the effect of adding an activation layer in the second half of the $10$-layer neural network is better than the first half of the neural network. This verifies the experimental results in~\cref{fig:DARTS_heatmap2}.

\begin{figure}[t]
\centering
    \subfigure[no skip connections]{\includegraphics[width=0.495\linewidth]{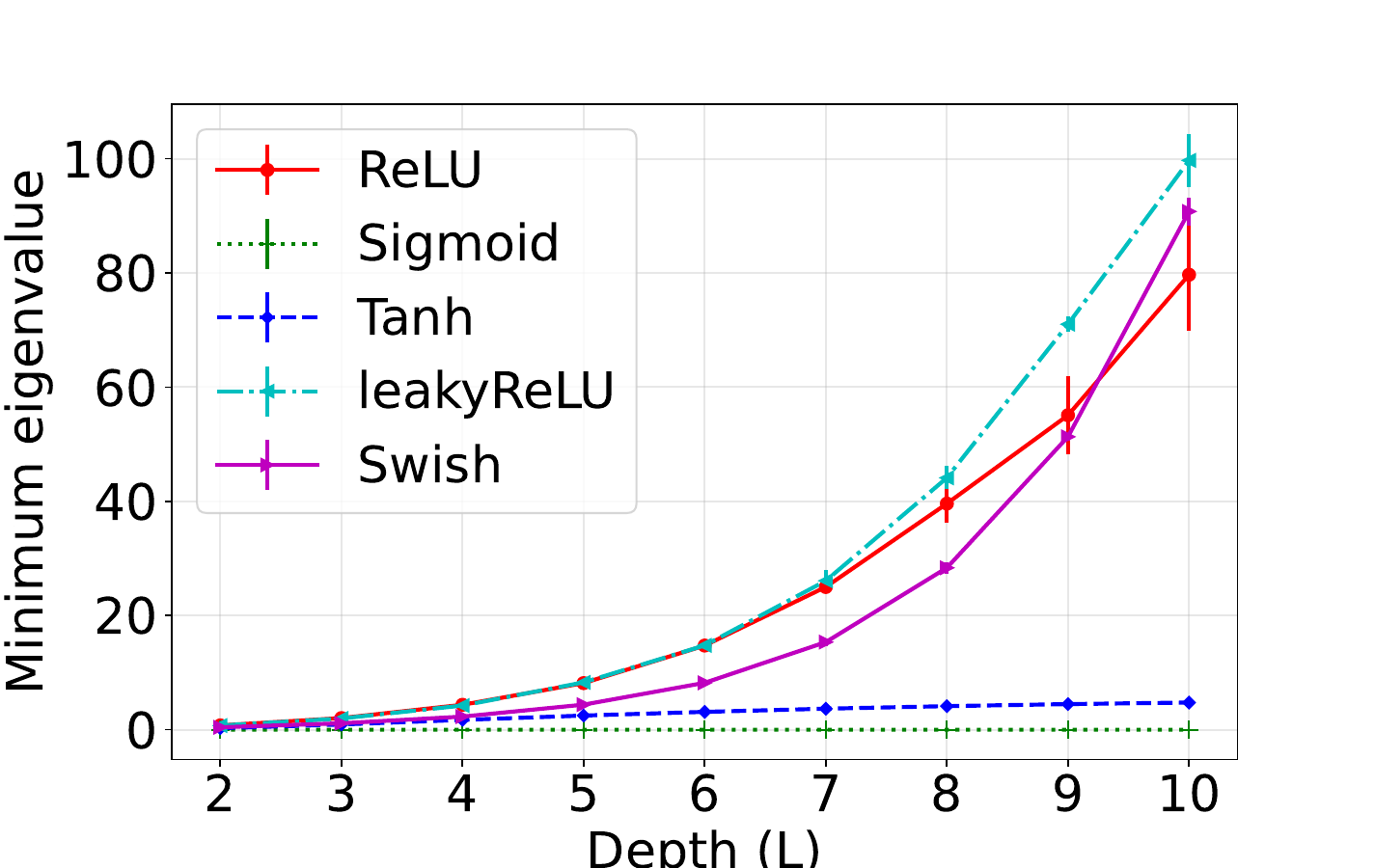}\label{fig:NTK_simulationa}\hspace{-2mm}}
    \subfigure[skip connections]{\includegraphics[width=0.495\linewidth]{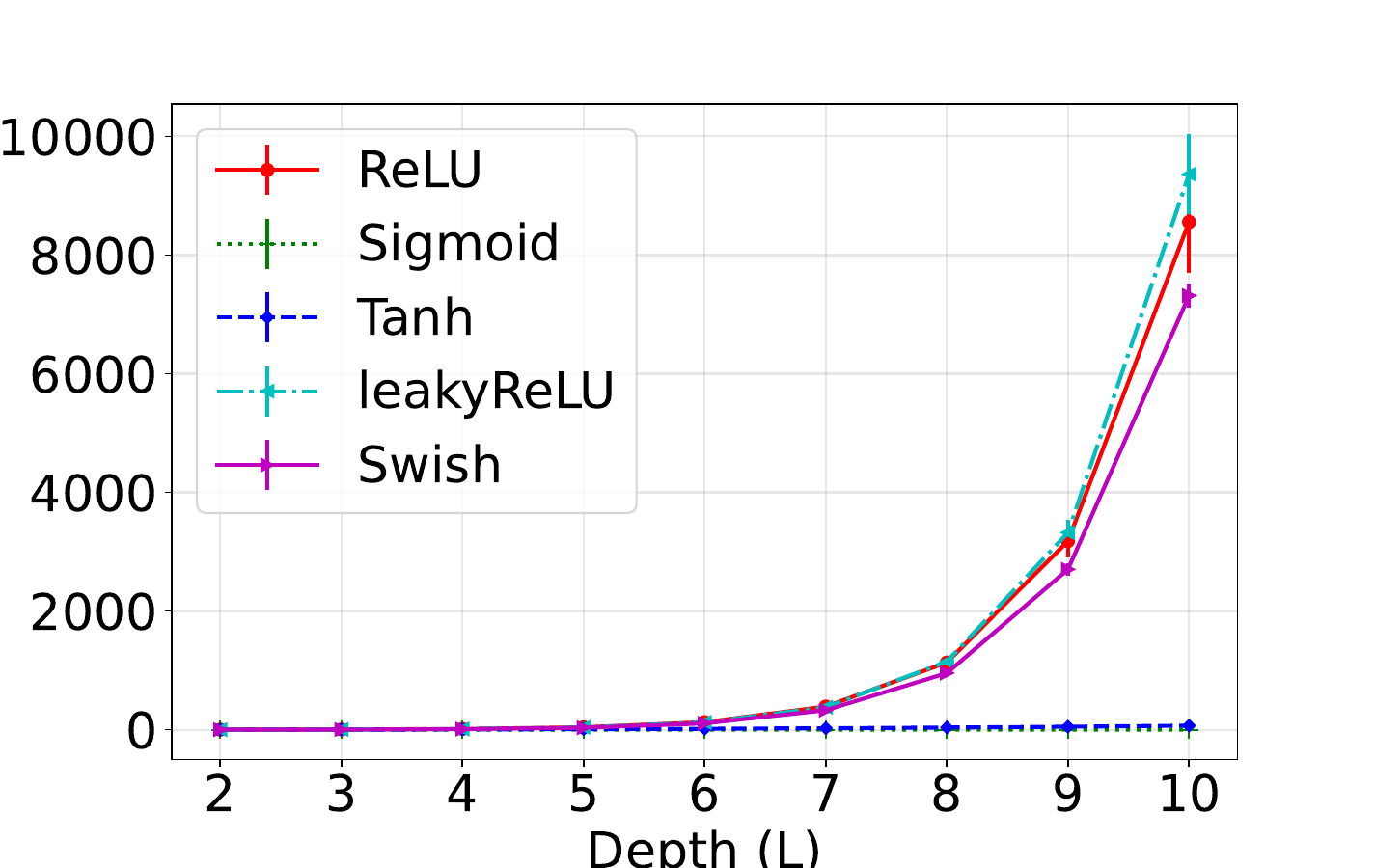}\label{fig:NTK_simulationb}}\vspace{-3mm}
\caption{Minimum eigenvalue of NTK \emph{vs.} depth ($L$) under various activation functions with/without skip connections in each layer. }
\label{fig:NTK_simulation}
\end{figure}

\begin{figure}[t]
\centering
    \subfigure[the first half has a skip layer]{\includegraphics[width=0.48\linewidth]{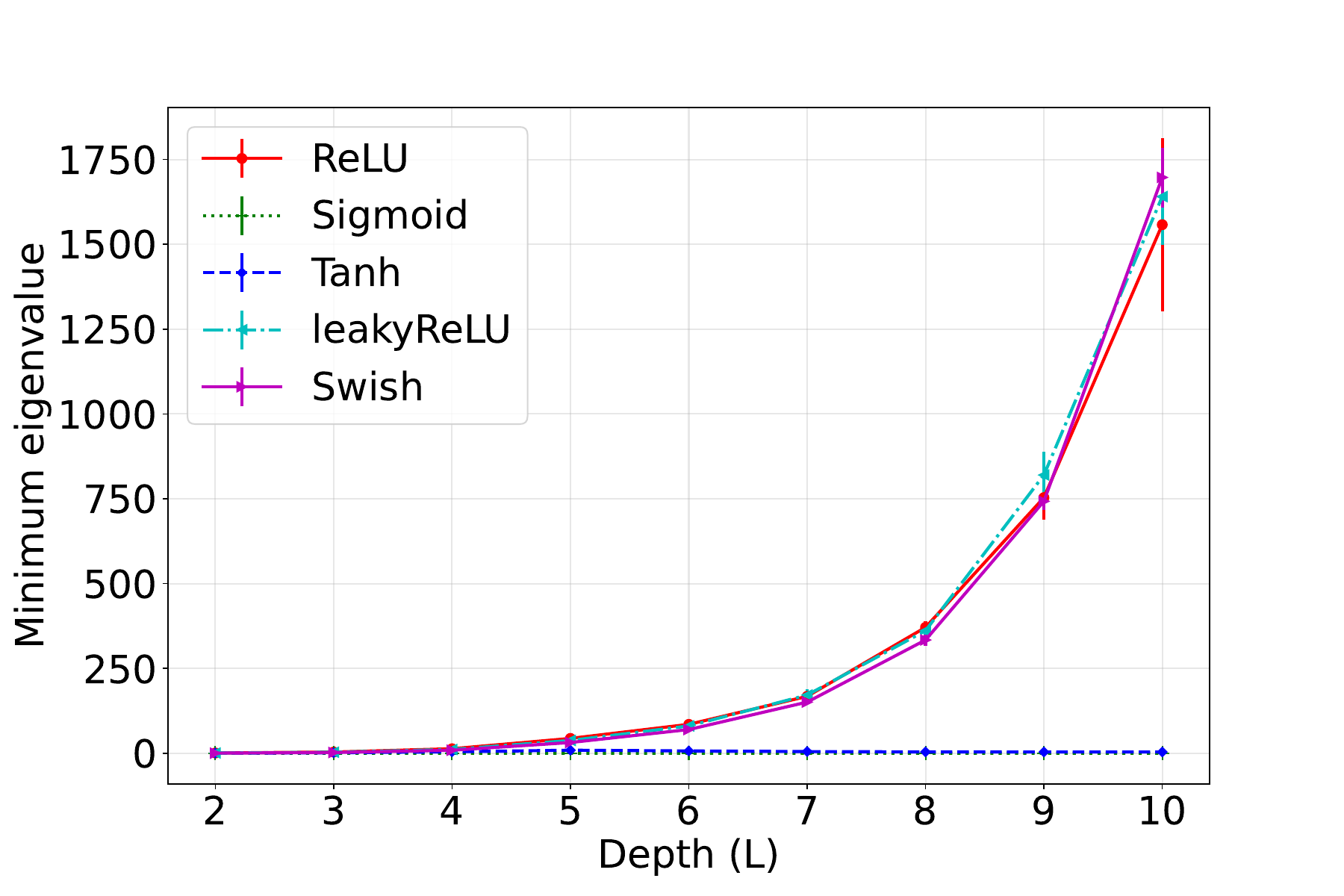}\label{fig:NTK_simulationc}\hspace{-2mm}}
    \subfigure[the second half has a skip layer]{\includegraphics[width=0.48\linewidth]{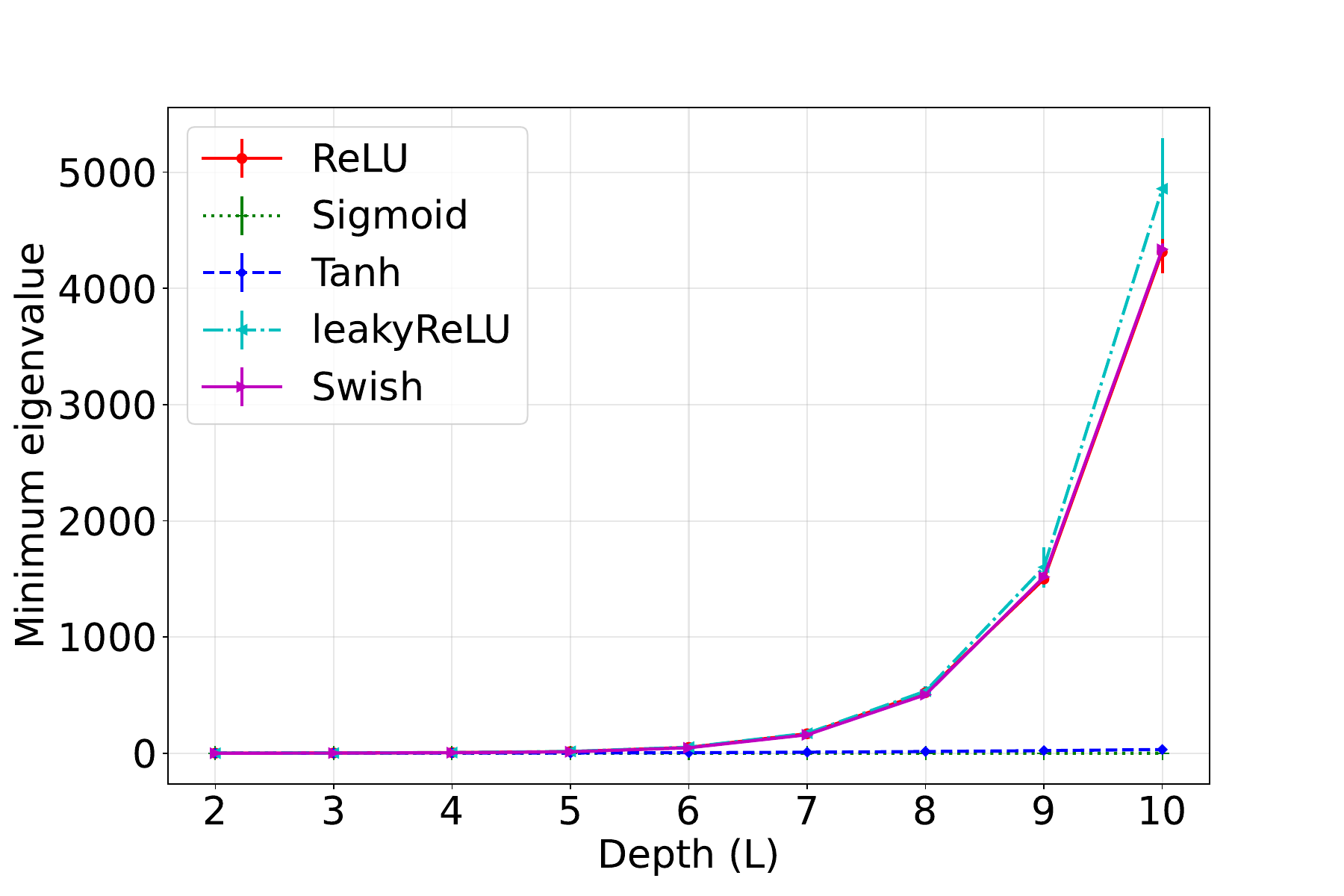}\label{fig:NTK_simulationd}}\vspace{-3mm}
\caption{Minimum eigenvalue of NTK \emph{vs.} depth ($L$) under various activation functions with/without skip connections in each layer. }
\label{fig:NTK_simulation_2}
\end{figure}

\begin{figure}[t]
\centering
    \subfigure[ReLU]{\includegraphics[width=0.495\linewidth]{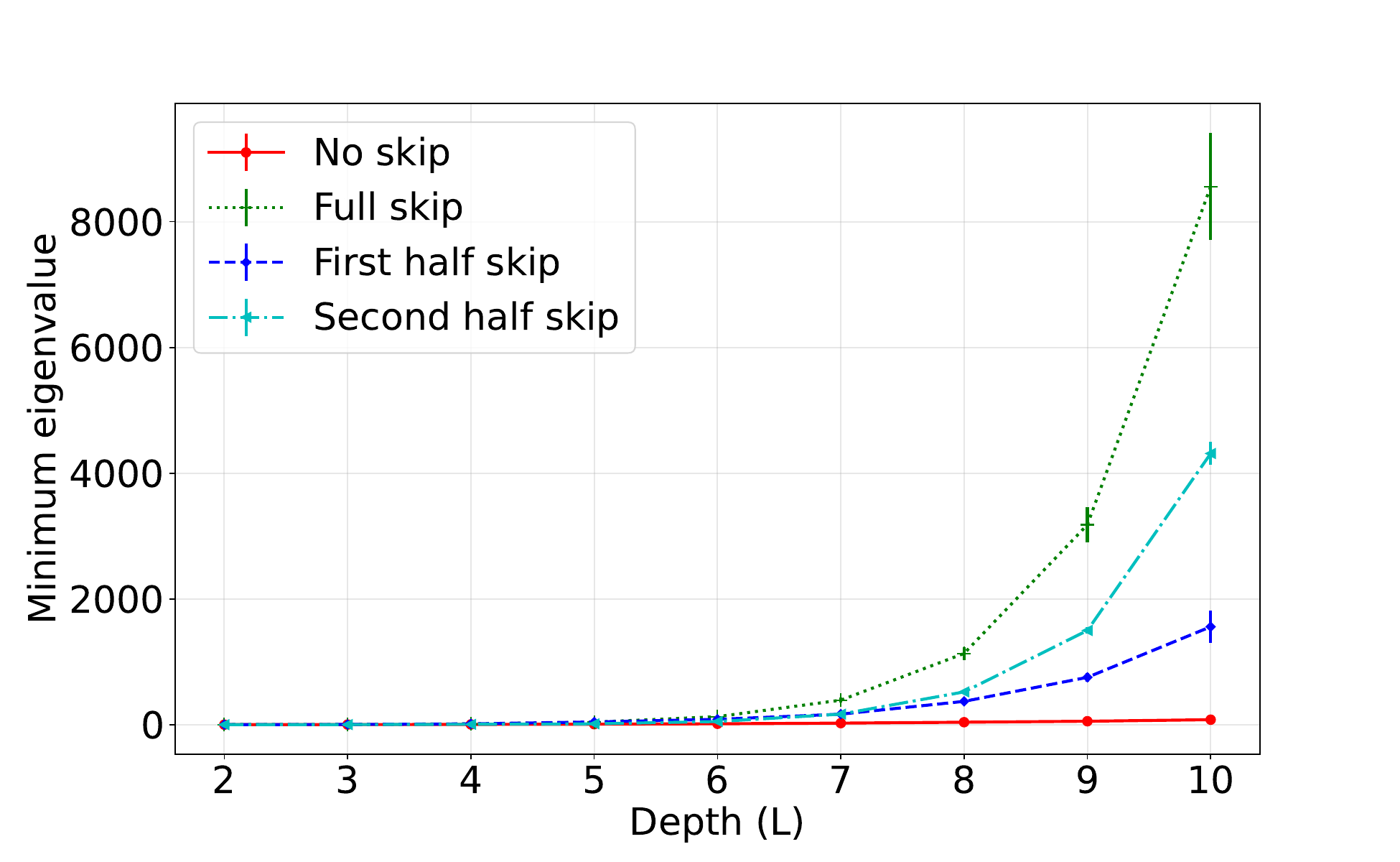}\label{fig:DARTS_experiment_ReLU}\hspace{-2mm}}
    \subfigure[LeakyReLU]{\includegraphics[width=0.495\linewidth]{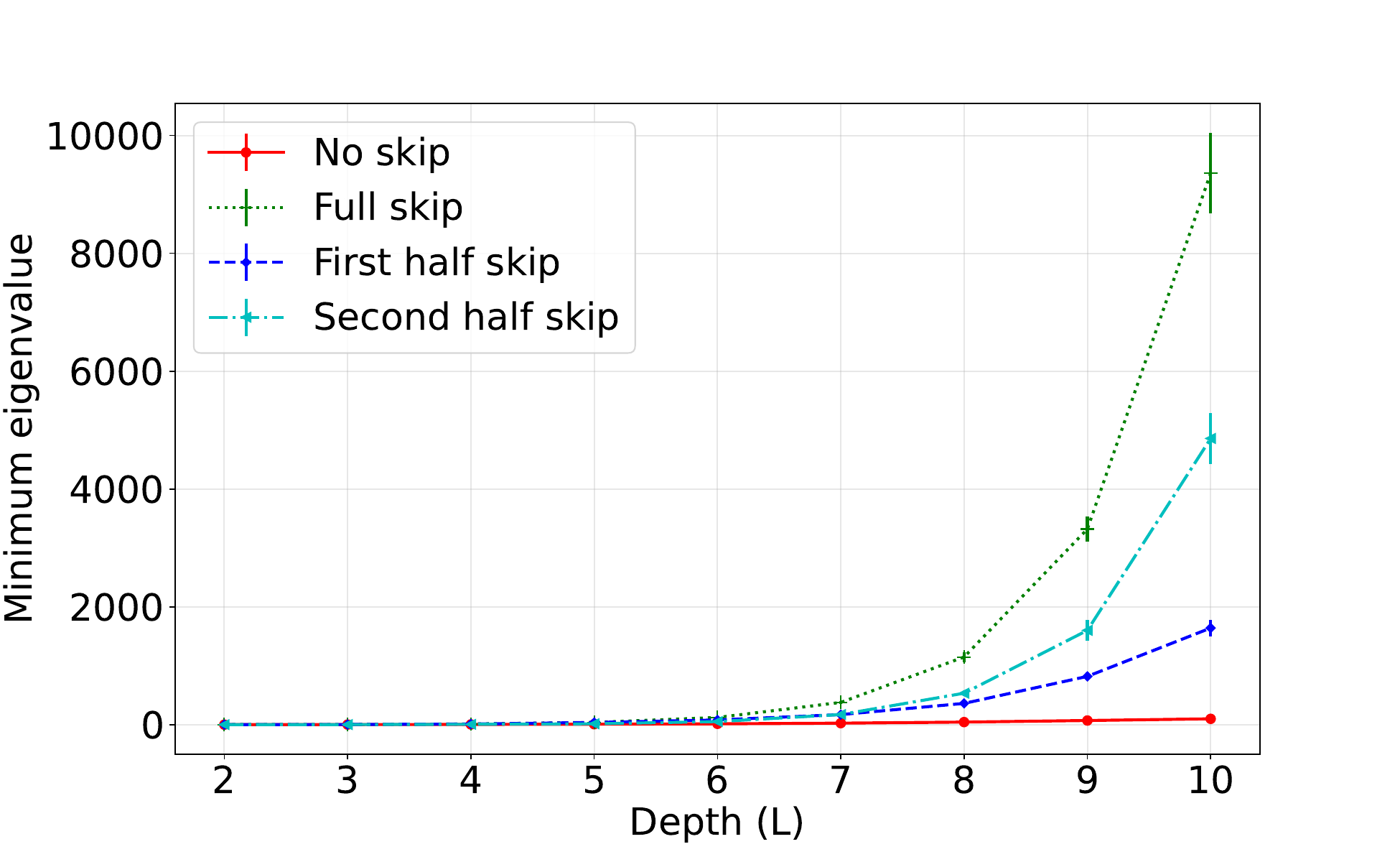}\label{fig:DARTS_experiment_LeakyReLU}}\vspace{-2mm}\\
    \subfigure[Sigmoid]{\includegraphics[width=0.33\linewidth]{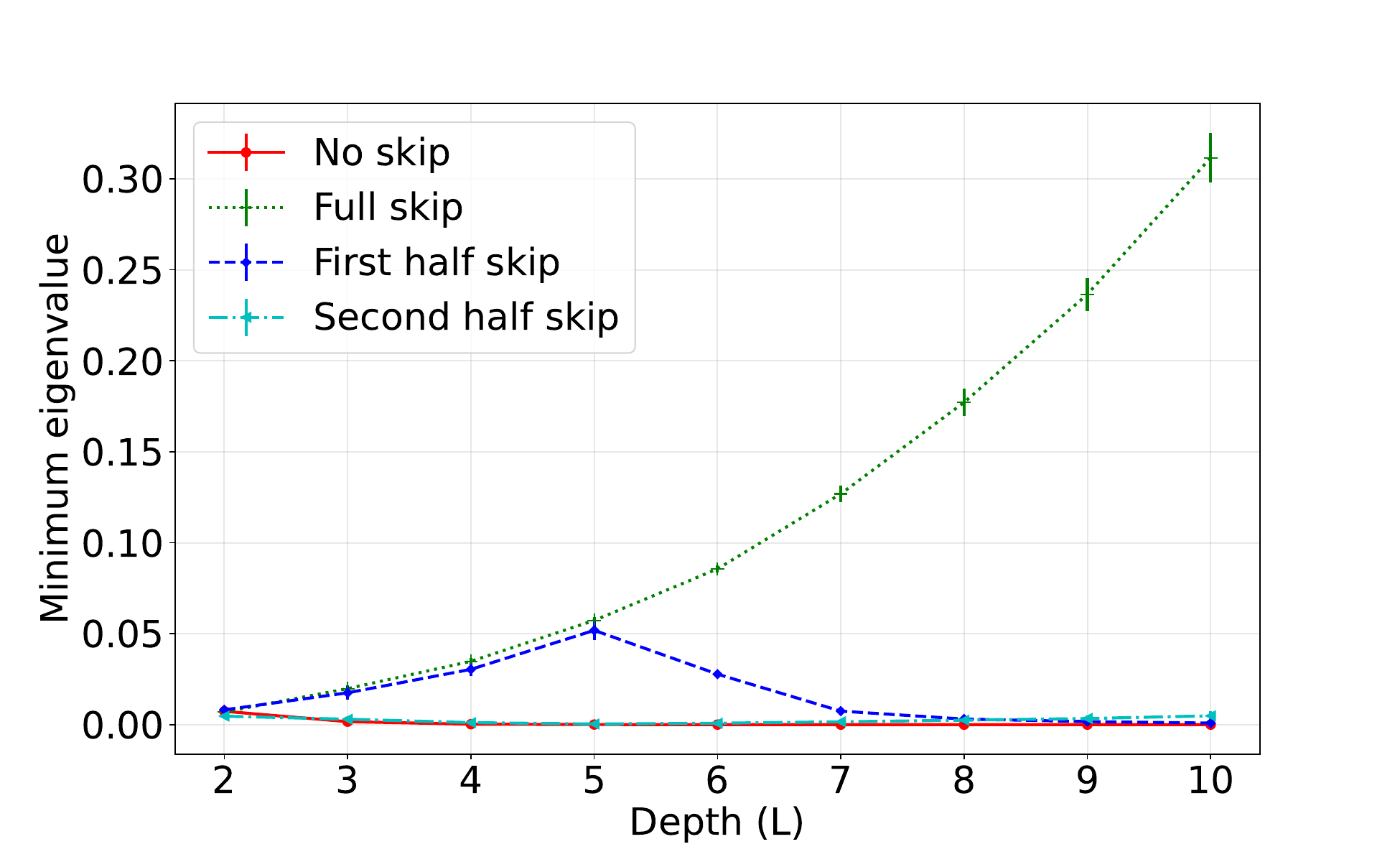}\label{fig:DARTS_experiment_Sigmoid}\hspace{-2mm}}
    \subfigure[Tanh]{\includegraphics[width=0.33\linewidth]{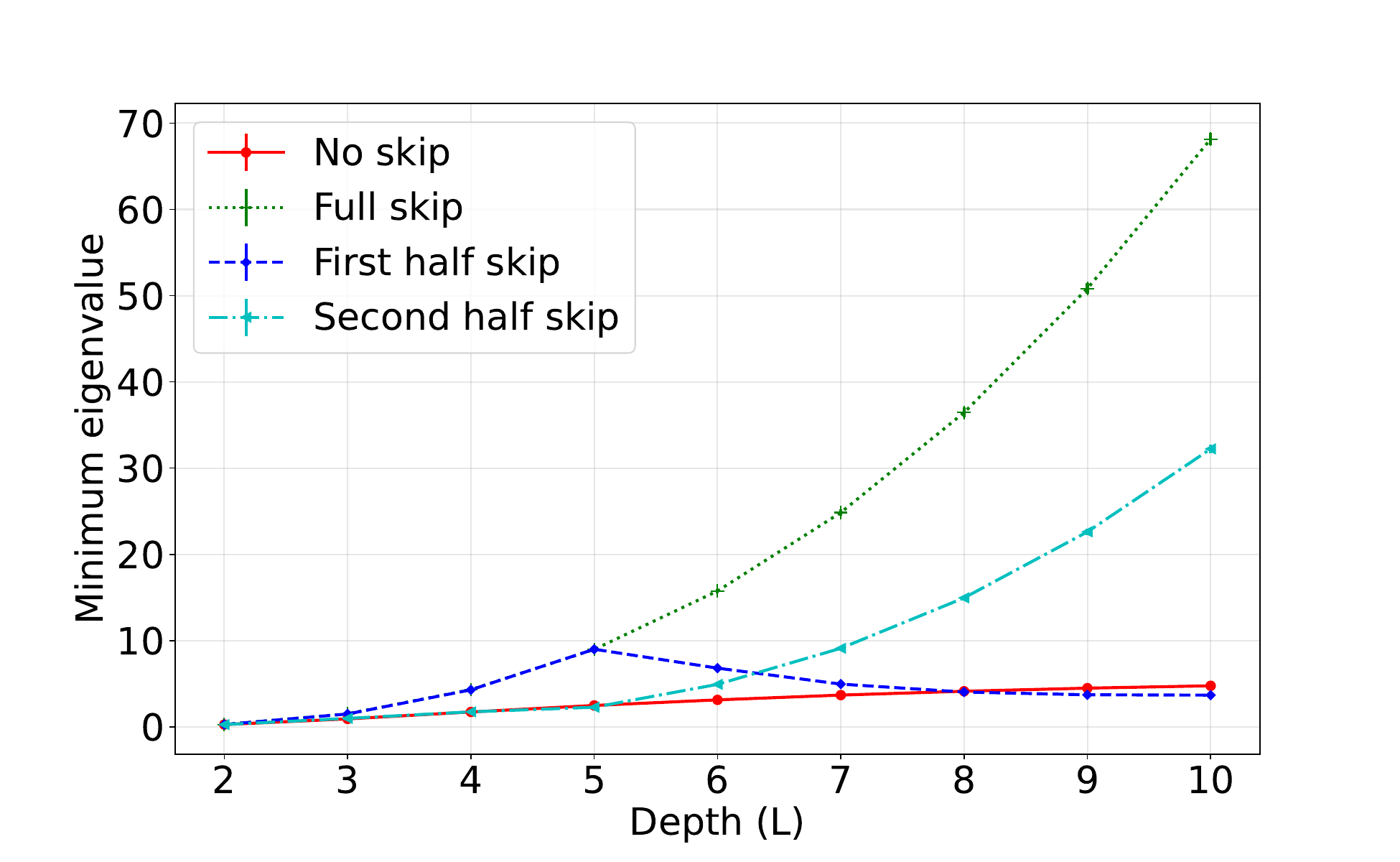}\label{fig:DARTS_experiment_Tanh}}\vspace{-2mm}
    \subfigure[Swish]{\includegraphics[width=0.33\linewidth]{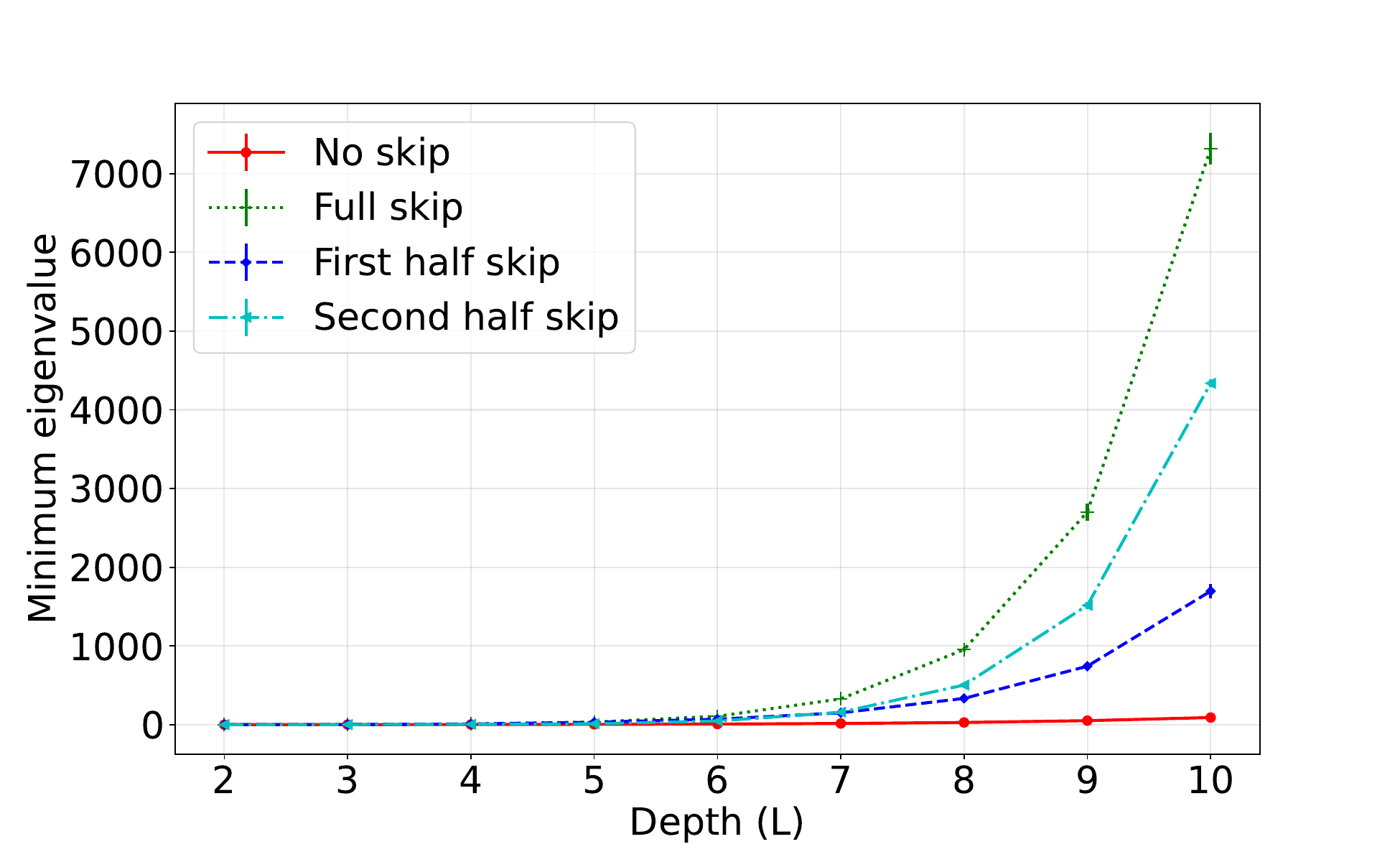}\label{fig:DARTS_experiment_Swish}}\vspace{-2mm}\\
\caption{Minimum eigenvalue of NTK for different activation function. The red line have skip connections in each layer, green line does not contain any skip connections, the blue line represents the skip connections in the first half and the cyan line represents the skip connections in the second half.}
\label{fig:DARTS_experiment_}
\end{figure}

\subsection{Additional experiments on NAS-Bench-101 and ranking correlations}
\label{ssec:bench101_experiment}

In this section, we conduct more experiments on two new benchmarks NAS-Bench-101~\citep{ying2019bench} and Network Design Spaces (NDS)~\citep{radosavovic2019network} using the same setting as~\cref{ssec:NAS_201_experiment}.

\cref{tab:NAS_benchmark_experiment_101} provides a comparison of the accuracy of Eigen-NAS, KNAS and NASWOT on four new search spaces. For all of four search spaces, our method achieves the best results with $1\%-2\%$ accuracy improvement.

\begin{table*}[tb]
\centering
\caption{New results on NAS-Benchmark-101, NDS-DARTS and NDS-PNAS using CIFAR-10 and ImageNette2, a subset of ImageNet.}
\begin{tabular}{l@{\hspace{0.25cm}} c@{\hspace{0.2cm}}c@{\hspace{0.2cm}}c@{\hspace{0.2cm}}c@{\hspace{0.2cm}} c} 
    \hline
    Benchmark  & NAS-Bench-101 & NDS-DARTS & NDS-PNAS & NDS-PNAS\\
    \hline
    Dataset  & CIFAR-10 & CIFAR-10 & CIFAR-10 & ImageNette2\\
    \hline
    \textbf{Eigen-NAS} ($k=20$)  & $\bm{92.7\%}$ & $\bm{92.6\%}$ & $\bm{93.8\%}$ &  $\bm{69.2\%}$\\
    KNAS ($k=20$)& $91.7\%$ & $90.1\%$ & $91.7\%$ &  $67.3\%$\\
    NASWOT & $91.3\%$ & $90.6\%$ & $93.3\%$ &  $68.4\%$\\
    \hline
\end{tabular}
\label{tab:NAS_benchmark_experiment_101}
\end{table*}

Moreover, we conduct more detailed experiments using the CIFAR-10 dataset on NAS-Bench-101.~\cref{tab:NAS_benchmark_experiment_time} provides the running time and Kendall rank correlation coefficient between minimum eigenvalues and accuracy for the above three train-free NAS algorithm. We can see that our Eigen-NAS method can get the best rank correlation coefficient with the fastest speed among three methods. The scatter plot of the relationship between the minimum eigenvalue and the accuracy is shown in~\cref{fig:ranking_correlation}.

\begin{table*}[tb]
\centering
\caption{Running time (in Second) and the Kendall rank correlation coefficient on NAS-Bench-101, CIFAR-10 (the larger the absolute value of Rank correlation, the stronger the correlation between the guide used by the algorithm and the network accuracy).}
\begin{tabular}{l@{\hspace{0.25cm}} c@{\hspace{0.2cm}}c@{\hspace{0.2cm}}c@{\hspace{0.2cm}}c@{\hspace{0.2cm}} c} 
    \hline
    Method  & Eigen-NAS ($k=20$) & KNAS ($k=20$) & NASWOT\\
    \hline
    Running time & $\bm{1136}$ & $1967$ &  $1468$\\
    Rank correlation & $ \bm{- 0.355}$ & $0.309$ &  $- 0.313$ \\
    \hline
\end{tabular}
\label{tab:NAS_benchmark_experiment_time}
\end{table*}

\begin{figure}[t]
\centering
    \includegraphics[width=0.6\linewidth]{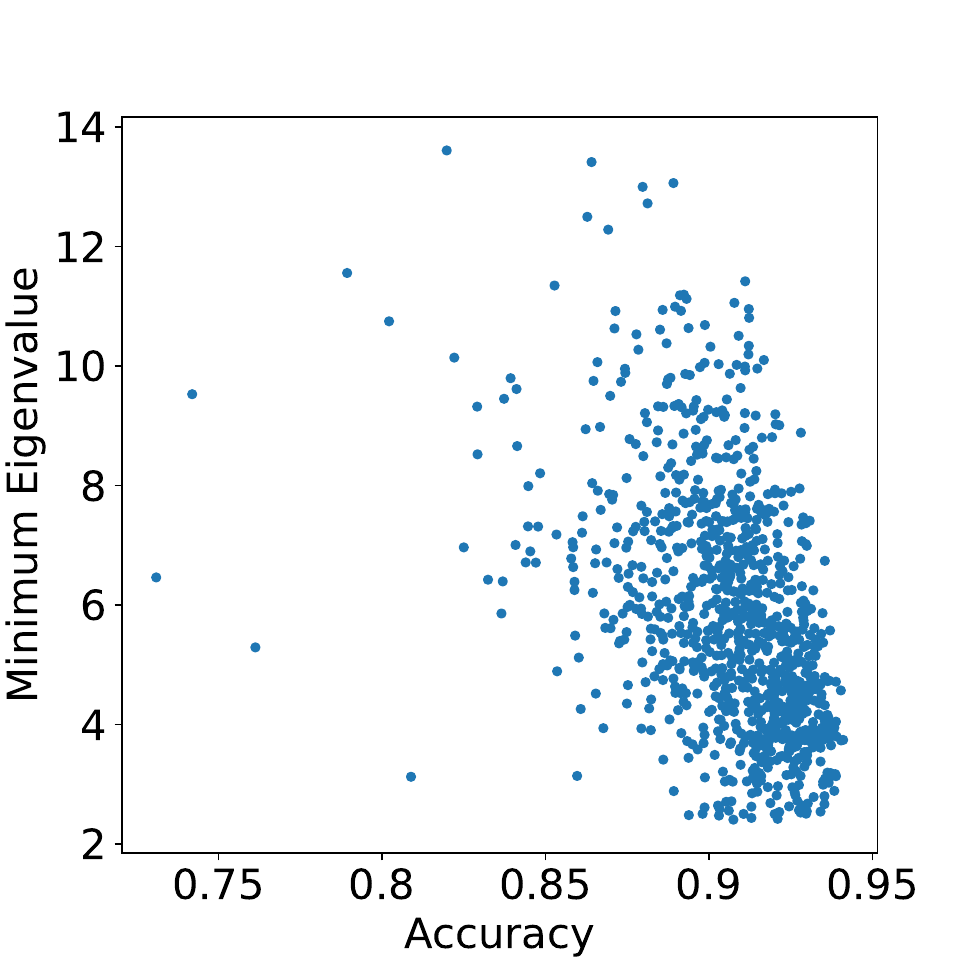}\vspace{-5mm}
\caption{The standard scatter plot on the kendall rank correlation coefficient.}
\label{fig:ranking_correlation}
\end{figure}

\subsection{Transfer learning experiment}
\label{ssec:transfer_learning_experiment}
Here we evaluate the proposed NAS framework on transfer learning. The algorithm from~\cref{ssec:DARTS_experiment} is employed for this experiment, e.g., the same search space and search strategy. 
The experiment setting is the following: we train the model on FashionMNIST for $20$ epochs, then we use the pretrained weights and fine-tune them for $5$ epochs on MNIST, with repeated three times.

\cref{tab:transfer_learning_experiment} show that, after the fine-tuning for just $2$ epochs, the method obtains up to $95\%$ accuracy and after fine-tuning for $5$ epochs it obtains up to $97\%$ accuracy. This verifies our intuition that the proposed NAS framework can obtain architectures that generalize well beyond the dataset they were optimized on.

\begin{table*}[tb]
\centering
\caption{Transfer learning result of our network for different width ($m$) which training in FashionMNIST (domain dataset) for $20$ epochs and then training in MNIST (target dataset) for $2$ or $5$ epochs. (the accuracy in the table are displayed in percentages)}
\begin{tabular}{c@{\hspace{0.2cm}}c@{\hspace{0.2cm}}c@{\hspace{0.2cm}} c@{\hspace{0.2cm}} c@{\hspace{0.2cm}} c} 
    \hline
    Epochs & $m=64$ & $m=128$ & $m=256$ & $m=512$ & $m=1024$\\
    \hline
    $20+2$ &  $94.13\pm0.64$ & $95.18\pm0.25$& $94.73\pm0.22$ & $94.40\pm0.80$& $\bm{95.41}\pm0.03$\\
    $20+5$ &  $95.73\pm0.28$ & $96.12\pm0.32$& $96.73\pm0.29$ & $96.73\pm0.11$& $\bm{96.96}\pm0.22$\\
    \hline
\end{tabular}
\label{tab:transfer_learning_experiment}
\end{table*}

\subsection{DARTS experiment on CNN}
\label{ssec:DARTS_experiment_CNN}

Our theory relies on fully-connected matrices and we have indeed verified experimentally the validity of our theoretical findings. To scrutinize our method even further, we attempt to extend our results to the popular convolutional neural networks. We believe this will provide some further insights on future extensions of our theory. 
In particular, we use DARTS (similarly with the experiment in \cref{ssec:DARTS_experiment}) with convolutional layers. The standard dataset of CIFAR-10 is selected; the details of the dataset are shared in \cref{ssec:dataset}. The search space and search strategy follow~\cref{ssec:DARTS_experiment} with one differentiating point: we use convolutional layers instead of fully connected layers in~\cref{eq:problem_setting_parameter_space}.

We select DARTS on a Convolutional Neural Network with $L=10$ and $m=1024$, while we repeat the experiment for 5 times. After training, the probability of these activation functions and skip connections in each layer are reported in Figure~\ref{fig:DARTS_heatmap_cnn_1} and~\ref{fig:DARTS_heatmap_cnn_2}, respectively. Compared with the~\cref{fig:DARTS_heatmap}, the activation function search exhibits similar characteristics with the results of the fully connected network. Namely: (1) ReLU and LeakyReLU have the highest probability to be selected, (2) the difference of probability between different activation functions in the first layer is the largest. But for skip layer search, CNN exhibits the opposite results with fully connected network, that is, almost all of the skip connections have a probability of being selected less than $50\%$. 

Based on the above results, our theory can still explain some of the phenomena observed in CNNs, e.g., activation functions search.
Nevertheless, our theory on skip connections search on CNNs mismatches with experimental demonstration in practice to some extent, which motivates us to conduct a refined analysis on CNNs for NAS.   

\begin{figure}[t]
\centering
    \subfigure[activation functions $\bm \sigma$]{\includegraphics[width=0.45\linewidth]{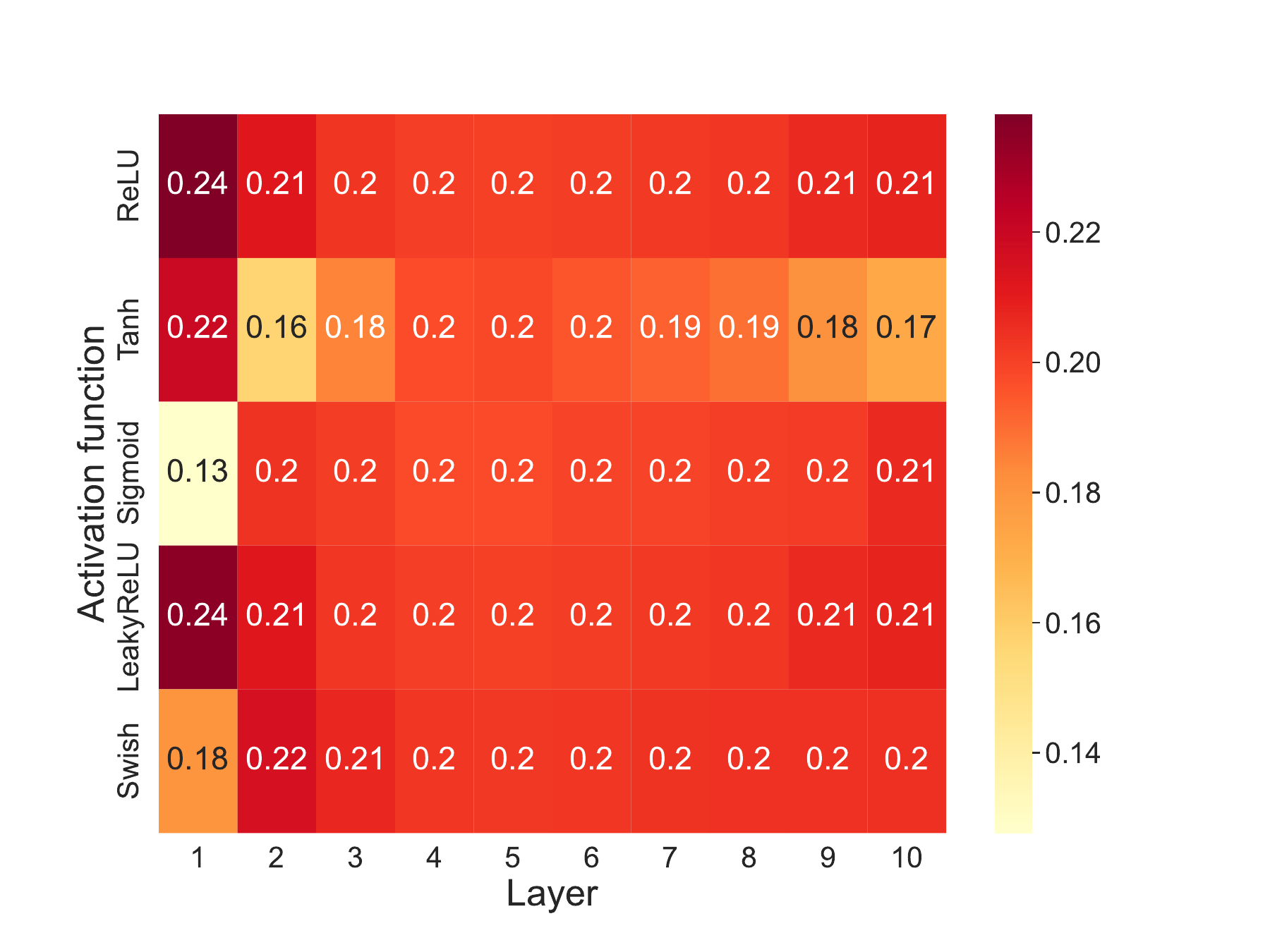}\label{fig:DARTS_heatmap_cnn_1}\hspace{-1mm}}
    \subfigure[skip connections $\bm \alpha$]{\includegraphics[width=0.45\linewidth]{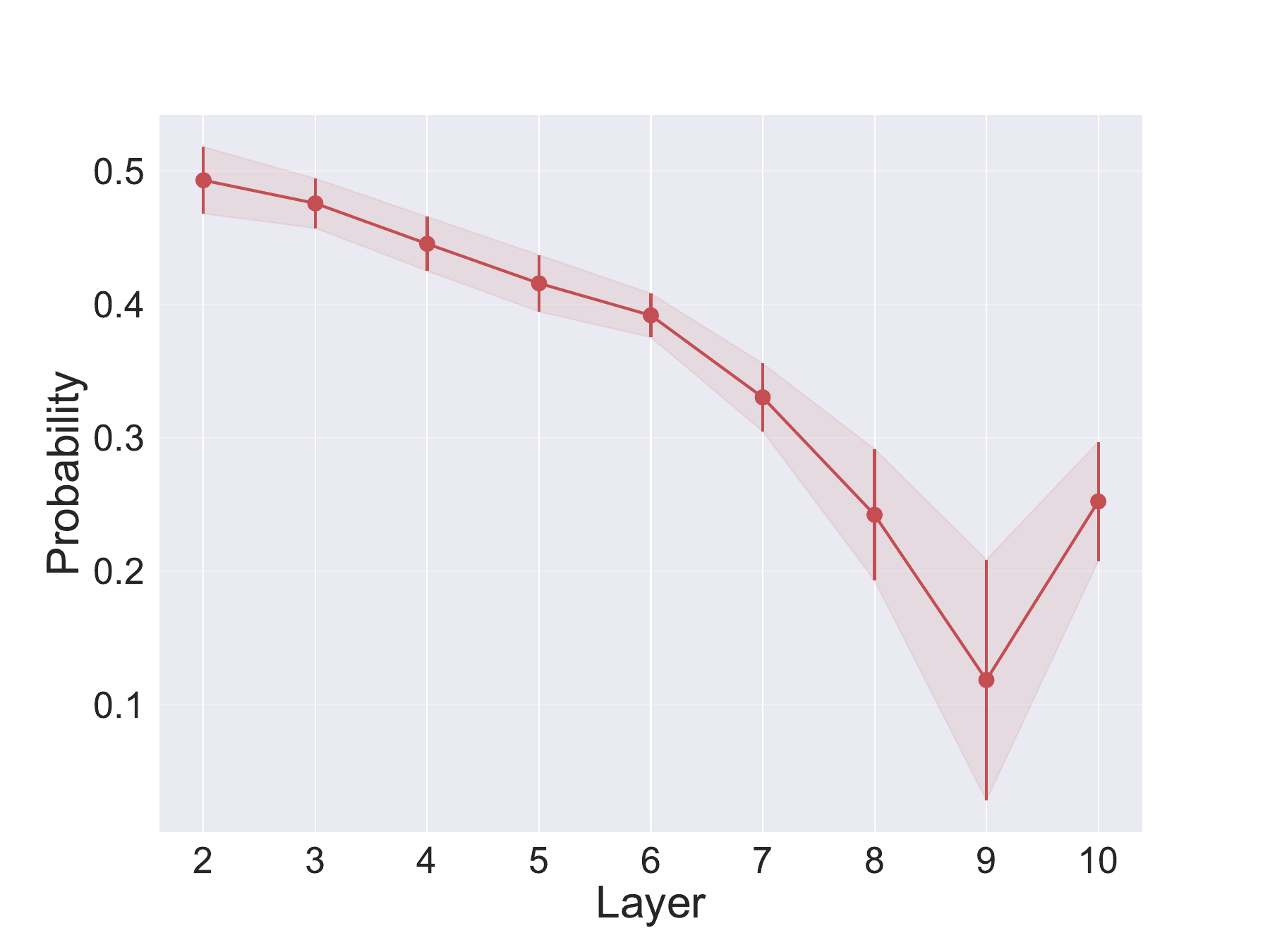}\vspace{-9mm}
    \label{fig:DARTS_heatmap_cnn_2}}
\caption{Architecture search results on activation functions indicated by the probability of $\bm \sigma$ in (a) and skip connections indicated by $\bm \alpha$ in (b). We notice that for each layer, ReLU and LeakyReLU are selected with the higher probability.}
\label{fig:DARTS_heatmap_cnn}
\end{figure}

\subsection{$\beta$-DARTS experiment on MLP}
\label{ssec:B_DARTS_experiment}

In this section, we use an improved DARTS-based algorithm, $\beta$-DARTS~\citep{https://doi.org/10.48550/arxiv.2203.01665}, for doing the activation function search. Our experiments are performed on a 5-layers MLP and the experimental results are presented in~\cref{fig:B_DARTS_experiment}. Compared with the results of DARTS in~\cref{fig:DARTS_heatmap}, the experimental results of $\beta$-DARTS indicate that the probability difference between different activation functions is smaller, which may verify that DART is more easily to overfit . This is also the advantage mentioned in the $\beta$-DARTS paper.

\begin{figure}[t]
\centering
    \includegraphics[width=0.6\linewidth]{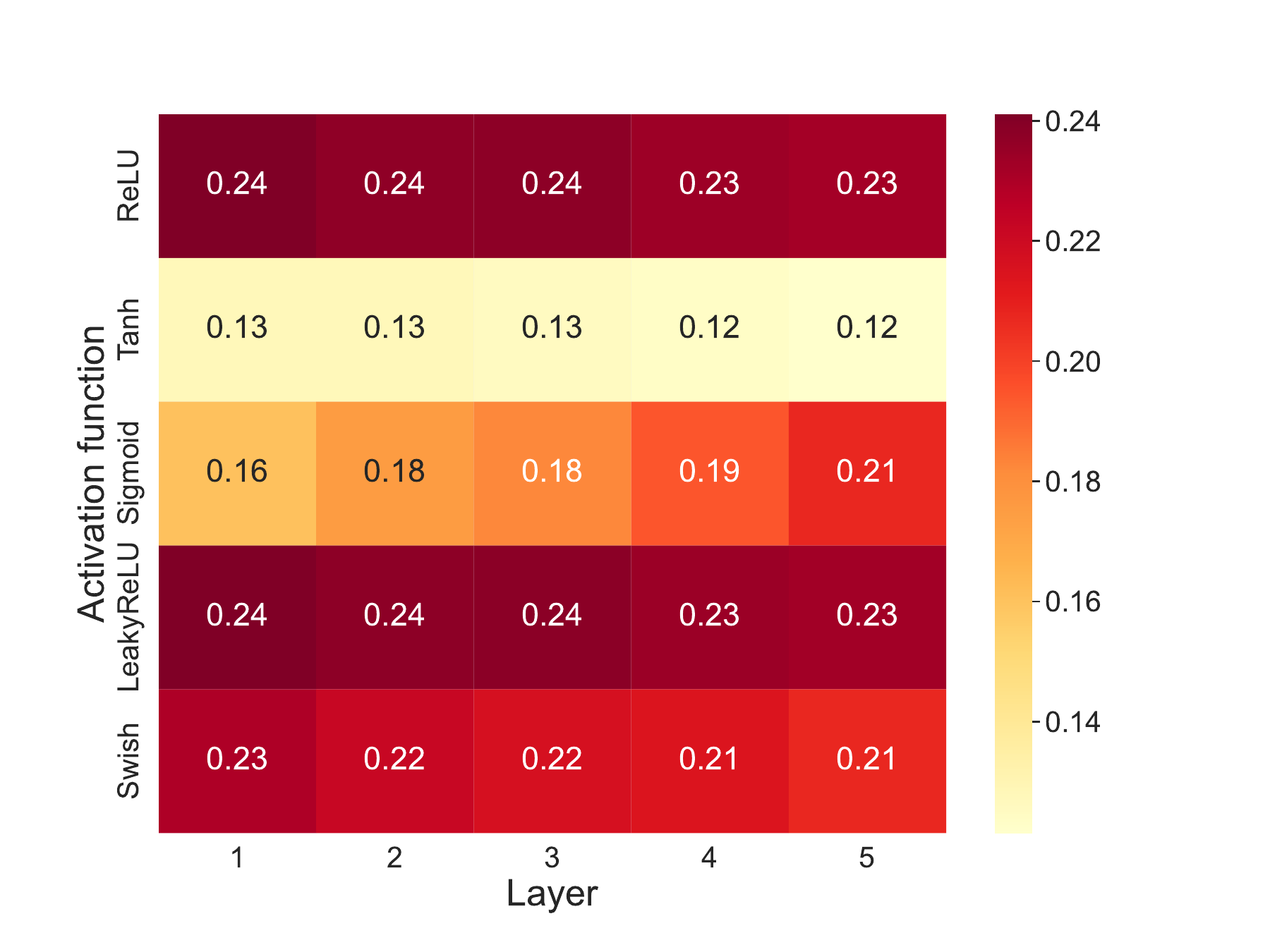}\vspace{-5mm}
\caption{Architecture search results using $\beta$-DARTS on activation functions indicated by the probability of $\bm \sigma$. We notice that for each layer, ReLU and LeakyReLU are selected with the higher probability.}
\label{fig:B_DARTS_experiment}
\end{figure}

\section{Societal impact}
\label{sec:nas_societal_impact}

This is a theoretical work that derived generalization bounds for the architectures obtained by NAS. As such, we do not expect our work to have negative societal bias, as we do not focus on obtaining state-of-the-art results in a particular task. On the contrary, our work can have various benefits for the community: 
\begin{itemize}
    \item We provide the first generalization bounds for the class of NAS architectures, which is expected to have a positive impact on the understanding and the application of such architectures. 
    \item As we illustrate in~\cref{sec:experiment}, we can use the minimum eigenvalue as a  promising metric to guide NAS. This can lead to further investigation on techniques for efficient evaluation of NAS by avoiding solving the intensive bi-level optimization of NAS explicitly. 
\end{itemize}

Nevertheless, we encourage researchers to further investigate the impact of different architectures and their inductive biases on the society.

\end{document}